%% file: main.tex
\documentclass[11pt]{article}
\usepackage[margin=1in]{geometry}
\def\authornotes{1}
\usepackage{pdfpages}
\ifnum\authornotes=1
    \newcommand{\tselil}[1]{{\color{blue} [T: #1]}}
    \newcommand{\jiaming}[1]{{\color{green} [J: #1]}}
	\newcommand{\xn}[1]{{\color{orange}[XN: #1]}}	
    \newcommand{\nn}[1]{{\color{orange}#1}}	
\else
    \newcommand{\tselil}[1]{}
    \newcommand{\jiaming}[1]{}
    \newcommand{\xn}[1]{}
    \newcommand{\nn}[1]{}	
\fi
\include{macros}
\usepackage{style}

\begin{document}

\pgfdeclarelayer{background}
\pgfdeclarelayer{foreground}
\pgfsetlayers{background,main,foreground}

\title{High-Dimensional Procrustes Matching via Tree Counts}

 \author{Xiaochun (Nora) Niu \and Tselil Schramm \and Jiaming Xu\thanks{
X.\ Niu and J.\ Xu are with The Fuqua School of Business, Duke University, Durham NC, USA, \texttt{\{xiaochun.niu,jx77\}@duke.edu}.
 T.\ Schramm is  with the Department of Statistics, Stanford University, Stanford, CA, USA, 
 \texttt{tselil@stanford.edu}.
 }}
 \date{}
\maketitle

\begin{abstract}
Suppose we observe two sets of $n$ Gaussian vectors in $\R^d$, with the promise that, after applying a permutation of $[n]$ and a rotation of $\R^d$, the two sets are $\rho$-correlated.
The Procrustes matching problem asks us to recover the unknown permutation of $[n]$ that aligns the two sets. 
The problem is well-studied in the low-dimensional regime $d=O(\log n)$, but the high-dimensional regime $d\gg \log n$ has remained largely uncharted: prior matching guarantees require nearly perfect correlation $\rho=1-o(1)$, even for information-theoretic recovery.

Our main result is a polynomial-time algorithm for exact recovery at constant correlation.
The algorithm works by computing and comparing weighted counts of a specially chosen family of ``wide'' trees. 
So long as $d\ge \polylog(n)$, the algorithm succeeds with high probability for any $\rho^2>\sqrt{\alpha}$, where $\alpha\approx 0.338$ is Otter's tree-counting constant. 

We complement this algorithmic result with an improved information-theoretic guarantee, showing that exact recovery is possible when
$\rho^2 \gtrsim \max\{\log n/d,\sqrt{\log n/n}\}$.
We also carry out a low-degree advantage calculation, which suggests that the condition $\rho^2 > \sqrt{\alpha}$ is necessary for any tree-counting algorithm.
\end{abstract}
\clearpage

\tableofcontents
\clearpage

\input{db-align}

\section*{Acknowledgment}
J.~Xu would like to thank Yihong Wu and  Cheng Mao for inspiring discussions on efficient algorithms for matching random geometric graphs. 
T.~Schramm would like to thank Jerry Li for suggesting the reference \cite{collins2017weingarten} for Weingarten calculus.
J.~Xu is supported in part by an NSF CAREER award CCF-2144593.
T.~Schramm is supported in part by NSF CAREER award CCF-2143246.
This research began as a result of conversations at the IDEAL Institute Workshop on Learning in Networks in 2024.
Part of this research was performed while X.~Niu and J.~Xu were visiting Simons Laufer Mathematical Sciences Institute in Berkeley, California, during the Spring 2025 semester.

\bibliography{matching}
\bibliographystyle{alpha}
\end{document}

%% file: macros.tex
\usepackage{amsmath,amsthm,amsfonts,amssymb}
\usepackage{graphicx,enumerate}
\graphicspath{ {./figures/} }
\usepackage[title]{appendix}

\usepackage[
            CJKbookmarks=true,
            bookmarksnumbered=true,
            bookmarksopen=true,
            colorlinks=true,
            citecolor=red,
            linkcolor=blue,
            anchorcolor=red,
            urlcolor=blue
            ]{hyperref}

\usepackage[nameinlink]{cleveref}
\crefname{algocf}{algorithm}{algorithms}

\newcommand{\E}{\mathbb{E}}

\newcommand{\inner}[1]{\langle #1 \rangle}
\newcommand{\Bin}{\mathrm{Bin}}

\usepackage{tikz}
\usepackage{tikz-qtree}

\usepackage{tikz-cd}

\usetikzlibrary{calc,arrows,cd}
\usetikzlibrary{decorations.pathreplacing,decorations.markings}

\tikzstyle{dot}=[circle,fill,black,inner sep=1pt]

\tikzset{
  on each segment/.style={
    decorate,
    decoration={
      show path construction,
      moveto code={},
      lineto code={
        \path [#1]
        (\tikzinputsegmentfirst) -- (\tikzinputsegmentlast);
      },
      curveto code={
        \path [#1] (\tikzinputsegmentfirst)
        .. controls
        (\tikzinputsegmentsupporta) and (\tikzinputsegmentsupportb)
        ..
        (\tikzinputsegmentlast);
      },
      closepath code={
        \path [#1]
        (\tikzinputsegmentfirst) -- (\tikzinputsegmentlast);
      },
    },
  },
  mid arrow/.style={postaction={decorate,decoration={
        markings,
        mark=at position .5 with {\arrow[#1]{stealth}}
      }}},
  early arrow/.style={postaction={decorate,decoration={
        markings,
        mark=at position .2 with {\arrow[#1]{stealth}}
      }}},
}

\tikzstyle{int}=[draw, fill=blue!15, minimum size=2em]
\tikzstyle{init} = [pin edge={to-,thin,black}]

\def\alternatecolorred{%
    \pgfkeysalso{red}%
    \global\let\alternatecolor\alternatecolorblue 
}
\def\alternatecolorblue{%
    \pgfkeysalso{blue}%
    \global\let\alternatecolor\alternatecolorred 
}

\newcommand{\altred}{\let\alternatecolor\alternatecolorred 
\tikzset{every edge/.append code = {%
    \global\let\currenttarget\tikztotarget 
    \pgfkeysalso{append after command={(\currenttarget)}}
			\alternatecolor
}}
}
\newcommand{\altblue}{\let\alternatecolor\alternatecolorblue 
\tikzset{every edge/.append code = {%
    \global\let\currenttarget\tikztotarget 
    \pgfkeysalso{append after command={(\currenttarget)}}
			\alternatecolor
}}
}

\tikzstyle{vertexdot}=[circle, draw, fill=black, minimum size=3,inner sep=0pt]

\theoremstyle{theorem}
\newtheorem{theorem}{Theorem}[section]
\newtheorem{lemma}{Lemma}[section]
\newtheorem{proposition}{Proposition}[section]

\theoremstyle{definition}
\newtheorem{definition}{Definition}[section]

\theoremstyle{remark}

\usepackage{color}
\usepackage[linesnumbered,ruled,vlined]{algorithm2e}
\usepackage{xspace}

\newcommand{\reals}{{\mathbb{R}}}

\newcommand{\expect}{\mathbb{E}}
\newcommand{\Prob}{\mathbb{P}}

\newcommand{\prob}[1]{ \mathbb{P}\left\{ #1 \right\} }

\newcommand{\var}{\mathrm{Var}}
\newcommand{\cov}{\mathrm{Cov}}

\newcommand{\ie}{i.e.\xspace}
\newcommand{\iid}{i.i.d.\xspace}

\newcommand{\norm}[1]{\left\|{#1} \right\|}

\newcommand{\lnorm}[2]{\left\|{#1} \right\|_{{#2}}}

\newcommand{\Fnorm}[1]{\lnorm{#1}{\rm F}}
\newcommand{\fnorm}[1]{\|#1\|_{\rm F}}

\newcommand{\iprod}[2]{\left \langle #1, #2 \right\rangle}
\newcommand{\Iprod}[2]{\langle #1, #2 \rangle}
\newcommand{\indc}[1]{{\mathbf{1}_{\left\{{#1}\right\}}}}
\newcommand{\Indc}{\mathbf{1}}
\newcommand{\Ind}{\mathbf{1}}
\newcommand{\diag}{\mathsf{diag}}

\newcommand{\calC}{{\mathcal{C}}}
\newcommand{\calD}{{\mathcal{D}}}

\newcommand{\calF}{{\mathcal{F}}}
\newcommand{\calG}{{\mathcal{G}}}

\newcommand{\calT}{{\mathcal{T}}}

\newcommand{\ttm}{{\mathbf{m}}}

\DeclareMathAlphabet{\varmathbb}{U}{bbold}{m}{n}
\newcommand{\id}{\textit{id}}

\renewcommand{\hat}{\widehat}

\usepackage{bbm}

\newcommand{\R}{\mathbb{R}}

\newcommand{\N}{\mathbb{N}}

\DeclareMathOperator{\Tr}{Tr}

\newcommand{\vecc}{\mathsf{vec}}

\newcommand{\fS}{\mathfrak{S}} 

\newcommand{\ER}{Erd\H{o}s-R\'{e}nyi\xspace}

\newcommand{\poly}{\mathrm{poly}}
\newcommand{\Cat}{\mathrm{Catalan}}

\def\##1\#{\begin{align}#1\end{align}}
\def\$#1\${\begin{align*}#1\end{align*}}

\newcommand{\bc}{\mathbf{c}}

\newcommand{\ttmbi}{\ttm_{\mathrm{bi}}}
\newcommand{\defeq}{\triangleq}
\newcommand{\similarity}{\mathrm{score}}
\newcommand{\Label}{\mathrm{label}}

\newcommand{\ovd}{o}
\newcommand{\xr}{x}

%% file: db-align.tex
\section{Introduction}
Aligning high-dimensional datasets is a fundamental problem in modern statistics and machine learning. It arises in tasks such as aligning word embeddings across languages \cite{lample2018word}, matching entities across networks \cite{feydeep}, integrating biological datasets across experiments \cite{stuart2019comprehensive}, and registering point clouds or shapes in computer vision \cite{zeng20173dmatch}. In these settings, two datasets may encode the same underlying objects, but their correspondence is obscured by measurement noise, relabeling, and  latent geometric transformations. A central question is therefore whether the unknown correspondence can be recovered efficiently, how to design algorithms achieving such recovery, and how the recovery guarantees depend on the ambient dimension and signal strength.

We study an idealized version of this question called the \emph{Procrustes matching} problem: we observe two datasets $\{X_i\}_{i=1}^n, \{Y_i\}_{i=1}^n \subset \mathbb{R}^d$ of correlated Gaussian random vectors whose correspondence is obscured by an unknown permutation and an unknown orthogonal transformation. Specifically, for each data point $i\in[n]$,
\begin{align}\label{eq:pm-model}
X_i &\overset{\text{i.i.d.}}{\sim} \mathcal{N}(0, I_d), \notag\\
Y_i &= \rho Q X_{\pi(i)} + \sqrt{1-\rho^2} Z_i,
\end{align}
where $\pi$ is an unknown permutation drawn uniformly at random from the set of all permutations on $[n]$, $Q \sim \mathcal{O}_d$ is an unknown orthogonal transformation drawn from the Haar measure on $d \times d$ orthogonal matrices, $\rho \in [0,1]$ is the correlation parameter, and $Z_i \overset{\mathrm{i.i.d.}}{\sim} \mathcal{N}(0, I_d)$ is independent noise. Given only the matrices $X = (X_1^\top; \cdots; X_n^\top)$ and $Y = (Y_1^\top; \cdots; Y_n^\top)$, our goal is to recover the hidden permutation $\pi$ and align the two datasets.

When the permutation $\pi$ is known, the Procrustes matching problem simplifies to the classical orthogonal Procrustes problem, which admits a closed-form solution via singular value decomposition~\cite{schonemann1966generalized}. On the other hand, when the orthogonal transformation $Q$ is known, recovering the hidden permutation becomes a linear assignment problem on correlated Gaussian observations, which can be solved in polynomial time, for example, in cubic time using the celebrated Hungarian algorithm~\cite{kuhn1955hungarian}. 
Recovery  thresholds for this linear assignment problem  have been characterized in a recent line of work~\cite{dai2018performance,dai2019database,kunisky2022strong, wang2022random}.

In the Procrustes matching problem \eqref{eq:pm-model}, both the permutation $\pi$ and the orthogonal transformation $Q$ are unknown. By a simple observation, Procrustes matching is equivalent to the \emph{geometric graph alignment} problem when one assumes a \emph{linear kernel}: there, only Gram matrices $A=XX^\top$ and $B=YY^\top$ are observed and treated as weighted adjacency matrices. Indeed, let $A^{1/2}\triangleq U\Lambda^{1/2}$ based on the eigenvalue decomposition $A=U \Lambda U^\top$ and similarly for $B^{1/2}$. Then $(A^{1/2}, B^{1/2}Q') \stackrel{\mathrm{d}}{=} (X,Y)$ for an independent $Q'\sim \cO_d$. Hence, statistical and algorithmic guarantees for one formulation translate directly to the other.

The Procrustes matching problem has been extensively studied (under a few different names) in the literature~\cite{rangarajan1997softassign,maron2016point,dym2017exact,grave2019unsupervised,wang2022random,gong2024umeyama,even2024aligning,li2025computational}. 
In the low-dimensional regime $d=o(\log n)$, \cite{wang2022random} established sharp information-theoretic thresholds for exact matching ($\rho^2 = 1 - o(n^{-4/d})$) and almost exact matching ($\rho^2 = 1 - o(n^{-2/d})$). Subsequently, \cite{gong2024umeyama} proved that a spectral method, namely the classical Umeyama algorithm~\cite{umeyama1988eigendecomposition}, achieves these thresholds up to polynomial factors in $d$, in polynomial time.  However, a substantial gap remains in our understanding of the  
high-dimensional regime $d \gg \log n$. The
state-of-the-art results in \cite{even2024aligning}, summarized in
\Cref{tab:recovery-guarantees}, show that almost exact recovery can be achieved
information-theoretically when the correlation $\rho$ tends to one.
Computationally, almost exact recovery was known to be achievable in polynomial
time only when the correlation 
$\rho=1-n^{-c}$ for some constant $c>0$. Whether matching can
be achieved efficiently, or even information-theoretically, under constant
correlation in the high-dimensional regime has remained a widely open problem.

In this work, we address this open problem, showing that exact matching is achievable in polynomial time under constant correlation when $d\ge\polylog(n)$. 

\begin{theorem}\label{thm:intro}
Suppose that $d\ge\polylog(n)$. There exists a polynomial-time estimator $\widehat\pi$ that correctly matches all the datapoints with high probability whenever
\#\label{eq:main-rho}
\rho^2 > \sqrt{\alpha},
\#
where $\alpha \approx 0.338$ is Otter's tree-counting constant~\cite{otter1948number}.
\end{theorem}
We complement these algorithmic results with improved information-theoretic guarantees that allow vanishing correlations, matching the impossibility result of \cite{wang2022random} up to constant factors when $d\ll \sqrt{n\log n}$. \Cref{tab:recovery-guarantees} summarizes the known statistical and computational guarantees in the high-dimensional regime.

\begin{table}[htbp]
\centering
\addtolength{\tabcolsep}{8pt} 
\renewcommand{\arraystretch}{1.5} 
\begin{tabular}{lll}
\toprule
& Information-theoretic & Computational \\
\midrule \relax 
\cite{even2024aligning} & $\rho^2 \ge 1 - o(1)$ & $\rho^2 \geq 1 - n^{-c}$ \\
This work & $\rho^2 \gtrsim \max \bigg\{ \dfrac{\log n}{d}, \sqrt{\dfrac{\log n}{n}} \bigg\}$ & $\rho^2 > \sqrt{\alpha}$ \\
\bottomrule
\end{tabular}
\caption{Comparison of sufficient conditions for matching when $d \ge \poly( \log n)$. Our results are for exact matching, whereas the results in \cite{even2024aligning} are for the weaker notion of almost exact matching and further require $d \ll n$. 
}
\label{tab:recovery-guarantees}
\end{table}

Otter's constant appears in \eqref{eq:main-rho} because our algorithm is based on  counting trees. One may compare to the problem of matching correlated \ER graphs~\cite{mao2022random}, where tree-counting based estimators achieve efficient recovery once the graph correlation exceeds $\sqrt{\alpha}$. 
The Wishart matrices $XX^\top$ and $YY^\top$ in our setting behave like correlated Gaussian matrices with correlation $\rho^2$ when $d$ is large.\footnote{The total variation distance between the distribution of the off-diagonals of the Wishart matrix $XX^\top$ and their iid Gaussian counterpart converges to zero if and only if 
$d \gg n^3$~\cite{jiang2015approximation,bubeck2018entropic}; our results cover significantly smaller $d$ as well, but we still find the analogy useful.
} From this perspective, we may view $\rho^2>\sqrt{\alpha}$ as the Wishart analogue of the condition for correlated \ER graph matching. Despite this connection, our setting is fundamentally different from correlated \ER graph matching, as we discuss below.

\subsection{Key Challenges and New Ideas}

The key defining feature of the Procrustes matching problem \eqref{eq:pm-model} is the presence of an unknown rotation matrix $Q$ in addition to the unknown matching $\pi.$  Most existing algorithmic approaches circumvent the unknown rotation $Q$ by matching through  Gram matrices $A=XX^\top$ and $B=YY^\top$. 
These  include, for instance, spectral methods such as Umeyama~\cite{umeyama1988eigendecomposition} or GRAMPA~\cite{FMWX19a}, which align the eigenvectors of $A$ and $B$~\cite{wang2022random,gong2024umeyama}, as well as convex relaxations of the quadratic assignment problem $\min_{\Pi}\fnorm{\Pi A\Pi^\top - B}^2$~\cite{grave2019unsupervised,even2024aligning}. However, existing algorithms either work only in the low-dimensional regime $d=O(\log n)$ or require nearly perfect correlation $\rho = 1- n^{-c}$ for some constant $c>0$.

Motivated by the recent advances in graph matching~\cite{barak2018nearly,mao2022random,chai2024efficient,chen2026detecting}, a promising approach is to construct similarity scores by counting a certain family of subgraphs in weighted graphs defined by $A$ and $B$. 
The success of this approach crucially relies on the counts of different subgraphs being only weakly correlated, so that they can be aggregated into robust signatures capable of withstanding small constant correlation. However, when  $A=XX^\top$ and $B=YY^\top$ are low-rank, there is excessive correlation among different subgraph counts. To see this, consider the $k$-walk count rooted at vertex $i$, that is, $e_i^\top A^k \bf{1},$ for different lengths $k$.
When $d \ll n$, all $d$ nonzero eigenvalues of $A$, denoted by $\lambda_\ell$'s,  concentrate around $n$ and moreover $\lambda_\ell^k \approx n^k$ when  $k \ll \sqrt{n/d}$. Therefore, 
$e_i^\top A^k \mathbf{1} \approx n^{k-1} e_i^\top A \mathbf{1}$.\footnote{More precisely, by standard matrix concentration inequality, $\lambda_\ell = n ( 1+ \delta_\ell)$
with $\delta_\ell=O_p(\sqrt{d/n})$, so that $\lambda_\ell^k = n^k (1+ O(k\delta_i)) $. Therefore, by letting $c_\ell=\iprod{e_i}{u_\ell}\iprod{u_\ell}{\mathbf{1}}$, we have
$e_i^\top A^k \mathbf{1} = n^k \sum_{\ell=1}^d (1+ O(k\delta_i)) c_\ell = n^k  
\left( \sum_{\ell=1}^d c_\ell + O_p(k d/n) \right),$ where the last equality follows because $c_\ell$'s are approximately \iid $N(0,1/n)$.}  In other words, the $k$-walk counts essentially collapse to degree counts up to scaling, and are therefore highly correlated across different $k.$

This motivates us to depart from this conventional Gram-matrix paradigm and match directly using the data matrices $X$ and $Y$. A major difficulty immediately arises: unlike in correlated random graph matching, correctly matched edges do not themselves induce positive correlation. Indeed, if a single edge $(\pi(u),a)$ in $X$ is matched to the edge $(u,b)$ in $Y$, even though the endpoint $u$  has been correctly matched, we still have $\expect{X_{\pi(u)a}Y_{ub}}=\rho\expect{Q_{ab}}=0$ because of the random rotation. 
The key observation, however, is that correlation reappears when pairs of edges sharing a node in $[d]$ are matched together. Namely, matching a pair of edges $(\pi(u),a),(\pi(v),a)$ in $X$ with another pair of edges $(u,b),(v,b)$ in $Y$ yields
$
\expect{X_{\pi(u)a}X_{\pi(v)a}Y_{ub}Y_{vb}}
=
\rho^2\expect{Q_{ab}^2}
=
\rho^2/d.$ 
This allows us to extract the desired correlation from corresponding subgraph counts in $X$ and $Y$, despite the presence of the unknown rotation $Q$. 
We therefore construct signature vectors by counting carefully chosen families of rooted subgraphs in the weighted bipartite graphs defined by $X$ and $Y$, designed so that edge pairs can be matched across the two graphs. We then define similarity scores through suitably weighted inner products of these signatures and recover the matching based on these scores.

The main analytical difficulty, however, is that the precise correlation structure among these subgraph counts depends on higher-order moments of the rotation matrix $Q$. To control these correlations, we use Weingarten calculus~\cite{matsumoto2013weingarten,collins2017weingarten} to evaluate and bound the correlations. This requires a delicate analysis of circuit decompositions in unions of subgraph copies. In addition, to simplify the correlation structure and suppress undesired correlations among distinct subgraph counts, we crucially constrain the degree of every node in $[d]$ to be exactly two.\footnote{The degree-$2$ constraint can be understood from the perspective of decorrelating walk counts. 
Indeed, the excessive correlations among walk counts based on the Gram matrix $A = XX^\top$ arise largely from walks that repeatedly revisit the same $d$-node, thereby creating high-degree nodes. By imposing the degree-$2$ constraint on $d$-nodes, we eliminate precisely these highly correlated contributions.}

Finally, our signature is comprised of counts of a specific family of ``wide'' trees, consisting of 
many non-isomorphic small branches sprouting from a high-degree common root. This ``wide'' structure plays a crucial role in controlling the undesired correlations among different tree counts. Moreover, despite having size $\Theta(\log n)$, these trees can be counted sufficiently accurately in polynomial time using the color-coding method~\cite{alon1995color,alon2008biomolecular}. 
A number of recent works in high-dimensional statistics have used a combination of color-coding and tree-counting to push algorithmic limits, for example in stochastic block models \cite{hopkins2017efficient,chai2024efficient,chen2026detecting} and matching correlated \ER random graphs~\cite{mao2021exact}. 
Each setting calls for a tailored tree structure. 
We draw inspiration from these works. 
In our setting, the bipartite graphs $X, Y$ are highly imbalanced in the regime $d=\polylog(n) \ll n$, which necessitates the use of ``wide'' trees with very large root degrees, whereas the chandelier trees in
\cite{mao2022random}, for instance, are ``thin'' and ``long''.

\subsection{Notation and Organization} We write $[n]=\{1,\ldots,n\}$ and let $\fS_n$ denote the set of permutations of $[n]$. For graphs $G$ and $H$, we write $G\cong H$ if $G$ and $H$ are isomorphic. For a bipartite graph $G$ on $[n]\times[d]$, we write $V_n(G)$ and $V_d(G)$ for the vertices of $G$ in $[n]$ and $[d]$, respectively. In all multigraphs, parallel edges incident to the same pair of vertices are treated as distinct. Unless otherwise specified, the union of two multigraphs retains all parallel edges.

The rest of the paper is organized as follows. \Cref{sect:main} presents our tree-counting matching algorithm, its performance guarantees, and the information-theoretic recovery guarantees. \Cref{sect:alg-guarantee} establishes the statistical guarantees of similarity scores derived from the tree counts. \Cref{sect:join-mom} derives the joint moment formula via Weingarten calculus. \Cref{sect:mean,sect:var} then prove the mean separation and variance bounds stated in \Cref{sect:alg-guarantee}. \Cref{sect:approx} shows how to approximately compute the similarity scores in polynomial time via color coding. \Cref{sect:exact} boosts the almost exact recovery to exact recovery using a seeded matching algorithm. \Cref{sect:qap} establishes the information-theoretic recovery guarantees by analyzing the quadratic assignment estimator. Finally, \Cref{sect:low-deg} applies the low-degree polynomial framework to justify our focus on the particular tree design.
\section{Main Results}\label{sect:main}
In this section, we introduce our matching algorithm and present its main performance guarantees.

\subsection{Matching Algorithm}
Our algorithmic approach is as follows: we assign a vector-valued polynomial signature to each datapoint, compute similarity scores between signatures of datapoints in $X$ and $Y$, and then match datapoints whose similarity scores are sufficiently large.

We first define the polynomial features used in the signatures. We represent monomials in the entries of $X$ using bipartite multigraphs. Let $G = (V(G), E(G))$ be a bipartite multigraph with vertex set $V(G) \subset [n]\cup [d]$ and edge multiset $E(G)\subset [n]\times [d]$. We define the monomial
\[
X^G \defeq \prod_{(u,a) \in E(G)} X_{ua}.
\]
Let $G$ be an unlabeled bipartite graph with one special node that we designate as the ``root.'' 
For $i\in[n]$, we consider all bipartite graphs rooted at $i$ and isomorphic to $G$, and define the polynomial
\#\label{eq:label-t-x-i}
\Label_{G}(X,i) = \sum_{\substack{G_1 \cong G\\ \mathrm{root}(G_1) = i}} X^{G_1}.
\#
For example, when $G$ is a single edge rooted at a node in $[n]$, then $\Label_G(X,i)$ is the $i$th row sum.

Any family $\calG$ of rooted bipartite graphs assigns a signature vector to each datapoint:
\$
&\mathrm{signature}(X_i) = (\Label_G(X,i)\colon G \in \calG),
\$
and similarly for $Y_j$.
We assign a similarity score to each pair of datapoints using the inner product of their signatures:
\#\label{def:score-ij}
\similarity_{ij}^{\cG} = \Iprod{\mathrm{signature}(X_i)}{\mathrm{signature}(Y_j)}_{\calG} = \sum_{G \in \calG} \aut(G)\cdot \Label_G(X,i) \cdot \Label_G(Y,j),
\#
where $\aut(G)$ denotes the number of root-fixing automorphisms of $G$. 
The matching rule is then straightforward: we match a datapoint from $X$ with a datapoint from $Y$ whenever their similarity score is sufficiently large. The procedure is summarized as \Cref{alg:main}.

\vspace{8pt}
\begin{algorithm}[H]
\DontPrintSemicolon  
  \KwIn{Data matrices $X$ and $Y$, a family of rooted bipartite graphs $\cG$, and a threshold $\tau$.}
  \KwOut{$\widehat\pi\colon I\to [n]$.}
  
  \For{each $G\in\cG$ and each $i\in [n]$}{Compute $\Label_G(X, i)$ and $\Label_G(Y, i)$ as given by \eqref{eq:label-t-x-i}.}

  \For{each pair $(i,j)\in[n]\times [n]$}{Compute the similarity score between $X_i$ and $Y_j$ as given by \eqref{def:score-ij}.\;}
  \For{each $i \in [n]$}{If there exists a unique $j \in [n]$ such that $\similarity_{ij}
\ge \tau$, let $\widehat{\pi}(i) = j$ and include $i$ in set $I$.\;}
\caption{Procrustes matching via counting trees}
\label{alg:main}
\end{algorithm}
\vspace{8pt}

The performance of this algorithm depends critically on the choice of the graph family $\cG$.
We next describe our choice of $\cG$.

\subsection{Wide Trees}

Our family $\cG$ will consist of rooted, bipartite trees in which every node in $[d]$ has degree $2$. Furthermore, we will require that the trees be ``wide,'' in the sense that the degree of the root is quite large relative to the size of the tree.
We impose these constraints because of computational and analytical considerations, as we now explain.

The construction of the graph family is guided by a few competing requirements: first, the signatures should be (approximately) computable in polynomial time.
This is the requirement that leads us to consider trees; indeed, if $G$ is a tree, the \emph{color-coding} technique~\cite{alon1995color,alon2008biomolecular} can be used to compute $\Label_G(X,i)$ more efficiently.

The other two requirements are statistical: the family must be sufficiently rich to create a large mean separation between true ($j=\pi(i)$) and false $(j\neq \pi(i))$ pairs, and the variance of the similarity scores should be well controlled.
To ensure that these conditions hold, it will be useful to consider wide trees whose nodes in $[d]$ have bounded degree.

Each tree in our family will have a root with many incident subtrees or ``branches'' of equal size, chosen from the following \emph{branch family} $\cJ$. We recall a classical result in combinatorics: the number of unlabeled trees with $k$ edges is $(\alpha + o(1))^{-k}$ as $k \to \infty$, where $\alpha \approx 0.338$ is Otter's constant~\cite{otter1948number}. Moreover, \cite{olsson2022automorphisms} shows that most such trees have at most exponentially many automorphisms: there exists an absolute constant $C$ such that the number of trees with $k$ edges and at most $\exp(Ck)$ automorphisms is still $(\alpha + o(1))^{-k}$. Fix such a constant $C$.

\begin{definition}[Branch family $\cJ$]\label{def:j} 
Each $J\in\cJ$ is a rooted bipartite simple tree with nodes in $[n]$ (the $n$-nodes) and $[d]$ (the $d$-nodes).
We begin with a rooted tree on $n$-nodes with $M$ edges whose root has degree one and whose automorphism group has size at most $\exp(CM)$, where $C$ is the constant fixed above.  
We then subdivide each edge $(u,v)$ of the tree into two edges, $(u,a)$ and $(v,a)$, by inserting a $d$-node $a$ between $u$ and $v$. 
\end{definition}

By construction, each $J$ contains $2M$ edges, and
\[
|\cJ| = (\alpha + o(1))^{-M}.
\]
The trees in $\cJ$ will serve as the branches of our final tree family.
\begin{definition}[Tree family $\cT$]\label{def:t}
Each tree $T \in \cT$ is constructed by taking $D$ pairwise non-isomorphic branches from $\cJ$ and merging their roots into a single common root.     
\end{definition}
Throughout the paper, we write $\similarity_{ij}$ for $ \similarity_{ij}^\cT$, the similarity score in \eqref{def:score-ij} under $\cT$.
Recall that $D$ is the root degree and $2M$ is the number of edges per branch. 
Let $K = D M$. 
Then each $T$ has $2K$ edges, $K+1$ $n$-nodes (including the root), and $K$ $d$-nodes, each of degree $2$. 
\Cref{fig:tree-plot} illustrates an example of such a tree. 

\input{fig/tree-t}

We will later choose the parameters $K$ and $D$ carefully to meet our statistical requirements, taking $K=\Theta(\log n)$ and $D= \Theta(K/\log K)$. 
The resulting trees are wide, with large root degree and small branch size ($M = \Theta(\log K)$). 
These choices will be explained in \Cref{sect:main-intui}.

\subsection{Statistical Properties of Our Similarity Scores}\label{sect:main-intui}

Here we elaborate further on the statistical properties of our similarity scores. 
To keep notation manageable, assume from here on out (without loss of generality) that $\pi =\id$ is the identity permutation.
We also at times neglect to account for the fact that the $n$-node and $d$-node labels within each labeled $T$ are distinct, leaving that accounting for the formal proofs.

\paragraph{Mean separation.} We first study the mean of $\similarity_{ij}$ for true and false pairs. By definition,
\$ 
\E\left[\similarity_{ij}\right] &= \sum_{T \in \cT} \aut(T) \cdot  \E\left[\Label_{T}(X,i)\cdot \Label_{T}(Y,j)\right]
= \sum_{T \in \cT} \aut(T) 
\sum_{\substack{T_1,T_2 \cong T\\\mathrm{root}(T_1) = i\\ \mathrm{root}(T_2) = j}}\E\left[X^{T_1} Y^{T_2}\right]. 
\$
Thus, the problem reduces to understanding the joint moments $\E[X^{T_1} Y^{T_2}]$. 
Intuitively, these joint moments are largest when the $n$-nodes of $T_1$ and $T_2$ are labeled identically.
Here we will only briefly explain, deferring a full analysis to \Cref{sect:join-mom}.

Recall that the trees $T \in \cT$ were constructed by first choosing a tree on the $n$-nodes, then subdividing each edge by inserting a $d$-node. 
Hence each $d$-node in $T$ induces a natural ``pair'' of edges.
Consider the special case where $G$ is one such ``pair'' of edges, so $E(G) = \{(u,a),(v,a)\}$, and let $G_1$ be a labeling of $G$ in $X$ and $G_2$ a labeling of $G$ in $Y$.
A straightforward calculation shows that
\begin{align}
\E\left[X^{G_1}Y^{G_2}\right] = \begin{cases}
    \rho^2/d, & \text{ if } V_{n}(G_1) = V_{n}(G_2),\\
    0, & \text{ otherwise}.
\end{cases}
\label{eq:edge_pair_correlation}
\end{align}
Thus, the moment is zero unless $G_1$ and $G_2$ induce the same edge after their $d$-nodes are contracted.

When $i$ and $j$ form a true pair, we can lower bound $\E[\similarity_{ii}]$ by considering terms $\E[X^{T_1}Y^{T_2}]$ for which \emph{every $n$-node} in $T_1$ has the same label as the corresponding $n$-node in $T_2$. In this configuration, every edge pair in $T_1$ is matched with its corresponding edge pair in $T_2$, as illustrated in \Cref{fig:mean-sep}.
Therefore, up to lower-order terms we have $\E[X^{T_1}Y^{T_2}] = (\E[X^{G_1}Y^{G_2}])^{|E(T)|/2} = (\rho^2/d)^{K}$, recalling that the tree has $2K$ edges. We now count the number of these dominant configurations. 
There are $|\cT|$ choices for the unlabeled tree, $n^K$ choices for the non-root $n$-nodes,
and $d^{2K}$ choices for the $d$-nodes in the two labeled trees. 
All other configurations contribute lower-order terms. Therefore, 
\$ 
\E[\similarity_{ii}] = (1+o(1))|\cT| n^K d^{2K} (\rho^2/d)^K  = (1+o(1))|\cT| n^K d^{K}\rho^{2K}.
\$
For false pairs where $i \neq j$, 
\[
\E[\similarity_{ij}] = 0.
\]
This is because the joint moment $\E[X^{T_1}Y^{T_2}]$ is only nonzero if the degree of $u$ in $T_1$ is equal to the degree of $u$ in $T_2$; since the root is the unique node of degree $D$ in $T$, $\E[X^{T_1}Y^{T_2}] = 0$ unless $i = \mathrm{root}(T_1) = \mathrm{root}(T_2) = j$.

\input{fig/mean-sep}

\paragraph{Variance bounds.} 
Given the separation in the mean of the similarity scores for true and false pairs, we prove concentration by bounding the variance.
To successfully argue that most of the $n$ true pairs can be distinguished from the $n(n-1)$ false pairs, we require  \#\label{eq:var-ratio-ij-req}
\frac{\var(\similarity_{ij})}{(\E[\similarity_{ii}])^2} 
= \begin{cases} 
o(1/n), & j \neq i,\\
o(1), & j = i.
\end{cases}
\#
In the following explanation, we focus on a false pair with $j\neq i$. Since $\E[\similarity_{ij}] = 0$, we have
\#\label{eq:var1}
\var(\similarity_{ij}) 
= \E[\similarity_{ij}^2] 
&= \sum_{T,S\in \calT}
\aut(T) \aut(S) \sum_{\substack{T_1, T_2 \cong T, S_1, S_2 \cong S\\ \mathrm{root}(T_1) = \mathrm{root}(S_1) = i\\ \mathrm{root}(T_2) = \mathrm{root}(S_2) = j}}
\E\left[X^{T_1} Y^{T_2} X^{S_1} Y^{S_2} \right].
\#
Thus, the variance is governed by joint moments of four trees $T_1,T_2,S_1,S_2$. 
The moments are governed by the degree to which these labeled trees' structure matches; there are two types of configurations whose contribution dominates variance, as illustrated in \Cref{fig:variance-illustrate}. 

\input{fig/var-2}

The first type of configuration dominates when $X$ and $Y$ are imbalanced, namely $d\ll n$, as illustrated in \Cref{fig:var-top}. Focus first on the four branches on the left. The first $d$-node in the branch of $T_1$ coincides with that in the branch of $S_1$ at $a$; similarly, the first $d$-node in the branch of $T_2$ coincides with that in the branch of $S_2$ at $b$. This creates two parallel edge pairs, one red and one blue, which contribute factors $\E[X_{ia}^2]=\E[Y_{jb}^2]=1$. Moreover, the red edge pair $((u,a),(v,a))$ is matched with the blue edge pair $((u,b),(v,b))$. Each red edge pair in the descendant subtree of this branch in $T_1$ can then be matched with a blue edge pair in the descendant subtree of the isomorphic branch in $T_2$, and the same holds for $S_1$ and $S_2$. 

Suppose that all branches are partitioned into groups of four branches arranged in this way. Then, except for the $2D$ parallel edges incident to the roots, the remaining $2K-D$ matched edge pairs together contribute a factor $(\rho^2/d)^{2K-D}$, up to lower-order terms, by~\eqref{eq:edge_pair_correlation}. The unlabeled trees $T$ and $S$ can be chosen independently, giving $|\cT|^2$ choices. There are $n^{2K}$ choices for the $n$-node labels. For the $d$-nodes, the root-adjacent $d$-nodes are shared by two $X$-trees or by two $Y$-trees, giving $d^{4K-2D}$ choices in total. Hence, after dividing this configuration's contribution to the variance by the squared mean, we obtain
\#\label{eq:wide-root} 
\frac{|\cT|^2 n^{2K} d^{4K-2D} (\rho^2/d)^{2K-D}}{(\E[\similarity_{ii}])^2}   
= \frac{1}{(\rho^2d)^{D}}. 
\# 
When $d$ is only $\polylog(n)$, to achieve \eqref{eq:var-ratio-ij-req}, the bound \eqref{eq:wide-root} forces 
\#\label{eq:wide-tree-2} 
D \ge C \log n/\log \log n,
\#
for a sufficiently large constant $C>0$.
That is, the trees must be wide.

The second dominant configuration occurs when $T_1 = S_1$ and $T_2 = S_2$ as labeled trees, as illustrated in \Cref{fig:var-bottom}. In this case, the interaction of $X$ and $Y$ is irrelevant; instead, pairs of parallel edges inside $X$ contribute factors of the form $\E[X_{ua}^2] = 1$ and pairs of parallel edges inside $Y$ contribute similarly. Since $T$ and $S$ must be isomorphic in this configuration, there are only $|\cT|$ such terms. 
There are $n^{2K}$ choices of the $n$-node labels and $d^{2K}$ choices of the $d$-node labels. 
Hence, the contribution of this term is of order $|\cT| n^{2K} d^{2K}$. Dividing by the mean square gives
\# \label{eq:match}
\frac{|\cT| n^{2K} d^{2K}}{(\E[\similarity_{ii}])^2}  = \frac{1}{|\cT|\rho^{4K}}.
\#
So we wish for $|\cT|$ as large as possible to upper bound this quantity by $o(1/n)$ as required in \eqref{eq:var-ratio-ij-req}.
In particular, we have 
\[
|\cT| = \binom{|\cJ|}{D} \approx
 \left(\frac{e\cdot (\alpha+o(1))^{-M}}{ D}\right)^D = (\alpha+o(1))^{-K} \exp(-D\log D + D).
\]
Now,
for each $\epsilon>0$, so long as $D \le c K /\log K$ for $c>0$ sufficiently small,
\[
|\cT| \ge (\alpha + \epsilon/2)^{-K}.
\]
Thus, by choosing 
\#\label{eq:wide-tree-1} 
\rho^4 \ge \alpha + \epsilon, \quad  \quad K = C_1\log n,
\quad \text{and} \quad D = c_1 K / \log K,
\#
for constants $C_1>0$ sufficiently large and $c_1>0$ sufficiently small, we satisfy both the wide tree constraint~\eqref{eq:wide-tree-2} and the requirement from \eqref{eq:var-ratio-ij-req} that \eqref{eq:match} is $o(1/n)$. This explains the parameter choices of $K$ and $D$ and the correlation condition $\rho^4 > \alpha$. We state these statistical guarantees formally in 
\Cref{sect:alg-guarantee}.

\subsection{Performance Guarantees}

We now state the main guarantee for our matching algorithm. 

\begin{theorem} \label{thm:algo}
Suppose that $\rho^4 \ge \alpha+\epsilon$, where $\epsilon$ is an arbitrarily small constant. There exist constants $C_1, C_2, c_1 >0$, depending only on $\epsilon$, such that the following holds. Let $\cT$ be the tree family defined in \Cref{def:t} with parameters specified in~\eqref{eq:wide-tree-1}. Let $\widehat{\pi}$ and $I$ denote the output of \Cref{alg:main} applied to $\cT$ with threshold $\tau= c|\cT|n^Kd^K\rho^{2K}$ for some constant $c \in (1/4, 3/4)$. If $d \ge (\log n)^{C_2}$, then with probability $1 - o(1)$, $\widehat{\pi} = \pi\,|\,_{I}$ and $|I| \ge n - n/\log^2 n$.
\end{theorem}

\paragraph{Polynomial-time approximation.}
With our choice of $\cT$ and $K = \Theta(\log n)$, a direct computation of $\Label_T(X,i)$ in \eqref{eq:label-t-x-i} and $\similarity_{ij}$ in \eqref{def:score-ij}, which requires enumerating all labeled trees $T_1 \cong T$ with $2K+1$ vertices, would take $n^{K+1} d^K$ time, which is super-polynomial in $(n,d)$ when $K=\Theta(\log n)$. To address this computational issue, in \Cref{sect:approx} we present a polynomial-time algorithm (\Cref{alg:approx-similarity}) that computes an approximation $\widehat{\similarity}_{ij}$ to $\similarity_{ij}$ using the color-coding technique, as in \cite{mao2022random}. 
The same performance guarantees for \Cref{alg:main} continue to hold when $\{\similarity_{ij}\}$ is replaced by $\{\widehat{\similarity}_{ij}\}$, stated as follows. Consequently, we complete the computational result.

\begin{proposition}\label{prop:approx}
\Cref{thm:algo} continues to hold with $\widehat{\similarity}_{ij}$ in place of $\similarity_{ij}$. Moreover, $\{\widehat{\similarity}_{ij}\}_{i,j \in [n]}$ can be computed in $O(n^{C}d)$ time, with a constant $C>0$ depending only on $\epsilon$. 
\end{proposition}

\paragraph{Exact recovery.} Moreover, we can boost almost exact recovery to exact recovery in polynomial time. The key observation is that, given a seed set $J$, if $i\notin J$ and $j = \pi(i)$ form a true pair, then the vectors $(\inner{x_i, x_r})_{r\in J}$ and $(\inner{y_j, y_{\pi(r)}})_{r\in J}$ are correlated, whereas for a false pair, these vectors are independent. Starting from the almost exact matching $\widehat\pi$ on $I$ as seeds, we threshold the correlations between these vectors for unmatched points, and iteratively extend the matching to all vertices. \Cref{alg:exact-while} in \Cref{sect:exact} formalizes this procedure. The following theorem establishes its success.
\begin{theorem}\label{thm:exact}
Suppose that there exists a large constant $C>0$ such that
\begin{align}
\rho^2 \ge C\max\left\{\frac{\log n}{d}, \, \frac{1}{\sqrt{\log n}}\right\}. \label{eq:boosting_cond}
\end{align}
Then, with probability $1-o(1)$, the following holds: for every partial matching $\widehat{\pi}=\pi|_I$, where $I\subset[n]$ satisfies $|I|\ge n-n/\log^2 n$, the seeded matching algorithm \Cref{alg:exact-while}, initialized with $\widehat{\pi}$, outputs $\widetilde{\pi}=\pi$ in $O(n^3d)$ time.
\end{theorem}
Observe that the boosting condition~\eqref{eq:boosting_cond} is implied by the almost-exact recovery condition $\rho^2>\sqrt{\alpha}$, and that the seeded matching algorithm succeeds under any seed set $I$ with $|I|\ge n-n/\log^2 n$. Therefore, combining \Cref{thm:algo} and \Cref{thm:exact} yields \Cref{thm:intro}, showing that exact recovery is achieved in polynomial time whenever $\rho^2>\sqrt{\alpha}$.

\subsection{Information-Theoretic Thresholds and Computational Gaps}
\label{sect:main-thresh}

We complement our algorithmic results with an improved information-theoretic guarantee in the high-dimensional regime. We analyze the following quadratic assignment problem (QAP) estimator:
 \#\label{eq:def-qap}
 \widehat{\Pi} \in \arg\max_{P \in \fS_n} \iprod{B}{ P A P^\top},
 \#
where $A=XX^\top$, $B = YY^\top$, and the maximization is taken over all possible permutation matrices $P$. Note that QAP is NP-hard to solve or even approximate in the worst case~\cite{makarychev2010maximum}.

\begin{theorem}\label{thm:IT} 
There exists a constant $C>0$ such that if
$d\ge C\log n$ and
\#\label{eq:IT_sufficient}
 \rho^2 \ge C\max\left\{\frac{\log n}{d}, \, \sqrt{\frac{\log n}{n}}\right\},
\#
then the QAP estimator in \eqref{eq:def-qap} satisfies $\widehat{\Pi} = \Pi$ with probability $1-o(1)$.
\end{theorem}
When $d \gg \log n$, \cite{wang2022random} shows that even if the orthogonal transformation $Q$ is known,
$
\rho^2 \ge \left(2-o(1)\right)\log n/d
$
is necessary for almost exact recovery of $\pi$.
Therefore, the sufficient condition \eqref{eq:IT_sufficient} matches this necessary condition up to constant factors when $d \ll \sqrt{n\log n}$. If instead $d \gtrsim \sqrt{n \log n}$, then the sufficient condition~\eqref{eq:IT_sufficient} reduces to  $\rho^4 \gtrsim {\log n}/{n}$,
which coincides with the recovery thresholds for the correlated Gaussian model~\cite{ganassali2020sharp,wu2021settling}. It remains open whether  this condition is also necessary for exact recovery in Procrustes matching.

Note that, unlike the correlated Gaussian model, the observed data matrix $B$
is not given by $\Pi A\Pi^\top$ plus
an additive Gaussian noise matrix; thus, the QAP \eqref{eq:def-qap} does not coincide with the maximum likelihood estimator (MLE). 
As shown in \cite[(3)]{wang2022random}, the exact MLE for Procrustes matching maximizes a likelihood function that involves an integral over the Haar measure on the orthogonal group, rendering the exact MLE difficult to analyze directly. The previous work~\cite{wang2022random} applies  Laplace's method to this integral to derive an approximate MLE $\arg\max_{P \in \fS_n} \max_{Q \in \cO_d} \iprod{Y}{P X Q}$ and shows that it achieves the sharp information-theoretic thresholds in the low-dimensional regime when $d \ll \log n.$ However, in the high-dimensional regime, \cite{even2024aligning} only shows that this approximate MLE succeeds in almost exact recovery when $\rho=1-o(1)$ and $d \ll n.$

\paragraph{Computational gaps.} In the high-dimensional regime $d\ge \polylog(n),$  our tree-counting algorithm succeeds under constant correlation $\rho^2> \sqrt{\alpha}$, 
whereas  exact recovery  is information-theoretically possible even when $\rho^2 $ decays polynomially in $n$ and $d.$ We see a substantial \emph{computational-statistical gap}, and whether it can be narrowed or closed remains open. The prior work \cite{li2025computational} shows, under the low-degree hardness framework~\cite{hopkins2017efficient,kunisky2019notes,wein2025computational}, that a closely related hypothesis testing problem exhibits a computational gap. Specifically, all degree-$O(\log n)$ polynomial tests fail to reliably distinguish between the null hypothesis $\rho = 0$ and alternative $\rho = \rho_*$, when $\rho_*^2 \ll 1/\log n$.

One may be tempted to ask whether  counting a richer family of graphs would yield computational guarantee below $\rho^2>\sqrt{\alpha}$. Indeed, going beyond trees, for example, to graphs with bounded tree-width, or using a different algorithmic design may help. For both the correlated Gaussian model and the dense correlated \ER graph model, previous work~\cite{ding2025polynomial_1,ding2025polynomial_2} has shown that one can achieve correct matching even when the correlation  $\rho$ is  an arbitrarily small but fixed constant. 
Within the tree-counting framework, however, our choice of trees appears close to optimal. In \Cref{sect:low-deg}, we show that among all trees with $2K$ edges,
the main contribution to 
the distingushing power  between the null hypothesis $\rho = 0$ and alternative $\rho = \rho_*$, 
comes from trees with degree-$2$ $d$-nodes, as long as $d \ge \poly(K).$ 
This supports the choice made by our algorithm, which counts precisely this family of trees.

\section{Statistical Guarantees} \label{sect:alg-guarantee}

In this section, we present the statistical guarantees for the similarity scores suggested by the preceding intuition, and then deduce \Cref{thm:algo} from them.  
We begin with the mean separation.

\begin{proposition}\label{prop:mean-sep}
The similarity scores defined in \eqref{def:score-ij} with $\cT$ being
the tree family from \Cref{def:t} satisfy
\#\label{eq:mean-main}
\E[\similarity_{ii}] \ge \left(1- o(1)\right) |\cT|n^Kd^K\rho^{2K}, \quad \E[\similarity_{ij}] = 0, \; \forall j \neq i.
\#
\end{proposition}

The intuition behind \Cref{prop:mean-sep} has already been discussed in \Cref{sect:main-intui} and illustrated in \Cref{fig:mean-sep}; the formal proof is given in \Cref{sect:mean}.
The next proposition bounds the variance.
\begin{proposition}\label{prop:var-control}
Define $\beta=\alpha+\epsilon/2.$ Suppose that 
$\rho^4 \ge \alpha+\epsilon$ and 
$$
\beta^M |\cJ|\ge 7D^2, \quad (\sqrt{\beta}/\rho^2)^M<1/2,\quad  (\poly(K)/\min\{d, n\})^{1/8}<1/2.
$$
Then there exists an absolute constant $C>0$ such that 
\# 
& \frac{\var(\similarity_{ii})}{\left(\E[\similarity_{ii}]\right)^2} \le C\left[\left(\frac{\sqrt{\beta}}{\rho^2} \right)^{M} + \left(\frac{\poly(K)}{\min\{d, n\}}\right)^{1/8}\right], \label{eq:var-bd-ii}\\
& \frac{\var(\similarity_{ij})}{\left(\E[\similarity_{ii}]\right)^2} \le CD\left[\left(\frac{\beta}{\rho^4} \right)^{K} + \left(\frac{\poly(K)}{\min\{d, n\}}\right)^{D/2}\right], \quad \forall j \neq i. \label{eq:var-bd-ij}
\#
\end{proposition}
The intuition behind the false-pair variance bound \eqref{eq:var-bd-ij} has already been discussed in \Cref{sect:main-intui}. In particular, the two dominant terms arise from the two configurations depicted in \Cref{fig:variance-illustrate}. The intuition behind the true-pair variance bound \eqref{eq:var-bd-ii} is similar, except that these structures can occur only at a branch pair; as a result, the exponents in the two terms are each reduced by a $\Theta(D)$ factor. The formal proof of \Cref{prop:var-control} involves intricate combinatorial arguments on trees, and is deferred to \Cref{sect:var}.

Together, \Cref{prop:mean-sep,prop:var-control} guarantee the success of the algorithm.
\begin{proof}[Proof of \Cref{thm:algo}]
We first show that, under the parameter specified in~\eqref{eq:wide-tree-1} with appropriate choice of constants $C_1,C_2,c_1>0$, \Cref{prop:var-control} implies the following bounds: 
\# 
& \frac{\var(\similarity_{ii})}{\left(\E[\similarity_{ii}]\right)^2} \le C\left[\left(\frac{\sqrt{\beta}}{\rho^2} \right)^{M} + \left(\frac{\poly(K)}{\min\{d, n\}}\right)^{1/8}\right] \le \frac{C}{\log^5 n}, \label{eq:var-bd-ii-pf}\\
& \frac{\var(\similarity_{ij})}{\left(\E[\similarity_{ii}]\right)^2} \le CD\left[\left(\frac{\beta}{\rho^4} \right)^{K} + \left(\frac{\poly(K)}{\min\{d, n\}}\right)^{D/2}\right] \le \frac{C}{n^3}, \quad \forall j \neq i. \label{eq:var-bd-ij-pf}
\#
First, note that $|\cJ| = (\alpha + o(1))^{-M} \ge [(\alpha + \epsilon/2)/(1+\epsilon/4)]^{-M}$ and recall that $M= K/D = (1/c_1)\log K$. Thus, if $c_1$ is sufficiently small so that $\log(1+\epsilon/4)/c_1 > 2$, then
\$ 
\beta^M |\cJ| \ge (1+\epsilon/4)^M = K^{\log(1+\epsilon/4)/c_1} \ge 7D^2. 
\$
Second, if $c_1 \le \log((\alpha+\epsilon)/(\alpha+\epsilon/2))/12= \log(\rho^2/\sqrt{\beta})/6$, then with $M = K/D = (1/c_1)\log K$ and $K=C_1 \log n$, we have 
\$ 
\left(\frac{\sqrt{\beta}}{\rho^2} \right)^{M} = K^{-\log(\rho^2/\sqrt{\beta})/c_1}
=(C_1 \log n)^{-\log(\rho^2/\sqrt{\beta})/c_1} \le \frac{1}{2\log^5 n}.
\$
Next, choose $C_1 \ge 4/\log((\alpha+\epsilon)/(\alpha+\epsilon/2)) = 4/\log(\rho^4/\beta)$. 
Then
\$ 
D\left(\frac{\beta}{\rho^4} \right)^{K} = \frac{c_1K}{\log K}\left(\frac{\beta}{\rho^4} \right)^{C_1 \log n }\le \frac{1}{2n^3}.
\$
Moreover, note that $d \ge (\log n)^{C_2}$ and $K=C_1 \log n.$
By choosing $C_2$ to be a sufficiently large constant, we have $\poly(K)/\min\{d, n\} \le (\log n)^{-C_2/2}$ and 
\$\left(\frac{\poly(K)}{\min\{d, n\}}\right)^{1/8} \le (\log n)^{-C_2/16} \le 
\frac{1}{2\log^5 n}.\$ 
Finally, 
substituting $K=C_1\log n$ and 
$D= c_1 K/\log K$ and choosing $C_2$ to be a sufficiently large constant, we get 
\$
D\left(\frac{\poly(K)}{\min\{d, n\}}\right)^{D/2} 
\le D( \log n)^{-C_2 D/4}
\le ( \log n)^{-C_2 C_1 c_1 \log n/(4 \log (C_1 \log n)) + 1}
\le \frac{1}{2n^3}.
\$

We now prove \Cref{thm:algo}. Recall that the constant $c \in (1/4, 3/4)$. Let $\nu = \E[\similarity_{ii}]$.
For any $j \neq i$, by Chebyshev's inequality and \eqref{eq:var-bd-ij-pf}, we obtain
\$ 
\prob{\similarity_{ij} \ge \tau} = \prob{\similarity_{ij} -\E[\similarity_{ij}]\ge (c-o(1))\nu} \le \frac{\var(\similarity_{ij})}{(c-o(1))^2\nu^2} = o(1/n^2).
\$
Applying a union bound over all pairs with $j \neq i$, we obtain
\$ 
\prob{\exists j\neq i\in[n]\colon \similarity_{ij} \ge \tau} = o(1).
\$
Thus, with probability at least $1-o(1)$, we have $\similarity_{ij} < \tau$ for all $j\neq i$ and all $i\in[n]$. By our construction of $\widehat{\pi}$ and $I$, it follows that $I = \{i\colon \similarity_{ii} >\tau \}$ and $\widehat{\pi}\,|\,_I = \pi$ with probability at least $1-o(1)$.
For the true pairs, Chebyshev's inequality yields
\$ 
\prob{\similarity_{ii} \le \tau} \le \prob{\left|\similarity_{ii} - \nu \right|\ge (1-c)\nu} \le \frac{\var(\similarity_{ii})}{(1-c)^2\nu^2} \triangleq \gamma,
\$
where $\gamma\le 1/\log^4 n$ by \eqref{eq:var-bd-ii-pf}.
Hence $\E\left[n-\left|I\right|\right] \le \gamma n$. 
Applying Markov's inequality, we have
\$ 
\prob{n-\left|I\right| \ge \sqrt{\gamma}n} \le \frac{\E\left[n-\left|I\right|\right]}{\sqrt{\gamma}n} \le \sqrt{\gamma} = \frac{1}{\log^2 n}.
\$
Consequently, we conclude that $|I| \ge n - n/\log^2 n$ with probability at least $1-o(1)$.
\end{proof}

\section{Joint Moments Calculation} \label{sect:join-mom}
We first establish a general result for the joint moments of $X$ and $Y$, which will be useful for bounding the mean and variance.  
Given any bipartite multigraphs $G$ and $H$ with the same number of edges, we compute the expectation
\#\label{eq:mean-t12}
\E\left[X^G \cdot Y^H\right] \,.
\#
We show that this expectation can be bounded in terms of combinatorial properties of the \emph{$n$-union graph} of $G$ and $H$, which we will denote by $G\cup_{n} H$ and define as follows: 
\begin{enumerate}
\item Edges and $d$-nodes from $G$ are colored red. Edges and $d$-nodes from $H$ are colored blue.
\item $n$-nodes from $G$ and $H$ are merged according to their labels; $d$-nodes from $G$ are kept distinct from $d$-nodes from $H$. 
\end{enumerate}

Specifically, we show that the joint moments can be bounded as a function of the \emph{alternating circuit decompositions} of $G\cup_{n}H$ (to be defined later).
To arrive at this decomposition, we first take expectation over the Gaussian random variables $X$ and $Z$ using Isserlis' (or Wick's) theorem, and then take expectation over the random orthogonal matrix $Q$ via Weingarten calculus.

\subsection{Joint Moments of Gaussian Random Variables} Since the entries of $X$ and $Z$ are symmetrically distributed around zero, their odd-degree moments vanish. Thus, we only need to consider even-degree moments. 
The following lemma, known as Isserlis' (or Wick's) theorem \citep{isserlis1918formula,wick1950evaluation}, provides a formula for computing the moments of joint Gaussian distributions. 
Fix an even integer $N$. 
We define $\cM(N)$ as the set of all perfect matchings (pair partitions) $\ttm$ of the index set $[N]$, where each matching $\ttm\in \cM(N)$ can be written as
\$
\ttm = \{\{\ttm(1), \ttm(2)\}, \cdots, \{\ttm(N-1), \ttm(N)\}\},
\$
with $\{\ttm(1), \cdots, \ttm(N)\}$ being a permutation of $[N]$.
\begin{lemma}[Isserlis'/Wick's Theorem]\label{lem:isserlis} Let $(\xi_1, \cdots, \xi_N)$ be a multivariate Gaussian random vector with zero mean and $N$ be even. 
Then we have
\$
\E\left[\prod_{i=1}^N \xi_i\right] = \sum_{\ttm \in \cM(N)} \prod_{\{i,j\} \in \ttm} \E[\xi_i \xi_j],
\$
where the summation runs over $\cM(N)$, the set of perfect matchings (pair partitions) $\ttm$ of the index set $[N]$, and each product is over the pairs defined by the matching $\ttm$.    
\end{lemma} 
Isserlis' theorem converts higher-order moments of Gaussian random variables into sums of products of pairwise covariances. This allows us to compute the joint moments of $X$ and $Y$ conditional on $Q$. 
Specifically, we define $\widetilde\cM(G, H)$ as the set of perfect matchings on the \emph{edge set} of the $n$-union graph $G\cup_n H$. For each $\ttm\in \widetilde\cM(G, H)$, let $\ttmbi = \ttm\cap (E(G)\times E(H))$ be the set of bichromatic pairs.\footnote{For clarity, we always write pairs $\{(u,a), (v,b)\} \in \ttm$ with one edge from $E(G)$ and the other from $E(H)$ in an ordered manner, where the first component is from $E(G)$ corresponding to $X$, and the second component is from $E(H)$ corresponding to $Y$. Such pairs are denoted by $\{(u,a), (v,b)\} \in E(G)\times E(H)$.} 
The remaining pairs in $\ttm\setminus\ttmbi$ are called monochromatic pairs, since they consist of two edges both in $E(G)$ or both in $E(H)$.
We further define a subset $\cM(G, H)\subset \widetilde\cM(G, H)$ consisting of those perfect matchings such that all bichromatic pairs of edges share an $n$-node in $G \cup_{n} H$, and all monochromatic pairs share both endpoints.
That is, for any $\ttm\in\cM(G, H)$: 
\begin{enumerate}
\item For each pair $\{(u,a), (v,b)\} \in \ttmbi$, it must hold that $u = v$. 
\item For each pair $\{(u,a), (v,b)\} \in \ttm\setminus\ttmbi$, it must hold that  $u = v$ and $a=b$. 
\end{enumerate} 
Note that if both $G$ and $H$ are simple graphs without parallel edges, then no monochromatic pairs exist, and ${\cM}(G,H) \neq \emptyset$ only if every $n$-node in $G\cup_n H$ is incident to the same number of edges from $G$ as from $H$.

The lemma below computes the joint moments in $X$ and $Y$, conditioned on $Q$, by applying Isserlis' theorem.
Since every matching $\ttm\in \cM(G,H)$ consists of pairs of edges sharing an $n$-node, we abuse notation and write $Q^{\ttm} = \prod_{\{(u,a),(u,b)\}\in \ttm}Q_{ab}$, conflating the set of pairs of edges $\ttm$ with the corresponding set of pairs of their $d$-endpoints.
\begin{lemma}\label{lem:isserlis-x-z}
    For any bipartite multigraphs $G$ and $H$, we have
\$
\E_{X,Z}\left[X^G \cdot Y^H \given Q\right] = \sum_{\ttm \in \cM(G, H)} \rho^{|\ttmbi|} Q^{\ttmbi}.
\$
\end{lemma}
\Cref{lem:isserlis-x-z} implies that any $G$ and $H$ with ${\cM}(G,H) = \emptyset$ yield a joint moment of $0$. 
\begin{proof}[Proof of \Cref{lem:isserlis-x-z}]
Since $X,Y \, | \, Q$ follow a joint Gaussian distribution with zero mean, we apply Isserlis' Theorem 
by first computing the pairwise correlation of entries. 
First, since all entries of $X$ are independent, we have 
\#\label{eq:pair-x}
\E[X_{ua} X_{vb}] = \Ind_{\{(u,a)=(v,b)\}}.
\#
Similarly, we have 
\#\label{eq:pair-y}
\E\left[Y_{ua} Y_{vb} \given Q\right] = \Ind_{\{(u,a)=(v,b)\}}.
\#
Next, we compute the conditional correlation between $X_{ua}$ and $Y_{vb}$ and obtain
\#\label{eq:pair-x-y}
\E\left[X_{ua} Y_{vb}\given Q\right] & = \E\left[X_{ua} \left(\rho \sum_{c \in [d]} Q_{cb} X_{vc} + \sqrt{1-\rho^2}Z_{v b}\right)\given Q\right] \notag\\
&= \rho\E\left[X_{ua} \cdot \sum_{c \in [d]} Q_{cb} X_{vc} \given Q\right]\notag\\
& = \rho Q_{ab} \Ind_{\{u=v\}},
\#
where the second equality uses independence of $X$ and $Z$ and the last equality follows from \eqref{eq:pair-x}.

Recall that $\widetilde\cM(G, H)$ is the set of perfect matchings among all edges in $G$ and $H$.
\Cref{lem:isserlis} and the pairwise covariances in \eqref{eq:pair-x}, \eqref{eq:pair-y}, and \eqref{eq:pair-x-y} imply that
\$
\E\left[X^G \cdot Y^H\given Q\right] 
& = \sum_{\ttm \in \widetilde\cM(G, H)} \prod_{\substack{\{e,f\} \in \ttm \setminus \ttmbi\\ e,f \in E(G)}}  \E\left[X_{e} X_{f}\right] \prod_{\substack{\{e,f\} \in \ttm \setminus \ttmbi\\ e,f \in  E(H)}}\E\left[Y_{e} Y_{f}\given Q\right] \prod_{\substack{\{e,f\} \in \ttmbi}} \E\left[X_{e} Y_{f}\given Q\right]  \notag\\
&  = \sum_{\ttm \in \widetilde\cM(G, H)} \prod_{\substack{\{e,f\} \in \ttm \setminus \ttmbi \\ e,f\in E(G)}} \Ind_{\{e=f\}} \prod_{\substack{\{e,f\} \in \ttm \setminus \ttmbi\\ e, f \in E(H)}} \Ind_{\{e=f\}} \prod_{\substack{\{(u,a), (v,b)\} \in \ttmbi}} \rho Q_{ab} \Ind_{\{u=v\}}.
\$
The conclusion now follows because the indicators enforce the conditions defining ${\cM}(G,H)$.
\end{proof}

It now remains to take the expectation over $Q$.

\subsection{Weingarten Calculus for Orthogonal Groups.}
Weingarten calculus is a method for computing moments of matrices sampled from the Haar measure over the Orthogonal or Unitary group.
We use Weingarten calculus to take the expectation over $Q$, as in a number of recent works in high-dimensional/computational statistics \cite{chen2022toward,kunisky2024tensor} (our specific use is somewhat different).
We now summarize several key results from Weingarten calculus for computing moments of entries of $Q$ \citep{collins2017weingarten}. 
Since the Haar measure over orthogonal matrices is symmetric about zero, all odd-degree moments vanish, and it suffices to consider only even-degree moments.

Suppose we are given a bipartite multigraph $J = ([d],[d],E)$, whose left-hand vertex set is identified with a ``red'' copy of $[d]$ and whose right-hand vertex set is identified with a ``blue'' copy of $[d]$. Recall that $Q^J$ denotes the product of matrix entries corresponding to the edges of $J$.
It will turn out that the moment $\E_{Q\sim \cO(d)}[ Q^J]$ can be expressed in terms of features of the \emph{circuit decompositions} of $J$. To describe this relation, we introduce the following notion.

\begin{definition}[Circuit Decomposition] 
Let $J$ be a multigraph, and in what follows treat multi-edges as distinct parallel edges.

A \emph{closed walk} is a sequence of edges $(e_1, e_2, \cdots, e_k)$ for which there exists a sequence of vertices $(v_1, v_2, \cdots, v_{k+1})$ such that $e_i$ is an edge between $v_i$ and $v_{i+1}$ for each $i \in [k]$, and $v_1 = v_{k+1}$.

A \emph{circuit} is a closed walk in which all edges are distinct.

Two closed walks that traverse the same set of edges in a different order are considered distinct as circuits, unless one can be obtained from the other by a cyclic shift or by reversing the order of traversal. In other words, the walks $(e_1, e_2, \cdots, e_k)$, $(e_i, e_{i+1}, \cdots, e_k, e_1, \cdots, e_{i-1})$, and $(e_i, e_{i-1}, \cdots, e_1, e_k, \cdots, e_{i+1})$ for any $i \in [k]$ are considered the same circuit.

A \emph{circuit decomposition} of $J$ is a set of circuits that partitions $E(J)$. Two circuit decompositions are considered distinct if they contain different circuits. We denote by $\cC(J)$ the set of all circuit decompositions of $J$.

\end{definition}

A graph $J$ admits a circuit decomposition if and only if all of its vertices have even degrees. \Cref{fig:weing-example} shows an example of a graph $J$ and its circuit decompositions.
\input{fig/weing-example}

\paragraph{Weingarten calculus.}
Consider any circuit decomposition $\bc$ in which every circuit has even length. 
For each integer $\ell\ge 1$, let $\mu_\ell(\bc)$ denote the number of circuits of length $2\ell$ in $\bc$, and write $\mu(\bc) = (\mu_\ell(\bc))_{\ell\ge 1}$ for the \emph{circuit type} of $\bc$. We take a vertex-disjoint union of these circuits in $\bc$, and assign each vertex a distinct label so that vertices belonging to different circuits have different labels.  Let $G$ be a bipartite multigraph obtained in this way.
We define the \emph{Weingarten function} associated with the circuit type $\mu(\bc)$ by 
\#\label{eq:weingarten-fnc}
\mathrm{Wg}(\mu(\bc), d) \triangleq \E_{Q\sim \cO(d)}\left[Q^G\right].
\#
By the bi-invariance of the Haar measure, the value in \eqref{eq:weingarten-fnc} is invariant under permutations of vertex labels and therefore depends only on the circuit type $\mu(\bc)$, rather than on the specific labeling of $G$.

Now consider an arbitrary bipartite graph $J = ([d],[d],E)$. In general, $J$ may admit multiple circuit decompositions, each consisting of circuits of even length, see, e.g., \Cref{fig:weing-example}. For any decomposition $\bc\in \cC(J)$, its circuit type $\mu(\bc)$ determines a Weingarten function value $\mathrm{Wg}(\mu(\bc), d)$ defined in \eqref{eq:weingarten-fnc}. The following Weingarten calculus expresses the moment of $Q^J$ as a sum over all such circuit decompositions.
\begin{lemma}[{\citep[Lemma~4.1]{collins2017weingarten}}]\label{lem:weing-1}
    For a bipartite graph $J = ([d],[d],E)$, we have
    \$
    \E_{Q\sim \cO(d)}\left[Q^J\right] = \sum_{\bc\in\cC(J)} \mathrm{Wg}(\mu(\bc), d),
    \$
    where $\mu(\bc)$ denotes the circuit type of the decomposition $\bc$, and $\mathrm{Wg}(\mu(\bc), d)$ is the corresponding Weingarten function.
\end{lemma}
In the special case where every vertex of $J$ has degree exactly $2$, the graph $J$ itself is a disjoint union of cycles and thus admits a unique circuit decomposition $\bc$. In this case, \Cref{lem:weing-1} reduces to the definition in \eqref{eq:weingarten-fnc}.

\paragraph{Asymptotic values of 
the Weingarten function.}
Explicit formulas for $\mathrm{Wg}(\mu(\bc), d)$ can be derived directly from \eqref{eq:weingarten-fnc}.  
For our purposes, however, it suffices to use asymptotic bounds. The total number of edges of the graph is $\sum_{\ell\ge 1} 2\ell \mu_\ell(\bc)$, since each of the $\mu_\ell(\bc)$ circuits contains $2\ell$ edges. We denote this quantity by $2k$.
Define $\gamma_{\mu(\bc)} = \prod_{\ell\ge 1} \big[\Cat(\ell-1)\big]^{\mu_\ell(\bc)}$, where $\Cat(i) = \frac{1}{i+1}\binom{2i}{i}$ for $i\ge 0$ denotes the Catalan number. Since $\Cat(i) \le 4^i$ for all $i$, we have $1 \le \gamma_{\mu(\bc)} \le 4^{k-\sum_{\ell\ge 1} \mu_{\ell}(\bc)}$. The Weingarten function $\mathrm{Wg}(\mu(\bc), d)$ can be bounded as follows.

\begin{lemma}[{\citep[Theorem 4.11, Lemma 4.12]{collins2017weingarten}}]\label{lem:weing-fnc}
Given any circuit type $\mu(\bc)$, let $k\triangleq\sum_{\ell\ge 1} \ell \mu_\ell(\bc)$. So long as $d> 12 k^{7/2}$, we have
\$
\mathrm{Wg}(\mu(\bc),d) = \left(1 \pm  \frac{\poly(k)}{d}  \right)\gamma_{\mu(\bc)} (-1)^{k-\sum_{\ell\ge 1} \mu_{\ell}(\bc)} d^{-2k+\sum_{\ell\ge 1} \mu_{\ell}(\bc)}. 
\$
\end{lemma}
The leading factor $d^{-2k+\sum_{\ell\ge 1} \mu_{\ell}(\bc)}$ in \Cref{lem:weing-fnc} implies that circuit decompositions $\bc$ containing a larger number of circuits $\sum_{\ell\ge 1} \mu_{\ell}(\bc)$ tend to yield Weingarten function values of larger magnitude.

In addition, \Cref{lem:weing-fnc} implies that the Weingarten function admits an approximate factorization over edge-disjoint graphs. Specifically, for two edge-disjoint graphs $J_1$ and $J_2$ with union $J=J_1\cup J_2$, and their circuit decompositions $\bc_1$ and $\bc_2$ with union $\bc = \bc_1\cup \bc_2$, it holds that 
\#\label{eq:twg-mult}
\wg(\mu(\bc),d)= (1 + o(1))\wg(\mu(\bc_1),d)\wg(\mu(\bc_2),d).
\# 
This follows from the fact that, in the approximation, the quantities $\mu_\ell(\bc)$ appear only in the exponents, and the circuit counts are additive for edge-disjoint parts; that is, $\mu_\ell(\bc) = \mu_\ell(\bc_1) + \mu_\ell(\bc_2)$ for all $\ell\ge 1$.

\paragraph{Joint moments via Weingarten calculus.}
We now apply the Weingarten calculus to compute joint moments. Recall from \Cref{lem:isserlis-x-z} that each matching $\ttm \in \cM(G,H)$ corresponds to a monomial $Q^{\ttmbi}=\prod_{\{(u,a), (u,b)\} \in \ttmbi} Q_{ab}$.
We represent this monomial using a bipartite multigraph $J_{\ttmbi} = ([d], [d], E(J_{\ttmbi}))$, whose edge multiset is defined by 
\#\label{eq:def-J-m*}
E(J_{\ttmbi}) = \{(a,b)\in [d]\times [d]\colon 
\{(u,a), (u,b)\} \in \ttmbi \}.
\#
Since each pair $\{(u,a),(u,b)\} \in \ttmbi$ corresponds to a bichromatic $2$-path in $G \cup_{n} H$, we can think of $J_{\ttmbi}$ as the bipartite graph obtained by contracting each such $2$-path into a single edge with one red and one blue $d$-endpoint.
With this representation, the joint moments in \eqref{eq:mean-t12} can be expressed using circuit decompositions of $J_{\ttmbi}$ as follows.
\begin{proposition}\label{prop:mom-cir}
For any bipartite multigraphs $G$ and $H$, we have
\$ 
\E\left[X^G \cdot Y^H\right]  = \sum_{\ttm \in \cM(G, H)} \rho^{|\ttmbi|} \cdot \sum_{\bc\in\cC(J_{\ttmbi})} \mathrm{Wg}(\mu(\bc), d).
\$
\end{proposition}
\begin{proof}
Following from \eqref{eq:mean-t12} and \Cref{lem:isserlis-x-z}, we have
\$ 
\E_{X,Y}\left[X^G \cdot Y^H\right] 
& = \sum_{\ttm \in \cM(G, H)} \rho^{|\ttmbi|} \cdot \E\left[Q^{\ttmbi}\right].
\$
For each $\ttm\in \cM(G, H)$, the Weingarten calculus in  \Cref{lem:weing-1} 
implies that
\$ 
\E\left[Q^{\ttmbi}\right] & = \sum_{\bc\in\cC(J_{\ttmbi})} \mathrm{Wg}(\mu(\bc), d),
\$
where $J_{\ttmbi}$ is the bipartite multigraph defined in \eqref{eq:def-J-m*}.
\end{proof}

Combining \Cref{prop:mom-cir} with \Cref{lem:weing-fnc}, we can further obtain bounds for the joint moments. To simplify the presentation, we state the final result in the next section using \emph{alternating circuit decompositions} of the $n$-union of $G$ and $H$.

\subsection{Joint Moments Calculation via Alternating Circuit Decompositions}

For each $\ttm \in \cM(G,H)$, each circuit decomposition $\bc$ of $J_{\ttmbi}$ induces a circuit decomposition of the $n$-union $G \cup_{n} H$ as follows: 
\begin{enumerate}
    \item Each circuit of length $2\ell$ maps to a circuit of length $4\ell$ in $G\cup_{n} H$, by un-contracting each edge of $J_{\ttmbi}$ to the original bichromatic 2-path from which it arose. 
    These circuits are ``alternating'' in that whenever they reach an $n$-node, the edge color alternates from red to blue or vice-versa.
    \item Each pair $\{(u,a),(u,a)\} \in \ttm \setminus \ttmbi$ gives a circuit of length $2$ in $G\cup_{n}H$.
\end{enumerate}
We define the set of circuit decompositions that can be induced in $G\cup_{n}H$ by such $J_{\ttmbi}$ \emph{alternating circuit decompositions}, denoted by $\cC(G\cup_{n} H)$, where ``alternating'' refers to the fact that the edge color in circuits of length $\ge 4$ alternates between blue and red every time an $n$-node is reached.

\begin{definition}[Alternating Circuit Decompositions] \label{def:alt-cir-decomp}
An alternating circuit decomposition of the $n$-union $G\cup_{n} H$ is a circuit decomposition where each circuit satisfies:
\begin{enumerate}
\item Within each circuit of length $4$ or greater, any two edges incident to each $n$-node have distinct colors;
\item Circuits of length $2$ are multi-edges in $G\cup_{n} H$, where both edges have the same color.
\end{enumerate}    
\end{definition}
Recall that we treat each parallel copy of a multi-edge as distinct. Thus, two alternating circuit decompositions whose circuits differ only by the choice or order of edges from a multi-edge are considered distinct in the set $\cC(G\cup_{n} H)$. For $\cC(G\cup_{n} H)$ to be nonempty, graphs $G$ and $H$ must satisfy the following conditions:
\begin{enumerate}
\item All $d$-nodes from both $G$ and $H$ have even degrees;
\item After removing all $2$-circuits, the remaining degree sequences of $n$-nodes from $G$ and $H$ are the same. 
\end{enumerate}

With a slight abuse of notation, for any $\bc\in \cC(G\cup_{n} H)$, we let $\mu_{\ell}(\bc)$ denote the number of alternating circuits of length $4\ell$ in $G\cup_{n} H$ for $\ell\ge 1$, and $\mu(\bc)= \{\mu_{\ell}(\bc)\}_{\ell\ge 1}$ for the circuit type. Each alternating circuit of length $4\ell$ arises from a circuit of length $2\ell$ in $J_{\ttmbi}$. Thus, $\mu(\bc)$ also corresponds to the circuit type of the original circuit decomposition of  $J_{\ttmbi}$. Accordingly, we denote by $\mathrm{Wg}(\mu(\bc), d)$ the Weingarten function associated with this circuit type.

We consider $G$ and $H$ both with $2k$ edges.  Let $2\kappa(\bc)$ be the number of circuits of length $2$. 
Since every circuit of length $\ge 4$ alternates between edges of $G$ and $H$, each such $\bc$ contains exactly $\kappa(\bc)$ $2$-edge circuits from $G$ and $\kappa(\bc)$ $2$-edge circuits from $H$.  
The following theorem summarizes the joint moments computation via alternating circuit decompositions.

\begin{theorem}[Joint Moments Calculation] \label{thm:altcir}
For any bipartite multigraphs $G$ and $H$ both with $2k$ number of
 edges, 
\$ 
\E\left[X^G \cdot Y^H\right] & = \sum_{\bc \in \cC(G\cup_{n} H)} \rho^{2(k-\kappa(\bc))} \mathrm{Wg}(\mu(\bc),d ),\\
\intertext{and as long as $d > 12k^{7/2}$, }
& = \left(1 \pm \frac{\poly(k)}{d}\right)\sum_{\bc \in \cC(G\cup_{n} H)} \gamma_{\mu(\bc)}  \left(\frac{\rho^2}{d}\right)^{k-\kappa(\bc)}\left(-\frac{1}{d}\right)^{k - \kappa(\bc)-\sum_{\ell\ge 1}\mu_\ell(\bc)}.
\$
\end{theorem}
The proof follows immediately from combining~\Cref{prop:mom-cir} and~\Cref{lem:weing-fnc}, noting that the total number of edges in alternating circuits of length $\ge 4$ is $4\sum_{\ell\ge 1}\ell\mu_\ell(\bc) = 4k-4\kappa(\bc)$. Moreover, we have the following bound on the total number of alternating circuits:
\$
\sum_{\ell\ge 1}\mu_\ell(\bc) = \mu_1(\bc) + \sum_{\ell\ge 2}\mu_\ell(\bc) \le  \mu_1(\bc) + \sum_{\ell\ge 2}\ell\mu_\ell(\bc)/2 = \mu_1(\bc) + (k - \kappa(\bc) - \mu_1(\bc))/2.
\$
Consequently, we obtain $k - \kappa(\bc) -\sum_{\ell\ge 1}\mu_\ell(\bc) \ge (k - \kappa(\bc) - \mu_1(\bc))/2$. Recall also that $\gamma_{\mu(\bc)} \le 4^{k - \kappa(\bc) - \sum_{\ell\ge 1}\mu_\ell(\bc)}$. Thus, each summand in \Cref{thm:altcir} is bounded by
\#\label{eq:bound-mu1}
\rho^{2(k-\kappa(\bc))} \left|\mathrm{Wg}(\mu(\bc),d )\right| \le \left(1 + \frac{\poly(k)}{d}\right)\left(\frac{\rho^2}{d}\right)^{k-\kappa(\bc)} \left(\frac{4}{d}\right)^{(k - \kappa(\bc) - \mu_1(\bc))/2}.
\#
Therefore, the key parameter determining each decomposition's contribution is $k - \kappa(\bc) - \mu_1(\bc)$. In what follows, we will show that decompositions $\bc$ satisfying $\mu_1(\bc) = k - \kappa(\bc)$ yield the leading contribution to the mean or variance, while those with $\mu_1(\bc) < k - \kappa(\bc)$ contribute negligibly. 

Since \Cref{thm:altcir} will be used repeatedly throughout the paper to approximate the Weingarten function for graphs with $2K$ or $4K$ edges, where $2K$ is the number of edges in graphs from $\cT$ (see \Cref{def:t}), we henceforth assume that $d > 12 (2K)^{7/2}$. For brevity, we omit this standing assumption from subsequent theorem and proposition statements in \Cref{sect:mean,sect:var}.

\section{Mean Calculation} \label{sect:mean} 
This section establishes \Cref{prop:mean-sep}.
We prove a slightly more general result. Let $\cH$ be any set of rooted simple bipartite graphs on $[n] \times [d]$ with $2K$ edges, such that the root node is an $n$-node with strictly larger degree than every non-root node, and every $d$-node has degree $2$. The tree family $\cT$ from \Cref{def:t} is one such set. We compute the mean of $\similarity_{ij}^{\cH}$, defined in \eqref{def:score-ij}, separately for true pairs ($j = i$) and false pairs ($j \neq i$). Then \Cref{prop:mean-sep} follows directly by taking $\cH=\cT$.

\subsection{True Pairs}
For any $i\in[n]$, we have
\#\label{eq:mean-1}
\E\left[\similarity_{ii}^{\cH}\right] 
& = \sum_{T \in \cH} \aut(T) 
\sum_{\substack{T_1,T_2 \cong T\\\mathrm{root}(T_1) = \mathrm{root}(T_2) =  i}}\E\left[X^{T_1} Y^{T_2}\right] \notag\\
& = \sum_{T \in \cH} \aut(T) 
\sum_{\substack{T_1,T_2 \cong T\\\mathrm{root}(T_1) = \mathrm{root}(T_2) =  i}} \sum_{\bc\in \cC(T_1\cup_n T_2)} \rho^{2K} \wg(\mu(\bc), d), 
\#
where the last equality follows from the Weingarten calculus in \Cref{thm:altcir}, noting that each of $T_1$ and $T_2$ has $2K$ edges. Moreover, since $T_1$ and $T_2$ are simple graphs, their $n$-union is also simple (recall the $d$-nodes in $T_1$ and $T_2$ have distinct colors, so no multi-edge can be formed), and thus contains no $2$-circuits. Hence, every $\bc \in \cC(T_1\cup_n T_2)$ satisfies $\kappa(\bc)=0$.

Recall the intuition following \Cref{lem:weing-fnc} and \eqref{eq:bound-mu1}: among all decompositions $\bc$, those with the largest $\mu_{1}(\bc)$, equivalently containing the most $4$-circuits, make the dominant contribution to the mean, while the contributions of the remaining decompositions are negligible. We now use this observation to derive the order of $\E[\similarity_{ii}^{\cH}]$. 
We also define 
\#\label{eq:exp-score-abs}
\nu^{\cH}\triangleq \sum_{T \in \cH} \aut(T) 
\sum_{\substack{T_1,T_2 \cong T\\\mathrm{root}(T_1) = \mathrm{root}(T_2)}} \sum_{\bc\in \cC(T_1\cup_n T_2)} \rho^{2K} \left|\wg(\mu(\bc), d)\right|
\#
by replacing $\wg(\cdot)$ in \eqref{eq:mean-1} with its absolute value, which will be used in the variance bound.

\begin{proposition}\label{prop:mean-true}
Fix any $i\in[n]$. 
If $K=o(\sqrt{n})$, we have
\$ 
\E\left[\similarity_{ii}^{\cH}\right] & = \left(1 \pm \frac{\poly(K)}{d}\right) 
\nu^{\cH} = \left(1 \pm \frac{\poly(K)}{\min\{n,d\}}\right) |\cH|n^{K} d^{K} \rho^{2K}.
\$
\end{proposition}
\begin{proof}
We suppress the root index $i$ to simplify notation. By \eqref{eq:mean-1}, we have
\$
\E\left[\similarity_{ii}^\cH\right] = \sum_{T \in \cH} \aut(T) 
\sum_{\substack{T_1,T_2 \cong T}} \sum_{\bc\in \cC(T_1\cup_n T_2)} \rho^{2K} \wg(\mu(\bc), d). 
\$ 
Fix $T\in \cH$. Note that $T$ has $2K$ edges and recall from \Cref{thm:altcir} that
\$
\rho^{2K} \wg(\mu(\bc), d) = \left(1\pm \frac{\poly(K)}{d}\right) \left(\frac{\rho^2}{d}\right)^{K} \gamma_{\mu(\bc)}  \left(-\frac{1}{d}\right)^{K-\sum_{\ell\ge 1}\mu_\ell(\bc)}.
\$
We now analyze the contribution according to the value of $\mu_1(\bc)$. 
\paragraph{Leading term: $\mu_1(\bc)=K$.}
We first show that the dominant contribution arises from decompositions $\bc$ satisfying $\mu_1(\bc) = K$, as illustrated in \Cref{fig:mean-sep}.
Whenever $\mu_1(\bc)=K$, we have $\gamma_{\mu(\bc)} = 1$ and $\sum_{\ell\ge 1}\mu_\ell(\bc) = K$ and therefore $\gamma_{\mu(\bc)}(-{1}/{d})^{K-\sum_{\ell\ge 1}\mu_\ell(\bc)} = 1$. Thus, for such $\bc$, we have
\#\label{eq:mean-ii-mu1-k} 
\rho^{2K} \wg(\mu(\bc), d) = \left(1\pm \frac{\poly(K)}{d}\right) \left(\frac{\rho^2}{d}\right)^{K}.
\#
It follows that
\$ 
& \sum_{T \in \cH} \aut(T) 
\sum_{\substack{T_1,T_2 \cong T}} \sum_{\substack{\bc\in \cC(T_1\cup_n T_2) \\ \mu_1(\bc) = K}} \rho^{2K} \wg(\mu(\bc), d) \\
& \qquad = \left(1\pm \frac{\poly(K)}{d}\right)  \left(\frac{\rho^2}{d}\right)^{K} \sum_{T \in \cH} \aut(T) 
\sum_{T_1 \cong T} \left|\{(T_2, \bc)\colon T_2\cong T_1, \bc \in \cC(T_1\cup_{n} T_2), \mu_1(\bc) = K \}\right|.
\$

For a fixed labeled tree $T_1$, the condition $\mu_1(\bc) = K$ forces complete overlap of all $K$ edge pairs, each sharing a common $d$-node. Hence $T_2$ must coincide with $T_1$ on all $n$-nodes, and there is a unique decomposition $\bc$ with $\mu_1(\bc) = K$. Thus, the only freedom in choosing $(T_2, \bc)$ lies in labeling the $K$ $d$-nodes of $T_2$, giving
\$ 
\left|\{(T_2, \bc)\colon T_2\cong T_1, \bc \in \cC(T_1\cup_{n} T_2), \mu_1(\bc) = K \}\right| = \binom{d}{K}K!.
\$
It remains to count the number of labeled trees $T_1 \cong T$. We first label the $K$ $n$-nodes using labels from $[n]\setminus \{i\}$, excluding the root label $i$, with $\binom{n-1}{K}K!/\aut(T)$ distinct choices. For each such labeling, we label the $K$ $d$-nodes, with additional $\binom{d}{K}K!$ choices. Thus, we have
\begin{align}
\left|\{T_1 \cong T\}\right| = \frac{\binom{n-1}{K} K!}{\aut(T)} \binom{d}{K}K!. \label{eq:counting_T_1}
\end{align}
Combining the above, we conclude that
\#\label{eq:mean-true-K-lead}
& \sum_{T \in \cH} \aut(T) 
\sum_{\substack{T_1,T_2 \cong T}} \sum_{\substack{\bc\in \cC(T_1\cup_n T_2) \\ \mu_1(\bc) = K}} \rho^{2K} \wg(\mu(\bc), d) \notag\\ 
&\qquad  = \left(1\pm \frac{\poly(K)}{d}\right)  \left(\frac{\rho^2}{d}\right)^{K}\sum_{T \in \cH}  \aut(T) \frac{\binom{n-1}{K} K!}{\aut(T)} \binom{d}{K}K! 
 \binom{d}{K}K! \notag \\
&\qquad = \left(1\pm \frac{\poly(K)}{d}\right)|\cH| \left(\frac{\rho^2}{d}\right)^{K}\binom{n-1}{K} K! \left[\binom{d}{K}K!\right]^2.
\#
Since the right-hand side of \eqref{eq:mean-ii-mu1-k} is positive, it follows that $\wg(\mu(\bc), d) > 0$, and hence $\wg(\mu(\bc), d) = \left|\wg(\mu(\bc), d)\right|$ for all $\bc$ such that $\mu_1(\bc) = K$. Therefore, \eqref{eq:mean-true-K-lead} continues to hold when $\wg(\cdot)$ is replaced by its absolute value.

\paragraph{Remaining terms: $\mu_1(C)\le K -2$.}
We next bound the contribution of all terms with $\mu_1(C) \le K-2$. In this case, $K-\sum_{r\ge1}\mu_r(C) > 0$, so each such term gets an additional factor of at least $1/d$, making it negligible compared to the leading term. By \eqref{eq:bound-mu1}, we have
\$ 
\rho^{2K} \left|\mathrm{Wg}(\mu(\bc),d )\right| \le \left(1+ \frac{\poly(K)}{d}\right)\left(\frac{\rho^2}{d}\right)^{K} \left(\frac{4}{d}\right)^{(K - \mu_1(\bc))/2}.
\$
Thus, for any $\mu_1\le K-2$, we have
\#\label{eq:mean-small}
\sum_{T \in \cH} \aut(T) 
\sum_{\substack{T_1,T_2 \cong T}} &  \sum_{\substack{\bc\in \cC(T_1\cup_n T_2) \\ \mu_1(\bc) = \mu_1}} \rho^{2K} \left|\mathrm{Wg}(\mu(\bc),d )\right| \le \left(1+ \frac{\poly(K)}{d}\right) \left(\frac{\rho^2}{d}\right)^{K} \left(\frac{4}{d}\right)^{(K - \mu_1)/2} 
\\
& \qquad \times \sum_{T \in \cH} \aut(T) \sum_{T_1 \cong T} \left|\{(T_2, \bc)\colon T_2\cong T_1, \bc \in \cC(T_1\cup_{n} T_2), \mu_1(\bc) = \mu_1 \}\right|. \notag
\#
Fix any $T_1$. \Cref{lem:bound-mu1-le-k-2} bounds the number of feasible pairs $(T_2, \bc)$, where the $d$-nodes in $T_2$ are unlabeled. Since there are $\binom{d}{K}K!$ ways to label the $K$ $d$-nodes in $T_2$, applying \Cref{lem:bound-mu1-le-k-2} yields
\#\label{eq:mean-t2} 
\left|\{(T_2, \bc)\colon T_2\cong T_1, \bc \in \cC(T_1\cup_{n} T_2), \mu_1(\bc) = \mu_1 \}\right| \le (4K)^{K-\mu_1} \binom{d}{K}K!.
\#
Recalling that $|\{T_1 \cong T\}| = \binom{n-1}{K} K! \binom{d}{K}K!/\aut(T)$ given in~\eqref{eq:counting_T_1}, we have
\# \label{eq:mean-22}
& \aut(T) \sum_{T_1 \cong T} \left|\{(T_2, \bc)\colon T_2\cong T_1, \bc \in \cC(T_1\cup_{n} T_2), \mu_1(\bc) = \mu_1 \}\right| \notag\\
& \qquad \le \aut(T) \cdot \frac{\binom{n-1}{K} K!}{\aut(T)} \binom{d}{K}K! \cdot (4K)^{K-\mu_1} \binom{d}{K}K! \notag\\
& \qquad = \binom{n-1}{K} K! \bigg[\binom{d}{K}K!\bigg]^2 (4K)^{K-\mu_1}. 
\#
Thus, combining \eqref{eq:mean-small} and \eqref{eq:mean-22}, and summing over all $\mu_1\le K-2$, we obtain 
\#\label{eq:bound-mu1-small} 
& \sum_{T \in \cH} \aut(T) 
\sum_{\substack{T_1,T_2 \cong T}}   \sum_{\substack{\bc\in \cC(T_1\cup_n T_2) \\ \mu_1(\bc) \le K-2}} \rho^{2K} \left|\mathrm{Wg}(\mu(\bc),d )\right|  \notag\\
& \qquad \qquad \le |\cH|  \left(1 + \frac{\poly(K)}{d}\right) \left(\frac{\rho^2}{d}\right)^{K}  \binom{n-1}{K} K! \bigg[\binom{d}{K}K!\bigg]^2 \sum_{\mu_1=0}^{K-2} \left(\frac{4(4K)^2}{d}\right)^{(K - \mu_1)/2} \notag\\
& \qquad \qquad \le \frac{8(4K)^2}{d}  |\cH| \left(\frac{\rho^2}{d}\right)^{K}  \binom{n-1}{K} K! \bigg[\binom{d}{K}K!\bigg]^2,
\#
where the last line follows from summing the geometric series with ratio ${4(4K)^2}/{d} \le 1/3$. 

Finally, combining the results from \eqref{eq:mean-true-K-lead} and \eqref{eq:bound-mu1-small}, since
\$ 
& \left|\E\left[\similarity_{ii}^\cH\right] - \sum_{T \in \cH} \aut(T) \sum_{T_1,T_2 \cong T}\sum_{\substack{\bc \in \cC(T_1\cup_{n} T_2) \\ \mu_1(\bc) = K}}  \rho^{2K} \wg(\mu(\bc), d) \right| \\
& \qquad \qquad \qquad \qquad \le \sum_{T \in \cH} \aut(T) \sum_{T_1,T_2 \cong T}\sum_{\substack{\bc \in \cC(T_1\cup_{n} T_2) \\ \mu_1(\bc) \le K-2}} \rho^{2K} \left| \wg(\mu(\bc), d) \right|, 
\$
we obtain
\$ 
\E\left[\similarity_{ii}^\cH\right] 
& = \left(1\pm \frac{\poly(K)}{d}\right) |\cH| \left(\frac{\rho^2}{d}\right)^{K} \binom{n-1}{K} K! \left[\binom{d}{K}K!\right]^2. 
\$
Furthermore, note that $\binom{n-1}{K} K! = (1+O(K^2/n)) n^K = (1+ o(1))n^K$ if $K=o(\sqrt{n})$.
Similarly, $\binom{d}{K} K! = \left(1 + o(1)\right) d^K $ if $K=o(\sqrt{d})$. Therefore, we conclude that
\$ 
\E\left[\similarity_{ii}^\cH\right] = \left(1 \pm \frac{\poly(K)}{\min\{n,d\}}\right) |\cH| n^{K} d^{K} \rho^{2K}
\$
Recall that $\wg(\mu(\bc), d) = \left|\wg(\mu(\bc), d)\right|$ for all leading terms where $\bc$ with $\mu_1(\bc) = K$. Thus, after replacing $\wg(\cdot)$ by its absolute value, we have $
\nu^\cH= (1+ o(1)) \E\left[\similarity_{ii}^\cH\right]$.
\end{proof}

We now bound the number of feasible pairs $(T_2,\bc)$ in \eqref{eq:mean-t2}. To this end, we drop the isomorphism constraint $T_2 \cong T_1$, obtaining a result that will also be used later in \Cref{sect:var}. In this setting, we are given a graph $G$ consisting of red edges in $[n]\times [d]$, paired so that each pair shares a common $d$-node. We aim to bound the number of $(H, \bc)$, where $H$ is a union of blue edges with unlabeled $d$-nodes, and $\bc \in \cC(G \cup_n H)$ satisfies $\mu_1(\bc) = \mu_1$. Intuitively, a larger value of $\mu_1$ (\ie, more $4$-circuits) implies that more blue edges form pairs of parallel edges with the red edge pairs. This imposes stronger structural constraints on $H$ and $\bc$, thereby reducing the number of feasible pairs.

\begin{lemma}\label{lem:bound-mu1-le-k-2}
Fix a (possibly disconnected) graph $G$ consisting of $2k$ red edges in $[n]\times [d]$, where the edges are paired so that each pair shares a common $d$-node.
Consider all unions $H$ of $k$ blue edge pairs (with unlabeled $d$-nodes), where each pair consists of two edges in $[n]\times[d]$ sharing a common $d$-node. For each such $H$, form the $n$-union $G\cup_n H$ and consider alternating circuit decompositions $\bc \in \cC(G\cup_n H)$ with no $2$-circuits, satisfying $\kappa(\bc)=0$ and $\mu_1(\bc)=\mu_1$. Then for the fixed $G$,
\$ 
\left|\{(H, \bc) \colon \bc \in \cC(G\cup_{n} H), \mu_1(\bc) =\mu_1 \}\right| \le (4k)^{k-\mu_1}.
\$
\end{lemma}
\begin{proof} 
We obtain the upper bound by counting valid circuit decompositions $\bc$ rather than counting $H$ directly.
Indeed, once $\bc \in \cC(G\cup_n H)$ is fixed, the blue edge union $H$ (with unlabeled $d$-nodes) is uniquely determined by $\bc$. Thus it suffices to bound the number of such decompositions $\bc$. 

Let $\mu_\ell$ denote the number of alternating circuits of length $4\ell$ in $\bc$, for $\ell\ge1$. The total number of red and blue edges satisfies $4\sum_{\ell\ge1}\ell\mu_\ell = 4k$. In particular, $\sum_{\ell\ge2}\ell\mu_\ell = k-\mu_1$. Note that each circuit of length $4\ell$ uses $\ell$ pairs of red edges and $\ell$ pairs of blue edges.
For each $\ell\ge 1$, we choose $\mu_\ell$ (ordered) sequences of $\ell$ pairs out of the $k$ red pairs; this gives a total of $\frac{k!}{\prod_{\ell \ge 1} \mu_\ell!}$ choices (we can globally order all of the pairs, split the global ordering into sequences of the desired lengths: $\mu_1$ sequences of length $1$, $\mu_2$ sequences of length $2$, etc; and finally divide by the number of orderings of the $\mu_\ell$ $\ell$-sequences for each $\ell$).
For $\ell=1$, each sequence contains just one pair of red edges: we simply concatenate a blue path of length $2$ (with an unlabeled blue $d$-node in the middle) to each pair.
For each $\ell \ge 2$, for each sequence of $\ell$ pairs, we choose a first and second edge in each pair (a total of $2^{k-\mu_1}$ choices), and according to this order form an alternating cycle by concatenating each consecutive ordered red pair with a blue $2$-path.

The number of $\{\mu_\ell\}_{\ell \ge 1}$ is no more than the number of integer partitions of $k-\mu_1$, which is at most $2^{k-\mu_1}$.
For each such partition, the number of compatible $\bc$ is upper bounded by
\begin{align*}
\frac{k!}{\prod_{\ell \ge 1}\mu_\ell!} \cdot 2^{k-\mu_1} 
\le \frac{k!}{\mu_1!}\cdot 2^{k-\mu_1} 
&\le k^{k-\mu_1} \cdot 2^{k-\mu_1}. 
\end{align*}
Thus, multiplying by the $2^{k-\mu_1}$ upper bound on the number of partitions, we have our result.
\end{proof}

\subsection{False Pairs}
For $i\neq j$, we have
\$
\E\left[\similarity_{ij}^{\cH}\right] 
& = \sum_{T \in \cH} \aut(T) 
\sum_{\substack{T_1,T_2 \cong T\\\mathrm{root}(T_1) = i\\ \mathrm{root}(T_2) = j}}\E\left[X^{T_1} Y^{T_2}\right].
\$
For a pair $T_1$ and $T_2$, their roots are distinct. Since the root of a tree in $\cH$ has strictly larger degree than any non-root node, the root $i$ of $T_1$ is incident to more edges from $T_1$ than from $T_2$ (where it can appear only as a non-root node) in $T_1\cup_n T_2$. On the other hand, by the remarks following \Cref{lem:isserlis}, $\E[X^{T_1} Y^{T_2}]$ can be nonzero only if every $n$-node in $T_1\cup_n T_2$ is incident to the same number of edges from $T_1$ as from $T_2$. This condition therefore fails at the root $i$, implying that $\E[X^{T_1} Y^{T_2}] = 0$. Since this holds for every pair $T_1$ and $T_2$, we conclude that $\E[\similarity_{ij}] = 0$.

\section{Variance Calculation}\label{sect:var}
In this section, we derive upper bounds on the variance of the scores and prove \Cref{prop:var-control}. Recall from \eqref{eq:var1} that
\#\label{eq:var}
\var(\similarity_{ij}) 
&= \sum_{T,S\in \calT}
\aut(T) \aut(S) \sum_{\substack{T_1, T_2 \cong T, S_1, S_2 \cong S\\ \mathrm{root}(T_1) = \mathrm{root}(S_1) = i\\ \mathrm{root}(T_2) = \mathrm{root}(S_2) = j}}
\cov\left(X^{T_1} Y^{T_2}, X^{S_1} Y^{S_2} \right).
\#
For labeled trees $T_1, T_2, S_1, S_2$, the Weingarten calculus in \Cref{thm:altcir} implies that the covariance term depends on the alternating circuit decomposition of the $n$-union graph $(T_1 \cup S_1)\cup_n(T_2 \cup S_2)$.
For ease of notation, we write 
$U \triangleq (T_1 \cup S_1)\cup_n(T_2 \cup S_2)$
for the \emph{decorated union graph} obtained from these four trees. The graph $U$ is a multigraph in which edges between the same pair of nodes that originate from different trees are retained as parallel edges.
Each edge in $U$ carries a decoration in $\{T_1, S_1, T_2, S_2\}$ indicating its origin. When forming alternating circuits, we color edges originating from $T_1$ or $S_1$ (from $X$) red, and edges originating from $T_2$ or $S_2$ (from $Y$) blue. 
We denote by $\cU_{ij}$ the set of all pairs $(U, \bc)$, where $U$ is a decorated union graph formed by some trees $T_1 \cong T_2 \in \calT$ and $S_1 \cong S_2 \in \calT$ with roots $i$ and $j$, respectively, and $\bc \in \mathcal{C}(U)$ is a circuit decomposition of $U$. 
By construction, there is a natural one-to-one correspondence between 
$(T_1, S_1, T_2, S_2)$ and $U$. 

To rewrite the variance in terms of $U$, we define the weight $w(\cdot)$ as follows. 
For a subgraph 
$H\subset U$ and a tree $L\in\{T_1, S_1, T_2, S_2\}$, let $\cB(L)\subset \cJ$ denote the set of $L$'s branches, and define
\#\label{eq:weight-u-aut} 
w_{L}(H) \triangleq \prod_{J\in \cB(L), J\subset H} \sqrt{\aut(J)}, \quad
w(H) \triangleq w_{T_1}(H) w_{S_1}(H)w_{T_2}(H)w_{S_2}(H).
\#
We now verify that
$w(U) = \aut(T) \aut(S)$. 
Since $T \simeq T_1 \simeq T_2$ and $S \simeq S_1 \simeq S_2$, we have
\#\label{eq:aut-ts}
\aut(T) \aut(S)= \sqrt{\aut(T_1) \aut(S_1)\aut(T_2) \aut(S_2)}. 
\# 
Moreover, since $T_1$ is a tree with non-isomorphic branches, we have
\[
\aut(T_1) = \prod_{J \in \cB(T_1)} \aut(J),
\]
and the same identity holds for $S_1$, $T_2$, and $S_2$. Therefore, the right-hand side of \eqref{eq:aut-ts} equals the product of $\sqrt{\aut(J)}$ over all branches $J$ appearing in the four trees, which is precisely $w(U)$.  

The next lemma applies the approximate factorization property of the Weingarten function from \eqref{eq:twg-mult} and rewrites the variance bound in terms of decorated union graphs. For $U = (T_1\cup S_1) \cup_{n}(T_2\cup S_2)$, let 
$\cC^\dag(U) = \cC(U)\setminus \calC(T_1 \cup T_2)\times \calC(S_1 \cup S_2)$ denote the set of alternating circuit decompositions, excluding those in which every circuit lies within either $T_1 \cup_{n} T_2$ or $S_1 \cup_{n} S_2$.
This exclusion corresponds to subtracting the mean square from the second moment.
\begin{lemma}\label{lmm:var-step1}
For any $i\neq j \in [n]$, with $\nu = \nu^\cT$ from \eqref{eq:exp-score-abs}, we have
\# 
& \var(\similarity_{ii}) \le \sum_{(U,\bc)\in \cU_{ii} \colon \bc\in \cC^\dag(U)} w(U) \rho^{2(2K-\kappa(\bc))}  \wg(\mu(\bc),d) + \frac{\poly(K)}{d} \nu^2, \label{eq:var-ii-pf1}  \\
& \var(\similarity_{ij})\le \sum_{(U,\bc)\in \cU_{ij}} w(U) \rho^{2(2K-\kappa(\bc))}  \wg(\mu(\bc),d). \label{eq:var-ij-pf1} 
\#
\end{lemma}
\begin{proof}
We consider true pairs ($i=j$) and false pairs ($i \neq j$) separately.

\paragraph{True pairs.} 
Following from \eqref{eq:var},   for fixed labeled trees $T_1, T_2, S_1, S_2$, we compute the covariance
\#\label{eq:cov-t}
\cov\left(X^{T_1}Y^{T_2}, X^{S_1} Y^{S_2} \right) = 
\E\left[X^{T_1}Y^{T_2}X^{S_1} Y^{S_2} \right] - \E\left[X^{T_1} Y^{T_2} \right] \E\left[ X^{S_1} Y^{S_2}\right].
\#
Let $U = (T_1\cup S_1) \cup_{n}(T_2\cup S_2)$.
Applying the Weingarten calculus from \Cref{thm:altcir}, and noting that each of the four trees contains $2K$ edges, the second moment is given by
\#\label{eq:cov-1}
\E\left[X^{T_1} Y^{T_2}X^{S_1} Y^{S_2} \right] = \sum_{\bc\in \cC(U)} \rho^{2(2K-\kappa(\bc))}  \wg(\mu(\bc),d).
\#
In addition, since $T_1\cup_n T_2$ and $S_1\cup_n S_2$ are simple graphs that do not contain $2$-circuits, we have
\#\label{eq:cov-2}
\E\left[X^{T_1}Y^{T_2}\right]\E
\left[ X^{S_1} Y^{S_2} \right] & = \sum_{\substack{\bc_1\in \cC(T_1\cup_n T_2) \\ \bc_2\in \cC(S_1\cup_n S_2)}} \rho^{4K}\wg(\mu(\bc_1),d)\wg(\mu(\bc_2),d). 
\#
Combining \eqref{eq:cov-t}, \eqref{eq:cov-1}, and \eqref{eq:cov-2}, and using the definition of $\mathcal{C}^{\dagger}(U)$, we obtain
\#\label{eq:cov} 
& \left|\cov\left(X^{T_1} Y^{T_2}, X^{S_1} Y^{S_2} \right) - \sum_{\bc\in \cC^\dag(U)} \rho^{2(2K-\kappa(\bc))}  \wg(\mu(\bc),d) \right| \notag\\
& \qquad = 
\left|\sum_{\substack{\bc_1\in \cC(T_1\cup_n T_2) \\ \bc_2\in \cC(S_1\cup_n S_2)}} 
\rho^{4K} \left[ \wg(\mu(\bc_1\cup\bc_2),d) - \wg(\mu(\bc_1),d)\wg(\mu(\bc_2),d)  \right] 
\right| \notag\\
& \qquad \le \frac{\poly(K)}{d} \sum_{\substack{\bc_1\in \cC(T_1\cup_n T_2) \\ \bc_2\in \cC(S_1\cup_n S_2)}} \rho^{4K}\left|\wg(\mu(\bc_1),d)\right|\left|\wg(\mu(\bc_2),d)\right|,
\#
where the last inequality follows from the approximate factorization of the Weingarten function in \eqref{eq:twg-mult}.
Summing \eqref{eq:cov} over $T,S\in\cT$, and then over $T_1,T_2\cong T$, $S_1,S_2\cong S$ (equivalently, over all possible decorated union graphs $U$), and using $w(U) = \aut(T)\aut(S)$,  
we conclude that
\$ 
\var(\similarity_{ii}) \le \sum_{(U,c)\in \cU_{ii} \colon \bc\in \cC^\dag(U)} w(U) \rho^{2(2K-\kappa(\bc))}  \wg(\mu(\bc),d) +  \frac{\poly(K)}{d}\nu^2.
\$
\paragraph{False pairs.} For $i\neq j$, it suffices to bound the variance by the second moment. Summing \eqref{eq:cov-1} over all decorated union graphs $U$ yields
\#\label{eq:var-ij-ratio-1}
\var(\similarity_{ij})\le \sum_{(U,\bc)\in \cU_{ij}} w(U) \rho^{2(2K-\kappa(\bc))}  \wg(\mu(\bc),d).
\#
\end{proof}

Introducing such decorated union graphs will also simplify the subsequent counting by allowing us to first count the union graph and then retrieve the individual trees via edge decorations.

\subsection{Decomposition of Decorated Union Graphs}
The approximate factorization of the Weingarten function in \eqref{eq:twg-mult} allows us to apply a divide-and-conquer strategy: we decompose each $(U,\bc)$ into edge-disjoint components and study them separately. Specifically, we consider $U = \uo \cup \ur$ and $\bc = \bco \cup \bcr$, where first component, $(\uo,\bco)$, is the \emph{fully overlapping part}, and the second, $(\ur,\bcr)$, is the \emph{grafted part}.

The fully overlapping part consists of those branches that fully overlap between exactly two of the four trees and whose circuit decomposition contains only $2$- or $4$-circuits.
As illustrated in \Cref{fig:fully-overlap}, two rooted branches are said to be fully overlapping in $(U, \bc)$ if they \emph{share the same root node} and one of the following holds:
\begin{enumerate}
\item[(i)] \textbf{Monochromatic case.} There are two parallel edges of the same color between each pair of adjacent $n$-node and $d$-node. These edges form only $2$-circuits. The two rooted branches are decorated by either $(T_1, S_1)$ or $(T_2, S_2)$.
\item[(ii)] \textbf{Bichromatic case.}
There are two parallel pairs of edges between each pair of adjacent $n$-nodes. Within each edge pair, the two edges have the same color and share a common $d$-node, whereas the two edge pairs have different colors and distinct $d$-nodes. In this case, these edges form only $4$-circuits. We further distinguish the \emph{cross-isomorphism} pairs $(T_1,S_2)$ or $(T_2,S_1)$, and the \emph{within-isomorphism} pairs $(T_1,T_2)$ or $(S_1,S_2)$.
\end{enumerate} 
\input{fig/fully-overlap}
We define $\uo$ as the union of all branches in $U$ that fully overlap between exactly two of the four trees, and $\bco$ as the set of circuits in $\uo$ (all of length $2$ or $4$). The grafted part is then defined by $\ur = U\setminus \uo$ and $\bcr = \bc\setminus \bco$, consisting of the remaining branches of the four trees that do not fully overlap with one another. 

We provide some intuition for this decomposition. The fully overlapping part $\uo$ consists of branches taken directly from $\cJ$, 
which allows for a more precise and explicit counting.
The grafted part $\ur$, by contrast, requires a more delicate analysis: unlike in $\uo$, the branches here intertwine without fully overlapping, leading to more intricate union structures. 
However, an important observation that facilitates counting is that each branch in the grafted part either shares nodes with other branches or participates in longer circuits, which helps constrain the number of admissible grafted structures. 

We now introduce parameters describing each part, as illustrated in \Cref{fig:decomposition}. 
Let $\cUo(r, m, s)$ denote the set of fully overlapping parts $(\uo, \bco)$ containing a total of $r$ fully overlapping branch pairs, among which $m$ are monochromatic pairs and $s$ are bichromatic cross-isomorphism pairs.
Let $\cU_\rr(\ell)$ denote the set of grafted parts $(\ur, \bcr)$ 
containing $\ell$ branches in total from the four trees $T_1,T_2,S_1,S_2$. 
Given $(U,\bc)\in\cU$, its two parts $(\uo,\bco)$ and $(\ur,\bcr)$ belong to some pair of
$\cUo(r, m,s)$ and $\cU_\rr(\ell)$, respectively, where $r,\ell$ satisfy
\# 
& 2r + \ell = 4D, \label{eq:r-l-D} 
\#
since the four trees contain $4D$ branches in total, while $\uo$ and $\ur$ contain $2r$ and $\ell$ branches, respectively. 

\input{fig/decomposition-u}

With this notation, we decompose the variance into contributions from the fully overlapping part and the grafted part. Specifically, with the weight $w(\cdot)$ defined in \eqref{eq:weight-u-aut}, we define 
\#
& \ov(r,m,s) = \frac{1}{|\cT|^2(nd\rho^2)^{rM}}\sum_{(\uo, \bco)\in \cUo(r,m,s)} w(\uo) \rho^{2(rM-\kappa(\bco))}  \left|\wg(\mu(\bco),d) \right|, \label{eq:ov-def} \\
& \re(\ell) = \frac{1}{(nd\rho^2)^{\ell M/2}} \sum_{(\ur, \bcr)\in \cU_\rr(\ell)} w(\ur) \rho^{2(\ell M/2-\kappa(\bcr))}  \left|\wg(\mu(\bcr),d) \right|, \label{eq:re-def}
\#
with the conventions $\ov(0,0,0) = \re(0)=1$.
The next lemma formalizes this decomposition.
\begin{lemma}\label{lmm:var-ratio} 
For $\ov(r,m,s)$ and $\re(\ell)$ defined in \eqref{eq:ov-def}
and \eqref{eq:re-def}, we have
\$
& \frac{\var(\similarity_{ii})}{|\cT|^2(nd\rho^2)^{2K}} \le \left(1 + o(1)\right) \sum_{\substack{r\ge m,s\ge 0, \, \ell\ge 0  \\ 2r+\ell = 4D \\ m+s+ \ell \ge 1}} \ov(r,m,s) \re(\ell) + \frac{\poly(K)}{d},  \\
& \frac{\var(\similarity_{ij})}{|\cT|^2(nd\rho^2)^{2K}} \le \left(1 + o(1)\right) \sum_{\substack{r= m\ge 0, \, \ell\ge 0  \\ 2m+\ell = 4D}} \ov(r,m,0) \re(\ell), \quad \forall j\neq i.
\$ 
\end{lemma}
\begin{proof}
We consider true pairs ($i=j$) and false pairs ($i \neq j$) separately.

\paragraph{True pairs.}
For $i=j$, we recall the bound on $\var(\similarity_{ii})$ from \eqref{eq:var-ii-pf1}.
Every $(U,\bc)\in\cU_{ii}$ admits a decomposition into $(\uo,\bco)\in \cUo(r, m, s)$ and $(\ur,\bcr) \in \cU_\rr(\ell)$. The condition $\bc\in \cC^\dag(U)$ requires that at least one circuit contain edges with decorations from both $\{T_1, T_2\}$ and $\{S_1, S_2\}$, rather than lying within either $T_1\cup_n T_2$ or $S_1\cup_n S_2$. This is only possible if the parameters satisfy
\#\label{eq:s-ell-1}
m + s + \ell \ge 1.
\#
Otherwise, when $m=s=\ell=0$, every branch of $U$ belongs to a bichromatic within-isomorphism fully overlapping pair, and therefore every circuit of $U$ lies entirely within either $T_1\cup_n T_2$ or $S_1\cup_n S_2$.
 
Since $\uo$ and $\ur$ are edge-disjoint and each branch of the four trees belongs to exactly one of them, the weight factorizes as
\[
w(U) = w(\uo) w(\ur). 
\]
Using $\kappa(\bc) = \kappa(\bco) + \kappa(\bcr)$ and $2K = rM + \ell M/2$ from \eqref{eq:r-l-D}, we split the factor 
\$
(\rho^2)^{2K-\kappa(\bc)} = (\rho^2)^{rM-\kappa(\bco)}\cdot (\rho^2)^{\ell M/2-\kappa(\bcr)}.
\$ 
Furthermore, the approximate factorization property in \eqref{eq:twg-mult} yields \$\left|\wg(\mu(\bc),d)\right| \le \left(1 + o(1)\right)\left|\wg(\mu(\bco),d)\right|\left|\wg(\mu(\bcr),d)\right|.\$ 
Combining them and summing over $(r,s,\ell)$ satisfying \eqref{eq:r-l-D} and \eqref{eq:s-ell-1},
we may bound \eqref{eq:var-ii-pf1} as
\$
\var(\similarity_{ii})
& \le \left(1 + o(1)\right)\sum_{\substack{r\ge m,s\ge 0, \, \ell\ge 0  \\ 2r+\ell = 4D \\ m+s+\ell \ge 1}} \left(\sum_{(\uo, \bco)\in \cUo(r,m,s)} w(\uo) \rho^{2(rM-\kappa(\bco))}  \left|\wg(\mu(\bco),d) \right| \right) \notag\\
& \qquad \qquad \qquad \qquad
\times \left( \sum_{(\ur, \bcr)\in \cU_\rr(\ell)} w(\ur) \rho^{2(\ell M/2-\kappa(\bcr))}  \left|\wg(\mu(\bcr),d) \right| \right) + \frac{\poly(K)}{d} \nu^2.
\$
Finally, dividing both sides by $|\cT|^2 (nd\rho^2)^{2K} = |\cT|^2 (nd\rho^2)^{rM} \cdot (nd\rho^2)^{\ell M/2}$, recalling  $\ov(r,m,s)$ and $\re(\ell)$ defined in \eqref{eq:ov-def} and \eqref{eq:re-def}, and using $\nu = (1+o(1))|\cT| (nd\rho^2)^{K}$ from \Cref{prop:mean-true}, we obtain the claimed bound.

\paragraph{False pairs.} For $i \neq j$, bichromatic overlapping branches cannot occur by definition, since branches of different colors are rooted at distinct vertices. 
Thus, we must have $r = m$ and $s=0$; by \eqref{eq:r-l-D},
\#\label{eq:para-consts-r-o-ij}
2m + \ell = 4D.
\#
The remainder of the argument is the same as the true-pair case. We decompose each $(U,\bc)\in\cU_{ij}$ into $(\uo,\bco)\in \cUo(r, m, s)$ and $(\ur,\bcr) \in \cU_\rr(\ell)$, apply the factorizations, and finally normalize by $|\cT|^2 (nd\rho^2)^{2K}$. By \eqref{eq:var-ij-pf1}, with the additional constraint in \eqref{eq:para-consts-r-o-ij}, we obtain 
\#\label{eq:var-ij-ratio-2}
\frac{\var(\similarity_{ij})}{|\cT|^2(nd\rho^2)^{2K}} \le \left(1 + o(1)\right) \sum_{\substack{r= m\ge 0, \, \ell\ge 0  \\ 2m+\ell = 4D}} \ov(r,m,0) \re(\ell).
\#
\end{proof}

By \Cref{lmm:var-ratio}, it suffices to bound $\ov(r,m,s)$ and $\re(\ell)$ separately. The results are given in \Cref{prop:uo-co,prop:ur-cr}. 
\begin{proposition}[Fully overlapping part]\label{prop:uo-co}
Suppose that $\beta^{M} |\cJ|\ge 7D^2$.
For any $0\le m\le r$ and $0\le s\le r$, $\ov(r,m,s)$ defined in \eqref{eq:ov-def} satisfies:
\$ 
\ov(r,m,s)
& \le  C \left(\frac{\sqrt{\beta}}{\rho^2}\right)^{mM} \beta^{sM/2}.
\$    
\end{proposition}
By \Cref{prop:uo-co}, if $\rho^2> \sqrt{\beta}$, the bound on
$\ov(r,m,s)$ decays geometrically in both the number of monochromatic pairs $m$ and the number of bichromatic cross-isomorphism pairs $s$.

\begin{proposition}[Grafted part]\label{prop:ur-cr}
For any $\ell\ge0$, $\re(\ell)$ defined in \eqref{eq:re-def} satisfies
\$ 
\re(\ell) \le C\left(\frac{\poly(K)}{\min\{d, n\}}\right)^{\ell/8}.
\$    
\end{proposition}
By \Cref{prop:ur-cr}, the upper bound on $\re(\ell)$ decays geometrically as the number of grafted branches $\ell$ increases. In what follows, we first combine \Cref{prop:uo-co,prop:ur-cr} to complete the proof of \Cref{prop:var-control}. We then analyze the more intricate grafted part in \Cref{sect:ur}, and the fully overlapping part in \Cref{sect:uo}.

\subsection{Proof of \Cref{prop:var-control}} 
We now combine \Cref{prop:uo-co,prop:ur-cr} to prove \Cref{prop:var-control}.

\paragraph{True pairs.} 
We first prove \eqref{eq:var-bd-ii} for true pairs $i=j$.
\begin{proof}[Proof of \eqref{eq:var-bd-ii}]
Substituting \Cref{prop:uo-co,prop:ur-cr} into \Cref{lmm:var-ratio}, when $\rho^2 > \sqrt{\beta}$, we obtain  
\#\label{var-ii-final}
\frac{\var(\similarity_{ii})}{|\cT|^2(nd\rho^2)^{2K}} & \le \left(1 + o(1)\right) \sum_{\substack{r\ge m,s\ge 0, \, \ell\ge 0  \\ 2r+\ell = 4D \\ m+s+ \ell \ge 1}} \ov(r,m,s) \re(\ell) + \frac{\poly(K)}{d}\notag\\
& \le C\sum_{\substack{m,s, \ell\ge 0 \\ m+s + \ell \ge 1}} \left(\frac{\sqrt{\beta}}{\rho^2} \right)^{mM} \beta^{sM/2} \left(\frac{\poly(K)}{\min\{d, n\}}\right)^{\ell/8} + \frac{\poly(K)}{d} \notag \\
& \le C \left[\left(\frac{\sqrt{\beta}}{\rho^2} \right)^{M} + \left(\frac{\poly(K)}{\min\{d, n\}}\right)^{1/8}\right],
\#
where the last inequality follows from $\sum_{m,s,\ell\ge 0, m+s+\ell \ge 1} x^m y^s z^\ell = {1}/{[(1-x)(1-y)(1-z)]} -1 
\le 14(x+y+z)/3$ when $x = (\sqrt{\beta}/\rho^2)^M$, $y=\beta^{M/2}$, and $z = (\poly(K)/\min\{d, n\})^{1/8}$ satisfy $x,y,z\in [0,1/2]$; and $y \le x$. 
\end{proof}

\paragraph{False pairs.} 
We now prove \eqref{eq:var-bd-ij} for false pairs $i\neq j$.
\begin{proof}[Proof of \eqref{eq:var-bd-ij}]
Substituting \Cref{prop:uo-co,prop:ur-cr} into \Cref{lmm:var-ratio}, when $\rho^2 > \sqrt{\beta}$, we have
\#\label{eq:2nd-mmnt-bd-ij} 
\frac{\var(\similarity_{ij})}{|\cT|^2(nd\rho^2)^{2K}} & \le \left(1 + o(1)\right) \sum_{\substack{r= m\ge 0, \, \ell\ge 0  \\ 2m+\ell = 4D}} \ov(r,m,0) \re(\ell) \notag \\
& \le C \sum_{\substack{m, \ell\ge 0 \\ 2m + \ell = 4D}} \left(\frac{\sqrt{\beta}}{\rho^2} \right)^{mM} \left(\frac{\poly(K)}{\min\{d, n\}}\right)^{\ell/8} 
\notag\\
& \le C(2D+1)\max\left\{\left(\frac{\beta}{\rho^4} \right)^{K}, \; \left(\frac{\poly(K)}{\min\{d, n\}}\right)^{D/2}\right\} ,
\# 
where the last inequality follows by applying $\sum_{m,\ell \ge0, 2m+\ell =4 D} x^{2m} y^{\ell} \le (2D+1) \max\{x, y\}^{4D}$ for $x,y>0$ with $x=(\sqrt{\beta}/\rho^2)^{M/2}$ and 
$y=(\poly(K)/\min\{d,n\})^{1/8}$, where $K = MD$. 
\end{proof}

\subsection{Bounding the Grafted Part} \label{sect:ur}
In this section, we prove \Cref{prop:ur-cr}. 
A key step is to bound $|\cUr(\ell)|$ by counting the admissible grafted parts 
$\ur$ and, for each such part, the corresponding alternating circuit decompositions. 
Throughout the counting argument, we will keep track of the following two types of nodes.

\begin{definition}[Overdecorated node]
We say an $n$-node, $u$, is \emph{overdecorated} if $u$'s incident edges carry at least 3 different decorations from the set $\{S_1,S_2,T_1,T_2\}$, and there is at least one tree in $\{S_1,S_2,T_1,T_2\}$ for which $u$ is not the root.
We use $\ovd$ to denote the number of overdecorated nodes.
\end{definition}

\begin{definition}[Crossroad node]
We say a $d$-node, $a$, is a \emph{crossroad node} with respect to a circuit decomposition $\bc$ if $a$ has degree 4 and participates in at least one circuit of size $\ge 4$ in $\bc$.
We use $\xr$ to denote the number of crossroad nodes.
\end{definition}
Collectively, we refer to overdecorated and crossroad nodes as \emph{graft witness} nodes.

Consider, for example, the grafted part in \Cref{fig:decomposition}. The node $a$ is a crossroad node. Moreover, the edge pairs corresponding to $(u,w)$, $(w,v)$, $(v,t)$, and $(t,u)$ form an alternating circuit of length $8$. Thus, each of the four grafted branches contains an edge lying in an alternating circuit of length greater than $4$. More generally, the size and structural complexity of the grafted part can be controlled by the numbers of overdecorated nodes, crossroad nodes, and edges lying in alternating circuits of length greater than $4$. The following lemma gives one concrete formulation of this observation; we will use more later.

\begin{lemma}\label{lmm:para-constr-ur}
For any $(\ur, \bc)$ with $\ell$ grafted branches, $\kappa$ 2-circuits, $\mu_1$ $4$-circuits, $\ovd$ overdecorated nodes and $\xr$ crossroad nodes,
\$
\ell \le 4\left(\ell M/2- \kappa /2- \mu_1 + \ovd + \xr \right).
\$
\end{lemma}
\begin{proof}
Fix a pair $(\ur,\bc)$. 
The total number of edges (counted with multiplicity) in $\ur$ is $2\ell M$, of which $2\ell M - 2\kappa - 4\mu_1$ belong to circuits of length $> 4$.
Hence, there are at most $2\ell M - 2\kappa -4\mu_1$ branches that include an edge participating in a circuit of length $> 4$.

We now show that every other grafted branch must correspond to a graft witness node.
Consider a branch $J$ from, say, $T_1$, in which all edges participate in circuits of length $\le 4$. 
Let $\widetilde V(J)$ be the node set including all $n$-nodes and $d$-nodes in $J$ and all $d$-nodes that lie in $4$-circuits involving edges in $J$. 
We claim that if $\widetilde V(J) \setminus \{\text{root}(J)\}$ has no graft witness nodes, then $J$ is actually not a grafted branch.
Indeed, consider any pair of edges $(u,a),(v,a)$ from $J$ (with $u,v \in [n]$ and $a \in [d]$). 

\begin{enumerate}
    \item
If $a$ has degree $2$ in $\ur$, the fact that all circuits are of length at most $4$ implies that $(u,a),(v,a)$ are in a $4$-circuit with blue edges $(u,a'),(v,a')$. If $a'$ has degree $4$, since its incident edges $(u,a'),(v,a')$ appear in a $4$-circuit, it must be a crossroad node by definition. Therefore, $a'$ must have degree $2$ and the edges $(u,a'),(v,a')$ are decorated by some $R_{uv} \in \{S_2, T_2\}$.
\item Otherwise, if $a$ has degree $4$ in $\ur$, the lack of crossroad nodes implies that $(u,a),(v,a)$ is covered by two $2$-circuits with edges $(u,a),(v,a)$ of decoration $R_{uv} = S_1$.
\end{enumerate}

In both cases, we may map each pair of edges $(u,a),(v,a)$ in $J$ to an edge $(u,v)$ in a tree $J^{\mathrm{map}}$ on $n$-nodes which is decorated with two decorations, $T_1$ and $R_{uv}$. 
Now, if there exist some $(u,v)$ and $(u',v')$ such that $R_{uv} \neq R_{u'v'}$, take a path in $J^{\mathrm{map}}$ from $u$ to $u'$ (such a path must exist because $J$ was connected).
Since the first endpoint is decorated with $T_1,R_{uv}$ and the second endpoint is decorated with $T_1,R_{u'v'}$, at least one non-root $n$-node in this path has to be incident to edges of three types of labels within $J$.
But then this $n$-node would be overdecorated, a contradiction.
Hence $R_{uv} = R$ for all edges $(u,v)$, so $J$ completely overlaps with a branch $B$ decorated by $R$. Moreover, branches $J$ and $B$ must share the same root; otherwise, the root of $B$ would appear as a non-root node in $J$ whose incident edges carry at least $3$ distinct decorations, and hence would be overdecorated. Thus, if $\widetilde V(J) \setminus \{\text{root}(J)\}$ has no graft witness nodes, then $J$ and $B$ form a fully-overlapping branch pair.

Therefore, we charge every 
grafted branch $J$ containing no edge in a circuit of length $> 4$ to a graft witness node in $\widetilde V(J) \setminus \{\text{root}(J)\}$. Each overdecorated node can be charged at most four times, since it belongs to at most one such branch in each of the four trees. 
Moreover, each crossroad node belongs to ${V}(J)$ for at most two branches $J$, and hence to $\widetilde{V}(J)$ for at most four branches $J$. Thus, each graft witness node is charged at most four times, giving at most $4(o+x)$ such branches. Combining this with the previous bound completes the proof.
\end{proof}

The following lemma is another example by which the graft witness nodes allow us to control the structure of the grafted part. Here $V_n(T_1)$ (resp. $V_n(S_1)$) denotes the set of $n$-nodes incident to $T_1$-decorated edges (resp. $S_1$-decorated edges) in $\ur$.
\begin{lemma}
For any $(\ur, \bc)$ with 
$\ovd$ overdecorated nodes, $\xr$ crossroad nodes, and $\kappa_1$ red $2$-circuits,
\begin{equation}
 |V_n(T_1)\cap V_n(S_1)\setminus\{i\}| \le \ovd + \xr +\kappa_1/2. \label{eq:bound-v1}
\end{equation}
\end{lemma}
\begin{proof}
We prove the result by constructing a one-to-one map from the set of nodes $u \in V_n(T_1)\cap V_n(S_1)$ with $u\neq i$ to the union of the set of overdecorated nodes, the set of crossroad nodes, and the set of pairs of red $2$-circuits that share a $d$-vertex in $S_1$.

Fix such a node $u$. If $u$ is overdecorated, we map $u$ to itself. 
Suppose next that $u$ is not overdecorated. Then by virtue of being in $S_1 \cap T_1$ it is incident only to red edges, and must participate exclusively in red $2$-circuits.
By assumption $u$ is not the root node, so let $a \in [d]$ be the parent of $u$ in $S_1$, and let $v \in [n]$ be the parent of $a$ in $S_1$. If $a$ is a crossroad node, we map $u$ to $a$. No other $n$-node has parent $a$ in $S_1$.
If $a$ is not a crossroad node, then both $(u,a)$ and $(v,a)$ must be in red $2$-circuits. We map $u$ to the pair of red $2$-circuits formed by these two edges. No other $n$-node has the same pair of parent edges.
\end{proof}

On the other hand, a straightforward counting argument shows that when there are many overdecorated and crossroad nodes, there are fewer ways to label the vertices.
This will allow us to account for the combinatorial complexity of the many possible grafted configurations. 

\begin{lemma}\label{lmm:para-constr}
For any $(\ur, \bc)$ with $\ell$ grafted branches, $\kappa$ 2-circuits, 
$\ovd$ overdecorated nodes, and $\xr$ crossroad nodes,
the following bounds hold:
\#
& |V_n(\ur)\setminus\{i,j\}| \le \ell M/2 - \ovd/2, \label{eq:num-v} \\
& |V_d(\ur)| \le \ell M - (\kappa + \xr)/2. \label{eq:num-a}
\#
\end{lemma}
\begin{proof} 
Both statements are a consequence of simple accounting.
To prove \eqref{eq:num-v}, let $m(u)$ be the number of trees in $\{S_1,T_1,S_2,T_2\}$ that contain node $u$ as a non-root.
Since each branch contains $M$ $n$-nodes, and because a valid circuit decomposition requires $m(u) \ge 2$ for each node $u$ appearing as a non-root,
\[
2|V_n(\ur)\setminus\{i,j\}| + |\{u \not\in \{i,j\} : m(u) \ge 3 \}| + m(i) + m(j) \le \sum_u m(u) = \ell M.
\]
By definition of overdecorated nodes, $\ovd \le |\{u : m(u) \ge 3\}| + m(i) + m(j)$.
Hence
\[
2|V_n(\ur)\setminus\{i,j\}| + \ovd \le \ell M 
\]
And rearranging gives \eqref{eq:num-v}. 

To prove \eqref{eq:num-a}, let $\cD_4$ be the set of $d$-nodes in $\ur$ of degree $4$. 
Since every $d$ node has degree $2$ before taking the union and degree $2$ or $4$ in $\ur$,
\[
2\ell M=\sum_{a \in V_d(\ur)}\deg(a)=2|V_d(\ur)|+2|\cD_4| \implies 
|\cD_4|=\ell M-|V_d(\ur)|.
\]
Now, each crossroad node has degree $4$ and is part of $\cD_4$.
There are also at least $(\kappa - \xr)/2$ degree-$4$ non-crossroad nodes, because there are at least $\kappa - \xr$ 2-circuits which are not incident on any crossroad node (each crossroad node can be incident on at most one $2$-circuit) but must be incident to a $d$-node of degree $4$, and any non-crossroad node of degree $4$ is incident on two $2$-circuits. 
Therefore
\[
|\calD_4| \ge \xr + (\kappa-\xr)/2 \ge (\kappa + \xr)/2.
\]
Combining this with the previous display gives \eqref{eq:num-a}.
\end{proof}

\subsubsection{Counting the Grafted Part}

We now give an upper bound on the number of pairs $(\ur,\bc)$ with a given number of $2$- and $4$-circuits and a given number of graft witness nodes.
We will first construct the union of $S_1,T_1$ and the red $2$-circuits, then appeal to \Cref{lem:bound-mu1-le-k-2} as we did for the mean calculation to add in the structure of $S_2,T_2$ and the alternating circuits, and finally add in the blue $2$-circuits (whose admissible locations will be constrained by the graft witness nodes, as we shall soon see).

\begin{lemma}\label{lmm:wr-b}
Let $\cUr(\ell,\kappa,\mu_1,\ovd,\xr)$ denote the set of all pairs $(\ur,\bc)$ with $\ell$ grafted branches, $\kappa$ 2-circuits, $\mu_1$ 4-circuits, $\ovd$ overdecorated nodes and $\xr$ crossroad nodes.
Then there exists a universal constant $C$ such that for all $n$ sufficiently large, for any $\ell, \kappa, \mu_1, \ovd, \xr\in \N$, 
\$ 
|\cUr(\ell, \kappa, \mu_1, \ovd, \xr)| 
\le C K^4 \left(Cn^{1/2} d\right)^{\ell M}\left(\frac{C K^{12}}{n} \right)^{\ovd/2}\left(\frac{CK^{12}}{d} \right)^{\xr/2} \left(\frac{C}{d}\right)^{\kappa/2} \left(CK\right)^{\ell M- \kappa - 2\mu_1}.
\$
\end{lemma}
\begin{proof}
The proof will be via an encoding argument. 
As in our bounds for the expectation in \Cref{prop:mean-true}, we will start by encoding the structure of the red part, and then simultaneously encode the alternating circuit decomposition and blue part.

\paragraph{The red part.} We first encode the configuration of the red part, $T_1 \cup S_1$.
Let $v_1+1$ denote the number of $n$-nodes (including the root) in $S_1\cup T_1$.
Since $T_1$ and $S_1$ are trees sharing a common root $i$, the graph $T_1 \cup S_1$ is connected and must admit a rooted spanning tree.
We choose an arbitrary minimum spanning tree of its $v_1+1$ $n$-nodes, for which there are at most $C\alpha^{-(v_1+1)}$ choices;
each edge in this spanning tree is subdivided into two edges by adding a $d$-node in the middle.

Next, we record the location of the remaining red edges.
We first place the $m$ remaining red edges, treating each $2$-circuit as a single edge, and then specify which red edges to double to form the red $2$-circuits.
There are fewer than $10K$ nodes in the union of all $4$ trees, and the number of possible simple edges between them is at most $(10K)^2$. Hence there are at most $(10K)^{2m}$ ways to add the $m$ red edges.
It remains to specify which edges are doubled to create the red $2$-circuits. Since there are at most $2\ell M$ red edges in total, this can be done in at most $2^{2\ell M}$ ways.

 Finally, given any graph constructed above, we specify the pairing of all red edges in circuits of length $> 2$ so that each pair shares a common $d$-node. 
At each $d$-node of degree $2$, the pairing is uniquely determined. 
The only nontrivial choices arise at $d$-nodes of degree $4$ that are not exclusively incident to red $2$-circuits. 
By definition, these are the crossroad nodes, and there are at most $\xr$ of them. 
At each crossroad node, the four incident edges are partitioned into two unordered pairs. The number of such pairings is $\binom{4}{2}/2 = 3$. Hence, the total number of possible pairings is at most
$ 
3^{\xr}.
$
Hence, the total number of possible $S_1\cup T_1$ structures is no more than
\[
C\alpha^{-(v_1+1)} \cdot (10K)^{2m} \cdot 2^{2\ell M} \cdot 3^{\xr}.
\]

We now claim that $m \le 2(\ovd + \xr)$. To see this, let $e_1$ and $\kappa_1$ denote, respectively, the number of red edges in $S_1\cup T_1$ counted with multiplicity, and the number of red $2$-circuits. The spanning tree constructed in the first step contains exactly $2 v_1$ red edges, so in the second step we must first place $m = e_1-2v_1-\kappa_1$ remaining red edges, treating each $2$-circuit as a single edge.
Note that the number of $n$-nodes common to $S_1$ and $T_1$ is exactly $|{V}_n(S_1)\cap{V}_n (T_1)\setminus\{i\}| = e_1/2 - v_1$. Indeed, the $e_1$ red edges form the two trees $S_1$ and $T_1$ rooted at $i$; since half of the non-root nodes are $n$-nodes and half are $d$-nodes, we have $|{V}_n(S_1)\setminus\{i\}|+|{V}_n (T_1)\setminus\{i\}| = e_1/2$. Moreover, the union of $S_1$ and $T_1$ satisfies $|{V}_n(S_1)\cup{V}_n (T_1)\setminus\{i\}| = v_1$ by definition. Thus, by the inclusion--exclusion principle, we immediately have $|{V}_n(S_1)\cap{V}_n (T_1)\setminus\{i\}| = e_1/2 - v_1$.
By \eqref{eq:bound-v1}, this quantity satisfies $e_1/2 - v_1 \le \ovd + \xr + \kappa_1/2$. 
Consequently, $m = e_1 - 2 v_1 - \kappa_1 \le 2(\ovd + \xr)$. 

Moreover, by \eqref{eq:num-v}, $v_1\le \ell M/2$.
Hence, the number of possible encodings of union structures $S_1 \cup T_1$ is bounded by 
\[
\sum_{\substack{v_1 \le \ell M / 2 \\ m \le 2(\ovd + \xr)}}
\alpha^{-(v_1+1)} \cdot (10K)^{2m} \cdot 2^{2\ell M} \cdot 3^{\xr}
\le C \alpha^{- \ell M / 2} (10K)^{4(\ovd + \xr)} 2^{2 \ell M} 3^{\xr},
\]
for $C$ a sufficiently large constant.

\paragraph{Blue edges and alternating circuit decompositions.}
As was done in the bound on the expectation, we now simultaneously encode the alternating circuit decompositions and the structure of blue edges participating in circuits of length $\ge 4$.

First we order the paired red edges to form alternating circuits of length $\ge 4$. 
In all of the circuits of length $\ge 4$ there are $2\ell M - 2\kappa \le 8K$ red and blue edges in total, of which $4\mu_1$ lie in $4$-circuits. 
By \Cref{lem:bound-mu1-le-k-2}, the number of ways to form these alternating circuits is at most
\$ 
(4K)^{2(\ell M/2- \kappa /2- \mu_1)}.
\$
These ordered pairs also determine the blue edges in these alternating circuits.

Now, we specify the location of the blue $2$-circuits. 
Since each $2$-circuit corresponds to a pair of identical edges, one from $S_2$ and one from $T_2$, we may consider the graph induced by the blue $2$-circuits in $T_2$, and note that it is a forest.
We first encode the structure of this forest. The work \cite{palmer1979number} establishes the asymptotic number of unlabeled forests on $p$ nodes as $F_p\sim Cp^{-5/2}\alpha^{-p}$.
Hence the number of unlabeled forests on at most $\ell M$ edges is upper bounded by $\sum_{p=2}^{2\ell M} F_{p} \le C\alpha^{-2\ell M}$. 

Next, we must specify how the blue $2$-circuit forest is attached to the rest of the graph.
We will encode the attachment point nodes in the $2$-circuit forest, the attachment point nodes in the rest, and the correspondence between them.
To see how many such encodings there are, note that if $u$ is a non-root $n$-node which appears in both the blue $2$-circuit forest and in the rest, it must appear in both $T_2$ and $S_2$ (by virtue of participating in a blue $2$-circuit) and in at least one of $T_1$ and $S_1$ through an alternating circuit, so $u$ must be an overdecorated node.
If $a$ is a $d$-node which appears in both the blue $2$-circuit forest and the rest, it must be incident on both a $2$-circuit and on a circuit of length $\ge 4$, so $a$ is a crossroad node.
All these nodes are graft witness nodes. Therefore, after including the possible two roots, their total number is at most $\ovd + \xr + 2$.
Because there are at most $5K$ nodes in the $2$-circuit forest and at most $5K$ nodes in the rest, there are at most $\sum_{j = 0}^{\ovd + \xr + 2} \binom{5K}{j}^2 j! \le (5K)^{2(\ovd + \xr + 2)}$ ways to encode the attachment.

\paragraph{Edge decorations and node labeling.} 
It remains to encode the edge decorations and the node labels.
We encode for each red edge whether it is decorated by $T_1$ or $S_1$, and for each blue edge whether it is decorated by $T_2$ or $S_2$.
Since there are $2\ell M$ edges in total, the number of encodings is at most $2^{2\ell M}$.
 
Finally, we encode the labels of the $n$-nodes and $d$-nodes. 
The root nodes $i$ and $j$ are uniquely determined since each is incident to $D$ red and blue edges, respectively. 
Thus, it suffices to encode the labels of the remaining nodes. 
By \eqref{eq:num-v} and \eqref{eq:num-a}, there are at most $(\ell M/2 - \ovd/2)$ $n$-nodes and $(\ell M - \kappa/2 - x/2)$ $d$-nodes. 
Thus, the number of possible labelings is at most
\$ 
n^{\ell M/2 - \ovd/2}d^{\ell M - \kappa/2 - x/2}.
\$

Multiplying the bounds together yields that the number of such encodings, and thus grafted parts, satisfies
\begin{align*} 
|\cUr(\ell, \kappa, \mu_1, \ovd, \xr)| 
\le C K^4 \left(Cn^{1/2} d\right)^{\ell M}\left(\frac{C K^{12}}{n} \right)^{\ovd/2}\left(\frac{CK^{12}}{d} \right)^{\xr/2} \left(\frac{1}{d}\right)^{\kappa/2} \left(CK\right)^{\ell M- \kappa - 2\mu_1},
\end{align*}
for $C$ a sufficiently large universal constant (depending on $\alpha)$.
\end{proof}

We are now ready to prove \Cref{prop:ur-cr} using \Cref{lmm:wr-b}. 
\begin{proof}[Proof of \Cref{prop:ur-cr}] Consider any grafted part with $\ell$ branches, that is, $U\in \cUr(\ell)$. Since each branch $J\subset \cJ$ satisfies $\aut(J) \le \exp(CM)$, it follows from the definition that
\$
w(U) \le \exp(C\ell M). 
\$    
Thus, the definition of $\re(\ell)$ in \eqref{eq:re-def} yields 
\#\label{eq:re-ell-1} 
\re(\ell) 
& \le \frac{\exp(C\ell M)}{n^{\ell M/2}d^{\ell M/2}\rho^{\ell M}} \sum_{(U, \bc)\in \cU_\rr(\ell)} \rho^{2(\ell M/2-\kappa(\bc))}  \left|\wg(\mu(\bc),d) \right|.
\#
Here $U$ contains $2\ell M$ edges in total. 
We now consider the grafted parts with fixed parameters $\kappa$, $\mu_1$, $\ovd$, and $\xr$. For any $(U,\bc)\in \cU_\rr(\ell,\kappa, \mu_1, \ovd, \xr)$,
by \eqref{eq:bound-mu1}, we have
\$ 
\rho^{2(\ell M/2-\kappa(\bc))}  \left|\wg(\mu(\bc),d) \right| \le (1+o(1))\left(\frac{\rho^2}{d}\right)^{\ell M/2-\kappa/2}\left(\frac{4}{d}\right)^{\ell M/2- \kappa /2- \mu_1}.
\$
Note that the parameters must satisfy $2\kappa +  4\mu_1\le 2\ell M$ since the total number of edges is $2\ell M$. We define $ a \triangleq \ell/4 - (\ell M/2- \kappa /2- \mu_1)$ as the lower bound on $\ovd + \xr$ given in \Cref{lmm:para-constr-ur}. Then, summing over all possible choices of these parameters yields
\#\label{eq:wr-ell-wg}
& \sum_{(U, \bc)\in \cU_\rr(\ell)} \rho^{2(\ell M/2-\kappa(\bc))}  \left|\wg(\mu(\bc),d) \right| \notag\\
& \qquad \le \sum_{\substack{\kappa, \mu_1 \ge 0 \\ \kappa +  2\mu_1\le \ell M}} \left(\frac{\rho^2}{d}\right)^{\ell M/2-\kappa/2}\left(\frac{4}{d}\right)^{\ell M/2- \kappa /2- \mu_1} \sum_{\substack{\ovd, \xr\ge 0 \\ \ovd + \xr \ge \max\{a,0\}}} \left|\cU_\rr(\ell,\kappa, \mu_1, \ovd, \xr)\right|. 
\#

Fix $\kappa$ and $\mu_1$ and recall that \Cref{lmm:wr-b} provides a bound on $|\cUr(\ell, \kappa, \mu_1, \ovd, \xr)|$. Substituting the bound, we obtain
\#\label{eq:wr-b}
& \sum_{\substack{\ovd, \xr\ge 0\\ \ovd + \xr \ge \max\{a,0\}}} |\cUr(\ell, \kappa, \mu_1, \ovd, \xr)| \notag\\
& \qquad \qquad \le CK^4 C^{\ell M} n^{\ell M/2}d^{\ell M - \kappa/2} (CK)^{\ell M- \kappa - 2\mu_1} \sum_{\substack{\ovd, \xr\ge 0\\ \ovd + \xr \ge \max\{a,0\}}} \left(\frac{C K^{12}}{n} \right)^{\ovd/2}\left(\frac{CK^{12}}{d} \right)^{\xr/2}\notag\\
& \qquad \qquad \le C K^4  n^{\ell M/2}d^{\ell M-\kappa/2} \left(\frac{\poly(K)}{\min\{d, n\}}\right)^{\ell/8}\left(\frac{\poly(K)}{\min\{d, n\}}\right)^{- (\ell M/4- \kappa /4- \mu_1/2)}.
\#
Here the last line follows from $C^{M} = \poly(K)$ and 
\$
\sum_{\substack{\ovd, \xr\ge 0 \\ \ovd + \xr \ge \max\{a,0\}}} \left(\frac{C K^{12}}{n} \right)^{\ovd/2}\left(\frac{CK^{12}}{d} \right)^{\xr/2} & \le \sum_{\substack{ \max\{a,0\} \le t\le 6K}} \left(\frac{CK^{13}}{\min\{d, n\}}\right)^{t/2} \\
&\le 2\left(\frac{CK^{13}}{\min\{d, n\}}\right)^{\max\{a,0\}/2},
\$
where the first inequality holds since $t = \ovd + \xr \le 6K$, the total number of nodes, and there are at most $6K+1$ ways to split $t$ into $\ovd$ and $\xr$, and the last inequality follows by bounding the sum with a geometric series (since ${CK^{13}}/{\min\{d, n\}}\le 1/2$) and cancelling the factor $(CK)^{\ell M- \kappa - 2\mu_1}$. 

Substituting \eqref{eq:wr-b} into \eqref{eq:wr-ell-wg}, we have 
\# \label{eq:re-2}
& \sum_{(U, \bc)\in \cU_\rr(\ell)} \rho^{2(\ell M/2-\kappa(\bc))}  \left|\wg(\mu(\bc),d) \right| \notag \\
& \qquad \qquad \le CK^4 n^{\ell M/2} d^{\ell M/2} \rho^{\ell M} \left(\frac{\poly(K)}{\min\{d, n\}}\right)^{\ell/8} \sum_{\substack{\kappa, \mu_1 \ge 0 \\ \kappa +  2\mu_1\le \ell M}} \rho^{-\kappa}\left(\frac{4}{d}\right)^{\ell M/4- \kappa /4- \mu_1/2} \notag\\
& \qquad  \qquad \le CK^4 n^{\ell M/2} d^{\ell M/2} \rho^{\ell M} \left(\frac{\poly(K)}{\min\{d, n\}}\right)^{\ell/8},
\#
where the last line follows from $\rho^{-\kappa} \le \rho^{-\ell M} = \poly(K)^\ell$ and 
\$ 
\sum_{\substack{\kappa, \mu_1 \ge 0 \\ \kappa +  2\mu_1\le \ell M}} \left(\frac{4}{d}\right)^{\ell M/4- \kappa /4- \mu_1/2} \le \sum_{\substack{a \ge 0 }} \left(\frac{8K}{d}\right)^{a} \le 2
\$
since there are at most $\ell M \le 2K$ number of ways to split each $a = \kappa +  2\mu_1$ into $\kappa$ and $\mu_1$.
Finally, combining \eqref{eq:re-ell-1} and \eqref{eq:re-2}, since $e^{M} = \poly(K)$, if $\ell>1$, we conclude that
\[
\re(\ell) \le C\left(\frac{\poly(K)}{\min\{d, n\}}\right)^{\ell/8}. \qedhere
\]
\end{proof}

\subsection{Bounding the Fully Overlapping Part} \label{sect:uo}
This section proves \Cref{prop:uo-co} for the fully overlapping part. Since the overlapping branches are taken directly from $\cJ$ and preserve a tree structure, we can perform a more explicit counting by analyzing how many non-isomorphic branches can appear, how they may be decorated, and how the nodes are labeled.
\begin{proof}[Proof of \Cref{prop:uo-co}]
Consider any $(U, \bc)\in \cUo(r,m,s)$.  
Recall that $U$ contains $r$ fully overlapping branch pairs, among which $m$ are monochromatic pairs and $s$ are bichromatic cross-isomorphism pairs. 
Then $U$ contains $4rM$ edges in total (counting multi-edges with multiplicity), and $\bc$ contains $2mM$ $2$-circuits and $(r-m)M$ $4$-circuits. Thus, by \eqref{eq:bound-mu1}, such $\bc$ satisfies
\$
\rho^{2(r M-\kappa(\bc))}  \left|\wg(\mu(\bc),d) \right|\le \left(1+o(1)\right)\left(\frac{\rho^2}{d}\right)^{(r-m)M}.
\$
Substituting this bound into \eqref{eq:ov-def}, we have
\#\label{eq:over-final} 
& \ov(r,m,s) \le \frac{1+o(1)}{|\cT|^2n^{rM}d^{(2r-m)M}\rho^{2mM}}\sum_{(U, \bc)\in \cUo(r,m,s)} w(U).  
\#
It therefore suffices to bound $\sum_{(U, \bc)\in \cUo(r,m,s)} w(U)$.
We first count the unlabeled graphs and then the labelings of their nodes.

\paragraph{Counting unlabeled graphs.}
Suppose that the $r$ fully overlapping branch pairs involve $r'$ non-isomorphic pairs. Since each non-isomorphic branch can appear in at most two overlapping pairs (there are four trees in total), we have $r/2\le r' \le r$. Moreover, since each of the $m+s$ monochromatic and bichromatic cross-isomorphism pairs has one decoration from $\{T_1,T_2\}$ and the other from $\{S_1,S_2\}$, the unlabeled trees $T$ and $S$ must share at least $(m+s)/2$ isomorphic branches. Consequently, $r' \le 2D - (m+s)/2$. Let $R(r) = \min\{r, 2D - (m+s)/2\}$. We obtain
\#\label{eq:r-range}
r/2 \le r' \le R(r).
\#

Among the $r'$ non-isomorphic branch pairs: $r-r'$ pairs appear twice among the four trees, and $2r'-r$ appear once. 
Thus, the number of choices for these unlabeled branch templates is at most 
\$ 
\binom{|\cJ|}{r'}\binom{r'}{r-r'}.
\$
We now decorate the branches. Since each branch appearing twice must receive complementary decorations across its two copies, it suffices to decorate the $r'$ non-isomorphic branches. Among these, $2r'-r$ appear only once. Of these, at least $(r'-D) \vee 0$ must be decorated by $(T_1, T_2)$, and at least $(r'-D) \vee 0$  by $(S_1, S_2)$, because each of $T$ and $S$ contains at most $D$ branches. The remaining branches each admit at most $\binom{4}{2}=6$ possible decoration choices. Thus, for each of the trees above, the number of decorations is at most
\$
\binom{2r'-r}{(r'-D) \vee 0, (r'-D) \vee 0, (2D-r) \wedge (2r'-r)} 6^{(2D - r')\wedge r'}.
\$
The decorations also determine the parallel edges in the branches. Monochromatic branches decorated by $(T_1, S_1)$ or $(T_2, S_2)$ contain two parallel edges between each pair of adjacent $n$-node and $d$-node, whereas bichromatic branches contain two parallel pairs of edges between each pair of adjacent $n$-nodes (see \Cref{fig:fully-overlap} for an illustration).
Note that the circuit decomposition is also uniquely determined by the unlabeled, decorated tree since by definition it consists of $2$-circuits in monochromatic branches and $4$-circuits in bichromatic branches.

We define $f(r', r) $ as the product of the two displayed expressions above
\# \label{eq:def-f-uo}
&f(r', r) = \binom{|\cJ|}{r'}\binom{r'}{r-r'}\binom{2r'-r}{(r'-D) \vee 0, (r'-D) \vee 0, (2D-r) \wedge (2r'-r)} 6^{(2D - r')\wedge r'}. 
\# 
Let $\widetilde{\cU}_\mathrm{o}(r,m,s)$ denote the set of unlabeled decorated graph-decomposition pairs. Then summing over all possible trees constructed above, it follows that 
\#\label{eq:wo-frr}
|\widetilde{\cU}_\mathrm{o}(r,m,s)| &\le \sum_{r/2\le r'\le R(r)}f(r', r).
\#
\Cref{lmm:f-r-r} will provide an upper bound on $\sum_{r/2\le r'\le R(r)}f(r', r)$.
Following from \Cref{lmm:f-r-r} and $2D - R(r) \ge 
(m+s)/2$ from \eqref{eq:r-range}, we have
\#\label{eq:unlabeled-uo}
|\widetilde{\cU}_\mathrm{o}(r,m,s)| \le 2\left(\frac{6D^2}{|\cJ|-2D}\right)^{(m+s)/2 } \binom{|\cJ|}{D}^2  \le 2|\cT|^2\beta^{(m+s)M/2},
\#
where the last inequality uses $|\cT| = \binom{|\cJ|}{D}$ together with $
\beta^{M} \ge 7D^2/ |\cJ|\ge 6D^2/(|\cJ|-2D)$.

\paragraph{Labeling the nodes.}
Fix an unlabeled graph. We now count the number of ways to label its $n$- and $d$-nodes. Each of the $r$ pairs of overlapping branches contains $M$ distinct $n$-nodes. Moreover, among these $r$ pairs, each of the $m$ monochromatic pairs contains $M$  $d$-nodes, while each of the $r-m$ bichromatic pairs contains $2M$ $d$-nodes. Hence, there are $rM$ $n$-nodes and $(2r-m)M$ $d$-nodes in total. Hence, the number of possible node labelings is at most $n^{rM}d^{(2r-m)M}$ choices.

We now account for overcounting. Different node labelings may yield the same labeled graph if a permutation of the nodes preserves the graph structure. Recall that $\aut(J)$ denotes the number of node permutations that produce the same labeled graph for a branch $J$.
Since $U$ is a tree composed of branches, any permutation that preserves each branch necessarily preserves the entire tree.
Hence, the number of node labelings that give rise to the same labeled graph $U$ is at least the product of $\aut(J)$ over the branches, that is, $\prod_{J\subset \cJ, J\subset U}\aut(J)$. Recall from the definition of $w(U)$ in \eqref{eq:weight-u-aut} that it assigns a weight of $\sqrt{\aut(J)}$ to each branch of the four trees in $U$. Since each branch of $U$ appears in exactly two of the four trees, we have
\$ 
w(U) = \prod_{L\in\{T_1, S_1, T_2, S_2\}}\prod_{J\in \cR(L), J\subset U} \sqrt{\aut(J)} = \prod_{J\subset \cJ, J\subset U}\aut(J).
\$
In particular, the overcounting is at least  $w(U)$. After dividing by this symmetry factor, the number of distinct labeled graphs is therefore at most
\#\label{eq:uo-label} 
\frac{n^{rM}d^{(2r-m)M}}{w(U)}.
\#
Combining \eqref{eq:unlabeled-uo} and \eqref{eq:uo-label}, we have
\#\label{eq:ov-sum-u} 
\sum_{\substack{(U, \bc)} \in \cUo(r, m, s)} w(U) & \le \left|\widetilde{\cU}_\mathrm{o}(r,m,s)\right| \sum_{\substack{(U, \bc)} \in \widetilde{\cU}_\mathrm{o}(r,m,s)} w(U) \times  \frac{n^{rM}d^{(2r-m)M}}{w(U)} \notag \\
&\le 2n^{rM}d^{(2r-m)M} |\cT|^2\beta^{(m+s)M/2}.
\#
We further combine \eqref{eq:over-final} and \eqref{eq:ov-sum-u} and obtain
\$ 
& \ov(r,m,s) \le 2(1+o(1))\beta^{(m+s)M/2}\rho^{-2mM}. 
\$
\end{proof}

Finally, we present the following auxiliary lemma to bound $\sum_{r/2 \le r'\le R(r)} f(r', r)$ in \eqref{eq:wo-frr}.
\begin{lemma}\label{lmm:f-r-r}
For any $D, r, r' \in \N$ such that $0 \le r \le 2D$, recall the definition of $f(r', r)$ from \eqref{eq:def-f-uo}.
Fix $R\le 2D$. 
Then 
\$ 
\sum_{r/2 \le r'\le R} f(r', r) \le 2\left(\frac{6D^2}{|\cJ|-2D}\right)^{2D-R} \binom{|\cJ|}{D}^2.
\$
\end{lemma}
\begin{proof}
For any $a,b\in\N$ such that $b/2 \le a \le b-1$ and $b\le 2D$, \cite[Claim 3]{mao2022random} shows that
\$ 
& \frac{f(a, b)}{f(a+1, b)} \le \frac{6D^2}{|\cJ|-2D} \le \frac{1}{2}, \quad \frac{f(a, a)}{f(a, a+1)} \le 1, \quad \frac{f(a, a)}{f(a+1, a+1)} \le \frac{6D^2}{|\cJ|-2D},
\$ 
and $f(2D, 2D) \le \binom{|\cJ|}{D}^2$. Then  for fixed $r$, we have
\$ 
\sum_{r/2\le r'\le R} f(r', r) & \le 2f(R, r) \le 2f(R, 2D) \\
& \le 2\left(\frac{6D^2}{|\cJ|-2D}\right)^{2D-R} f(2D, 2D) \le 2\left(\frac{6D^2}{|\cJ|-2D}\right)^{2D-R} \binom{|\cJ|}{D}^2,
\$
where the first inequality follows by summing over $r'$ and using $f(a, b)\le f(a+1, b)/2$, the second inequality uses $f(a, a)\le f(a, a+1)$, the third inequality follows from $f(a, b)\le f(a+1, b)\cdot 6D^2/(|\cJ|-2D)$ and the last inequality uses $f(2D, 2D) \le \binom{|\cJ|}{D}^2$.
\end{proof}

\section{Polynomial-Time Approximation via Color Coding}\label{sect:approx}
In this section, we provide a polynomial-time algorithm for approximating the similarity scores using the idea of color coding. The discussion follows closely that of \cite[Section 4]{mao2024testing} and \cite[Section 5]{mao2022random}.

Fix a rooted bipartite tree $T\in\cT$. We first approximate $\Label_T(X,i)$ defined in \eqref{eq:label-t-x-i}. We view $X$ as a weighted bipartite graph on vertex set $[n]\sqcup [d]$, where the edge between $u \in [n]$ and $a \in [d]$ has weight $X_{ua}$. We generate a random coloring $\phi = (\phi_n, \phi_d)$ as follows: $\phi_n \colon [n]\to [K+1]$ assigns to each $n$-node in $X$ an independent uniformly random color from $[K+1]$, and $\phi_d \colon [d]\to [K]$ assigns to each $d$-node an independent uniformly random color in $[K]$. 
For any $V = V_n \sqcup V_d$ with $V_n \subset [n]$ and $V_d \subset [d]$, let $\chi_{\phi}(V)$ denote the indicator that $V$ is colorful under $\phi$, meaning that $\phi_n(u) \neq \phi_n(v)$ for all distinct $u, v \in V_n$ and $\phi_d(a) \neq \phi_d(b)$ for all distinct $a, b \in V_d$. 
If $|V_n| = K+1$ and $|V_d| = K$, then 
\$ 
\prob{\chi_{\phi}(V) = 1} =  \frac{(K+1)!}{(K+1)^{K+1}} \cdot  \frac{K!}{K^K}  \triangleq p.
\$
We define
\$ 
\widetilde{\Label}_T(X,i; \phi) \triangleq \sum_{\substack{T_1 \cong T\\ \mathrm{root}(T_1) = i}} \chi_\phi(V(T_1)) X^{T_1}. 
\$
It follows that $$\E\left[\widetilde{\Label}_T(X,i; \phi) \given X\right] = p \cdot {\Label}_T(X,i).
$$
Hence $\widetilde{\Label}_T(X,i; \phi)/p$ is an unbiased estimator of ${\Label}_T(X,i)$. Similarly, we also view $Y$ as a weighted bipartite graph, generate a random coloring $\psi$ on $Y$, and define 
\$ 
\widetilde{\Label}_T(Y,j; \psi) \triangleq \sum_{\substack{T_2 \cong T\\ \mathrm{root}(T_2) = j}} \chi_\psi(V(T_2)) Y^{T_2}. 
\$

Since $T$ is a tree, color coding together with the recursive structure of $T$ allows us to use dynamic programming to count weighted colorful copies and compute $\widetilde{\Label}_T(X,i;\phi)$ efficiently. This procedure is described in \cite[Algorithm 2]{mao2024testing} for unrooted trees; with minor modifications, the same algorithm applies to rooted trees, as outlined in \cite[Section 5]{mao2022random}. The algorithm takes a weighted host graph on $[n]$ as input. For the rooted bipartite trees considered here, we apply it to the bipartite weighted host graphs induced by $X$ and $Y$.

To reduce the variance, we generate independent colorings $\phi_1, \cdots, \phi_t$ with $t \triangleq \ceil{1/p}$, and define 
\#\label{eq:hat-label} 
\widehat{\Label}_T(X,i) \triangleq \frac{1}{tp}\sum_{s=1}^t \widetilde{\Label}_T(X,i; \phi_s).
\#
Similarly, using independent colorings $\psi_1, \cdots, \psi_t$ for $Y$, we define $\widehat{\Label}_T(Y,j)$. We then define the approximate similarity score 
\#\label{eq:approx-similarity} 
\widehat\similarity_{ij} \triangleq \sum_{T \in \calT} \aut(T)\cdot \widehat\Label_T(X,i) \cdot \widehat\Label_T(Y,j).
\#
Then $\E\left[\widehat\similarity_{ij}\given X, Y\right] = \similarity_{ij}$. The procedure is summarized as \Cref{alg:approx-similarity}.

\vspace{8pt}
\begin{algorithm}[H]
\DontPrintSemicolon
  
  \KwIn{Data matrices $X$ and $Y$ and integers $M$, $D$}

  {Construct $\cJ$ as specified in \Cref{def:j}: first use the algorithms of \cite{beyer1980constant,colbourn1981linear} to construct rooted trees on $[n]$ with $M$ edges, root degree $1$, and at most $\exp(CM)$ automorphisms; then subdivide each edge by inserting a $d$-node.\;}
  
  {Construct $\cT$ using $\cJ$ according to \Cref{def:t}. \;}
  {Generate \iid random colorings $\{\phi_s\}_{s=1}^t$ and $\{\psi_s\}_{s=1}^t$. \;}
  \For{each $T\in\cT$}{
    \For{each $s=1,\cdots, t$}{
       Compute $\{\widetilde\Label_T(X,i;\phi_s)\}_{i\in[n]}$, $\{\widetilde\Label_T(Y,j;\psi_s)\}_{j\in[n]}$ via \cite[Algorithm 2]{mao2024testing}. 
    }
    {
       Compute $\{\widehat\Label_T(X,i)\}_{i\in[n]}$ and $\{\widehat\Label_T(Y,j)\}_{j\in[n]}$ via \eqref{eq:hat-label}.
    }}
  
  \KwOut{$\{\widehat\similarity_{ij}\}_{i,j \in [n]}$ according to \eqref{eq:approx-similarity}.}
  
\caption{Approximation of similarity scores via color coding}
\label{alg:approx-similarity}
\end{algorithm}
\vspace{8pt}

Finally, we prove \Cref{prop:approx} by showing that \Cref{alg:approx-similarity} runs in polynomial time and that Theorem  \ref{thm:algo} continues to hold with $\widehat{\similarity}_{ij}$ in place of $\similarity_{ij}$.

\begin{proof} 
We analyze the time complexity and control the variance in turn.
\paragraph{Time complexity.}    
We first show that \Cref{alg:approx-similarity} outputs $\{\widehat{\similarity}_{ij}\}_{i,j\in[n]}$ in polynomial time. The algorithm in \cite{beyer1980constant} generates all non-isomorphic rooted trees with $M$ edges in time $O(\alpha^{-M})$. Moreover, for each such rooted tree, the algorithm in \cite{colbourn1981linear} computes the number of automorphisms in time $O(M)$. Hence, the total time complexity to generate $\cJ$ is $O(M\alpha^{-M})$. Given $\cJ$, the time required to construct $\cT$ is $O\!\left(\binom{|\cJ|}{D}\right) = O(|\cT|)$.

As discussed above, we apply \cite[Algorithm 2]{mao2024testing} to compute $\{\widetilde\Label_T(X,i;\phi_s)\}_{i\in[n]}$ and $\{\widetilde\Label_T(Y,j;\psi_s)\}_{j\in[n]}$ for all $T\in \cT$ and $s\in[t]$. Since $t = \ceil{1/p} = e^{2K}$, \cite[Lemma 2]{mao2024testing} implies that the time complexity to compute $\{\widehat\Label_T(X,i)\}_{i\in[n]}$ and $\{\widehat\Label_T(Y,j)\}_{j\in[n]}$ is $O(|\cT|K(3e)^{2K}nd)$. Here the factor $nd$ comes from the bipartite structure: every edge connects an $n$-node to a $d$-node, so each recursion step sums over pairs in $[n]\times[d]$. 

In summary, since $|\cT| \le |\cJ|^D \le \alpha^{-K}$, 
the total time complexity is
\$
O\left(M\alpha^{-M} + |\cT| + |\cT|K(3e)^{2K}nd\right)
= O\left(K[(3e)^2/\alpha]^{K}nd\right) = O(n^Cd),
\$
where constant $C= C_1\log((3e)^2/\alpha) +2 >0 $ by the choice of $K$ given in \Cref{thm:algo}.

\paragraph{Variance bound.}
Next, to show that \Cref{thm:algo} continues to hold with $\widehat{\similarity}_{ij}$ in place of $\similarity_{ij}$, it suffices to establish an analogue of \Cref{prop:mean-sep,prop:var-control} for $\widehat{\similarity}_{ij}$; that is,
\#\label{eq:poly-appx-mean-var}
\E[\widehat\similarity_{ii}] \ge (1-o(1))\nu, \quad \E[\widehat\similarity_{ij}] = 0, \quad \var(\widehat\similarity_{ii})= o(\nu^2), \quad \var(\widehat\similarity_{ij}) = o(\nu^2/n^2), 
\#
where $\nu \triangleq |\cT| n^K d^K \rho^{2K}$.
For two independent colorings $\phi$ and $\psi$, we define
\$ 
S_{ij}(\phi,\psi) \triangleq \sum_{T\in\cT}\aut(T) \cdot \widetilde{\Label}_T(X,i; \phi) \cdot \widetilde{\Label}_T(Y,j; \psi).
\$
Then from \eqref{eq:approx-similarity}, we have
\$ 
\widehat\similarity_{ij} = \frac{1}{t^2p^2}\sum_{a=1}^t \sum_{b=1}^t S_{ij}(\phi_a,\psi_b).
\$
For any $a,b\in[t]$, $S_{ij}(\phi_a,\psi_b)/p^2$ is an unbiased estimator of $\similarity_{ij}$, since 
\#\label{eq:unbiased-score}
\E\left[S_{ij}(\phi_a,\psi_b)\given X, Y\right] = p^2\sum_{T\in\cT}\aut(T) \cdot {\Label}_T(X,i) \cdot {\Label}_T(Y,j) = p^2 \similarity_{ij}.
\#
Thus, we have $\E\left[\widehat\similarity_{ij}  \given X, Y\right] = \similarity_{ij}$ and  $\E\left[\widehat\similarity_{ij}\right] = \E[\similarity_{ij}]$ for any $i,j$. Hence, the mean result in \eqref{eq:poly-appx-mean-var} follows directly from \eqref{eq:mean-main}. It remains to bound the variance. Since $\E\left[\widehat\similarity_{ij}  \given X, Y\right] = \similarity_{ij}$, we have $\cov(\widehat\similarity_{ij} - \similarity_{ij}, \similarity_{ij}) = 0$; hence
\$ 
\var(\widehat\similarity_{ij}) 
 = \var(\widehat\similarity_{ij} - \similarity_{ij}) + \var(\similarity_{ij}).
\$
Thus, it suffices to bound $\var(\widehat\similarity_{ij} - \similarity_{ij})$.
In particular, we have 
\# \label{eq:appx-var}
\var(\widehat\similarity_{ij} - \similarity_{ij}) & = \var\left(\E\left[\widehat\similarity_{ij} - \similarity_{ij} \given X, Y\right]\right) + \E\left[\var\left(\widehat\similarity_{ij} - \similarity_{ij} \given X, Y\right)\right] \notag\\
&=\E\left[\var\left(\widehat\similarity_{ij} \given X, Y\right)\right] \notag\\
& = \frac{1}{t^4p^4}\sum_{a=1}^t \sum_{b=1}^t\sum_{c=1}^t \sum_{d=1}^t \E\left[\cov\left(S_{ij}(\phi_a,\psi_b), S_{ij}(\phi_c,\psi_d) \given X, Y\right)\right].
\#
Here $S_{ij}(\phi_a,\psi_b)$ and $S_{ij}(\phi_c,\psi_d)$ are independent if and only if $a\neq c$ and $b\neq d$.
Fix $i,j$ and $a,b,c,d$. We compute the covariance as follows,
\$
& \cov\left(S_{ij}(\phi_a,\psi_b), S_{ij}(\phi_c,\psi_d) \given X, Y\right) \\
& \qquad = \E\left[S_{ij}(\phi_a,\psi_b) S_{ij}(\phi_c,\psi_d) \given X, Y\right] - \E\left[S_{ij}(\phi_a,\psi_b)\given X, Y\right] \E\left[S_{ij}(\phi_c,\psi_d) \given X, Y\right] \\
& \qquad = \E\left[S_{ij}(\phi_a,\psi_b) S_{ij}(\phi_c,\psi_d) \given X, Y\right] - p^4\similarity_{ij}^2,
\$
where the last equality follows from \eqref{eq:unbiased-score}. We now compute the first term
\$ 
& \E\left[S_{ij}(\phi_a,\psi_b) S_{ij}(\phi_c,\psi_d) \given X, Y\right] = \sum_{T,S\in\cT} \aut(T)\aut(S)\\
& \qquad \times \sum_{\substack{T_1,T_2\cong T, S_1, S_2\cong S\\ \mathrm{root}(T_1) = \mathrm{root}(S_1) = i \\ \mathrm{root}(T_2) = \mathrm{root}(S_2) = j}} \E[\chi_{\phi_a}(V(T_1))\chi_{\phi_c}(V(S_1))]\E[\chi_{\psi_b}(V(T_2))\chi_{\psi_d}(V(S_2))]X^{T_1}Y^{T_2}X^{S_1}Y^{S_2},
\$
where $\chi_{\phi_a}(V(T_1))\chi_{\phi_c}(V(S_1))$ and $\chi_{\psi_b}(V(T_2))\chi_{\psi_d}(V(S_2))$ are independent since they are defined on vertices from different graphs $X$ and $Y$, respectively. If $V(T_1) \cap V(S_1) = \{i\}$ and $V(T_2) \cap V(S_2) = \{j\}$, since $T_1$ (resp. $T_2$) only shares one common node with $S_1$ (resp. $S_2$), we have
\$ 
\E[\chi_{\phi_a}(V(T_1))\chi_{\phi_c}(V(S_1))] = \E[\chi_{\psi_b}(V(T_2))\chi_{\psi_d}(V(S_2))] = p^2,
\$
In general, it is straightforward to verify that 
\$ 
\E[\chi_{\phi_a}(V(T_1))\chi_{\phi_c}(V(S_1))] \le p^{1+ \indc{a\neq c}} , \quad \E[\chi_{\psi_b}(V(T_2))\chi_{\psi_d}(V(S_2))] \le p^{1+ \indc{b\neq d}}.
\$
Combining the equations above, we have 
\$ 
& \cov\left(S_{ij}(\phi_a,\psi_b), S_{ij}(\phi_c,\psi_d) \given X, Y\right) \le (p^{2+ \indc{a\neq c} + \indc{b\neq d}} -p^4) \sum_{T,S\in\cT} \aut(T)\aut(S) \\
& \qquad \times \sum_{\substack{T_1,T_2\cong T, S_1, S_2\cong S\\ \mathrm{root}(T_1) = \mathrm{root}(S_1) = i \\ \mathrm{root}(T_2) = \mathrm{root}(S_2) = j}}  X^{T_1}Y^{T_2}X^{S_1}Y^{S_2}\indc{V(T_1) \cap V(S_1) \neq \{i\} \text{ or } V(T_2) \cap V(S_2) \neq \{j\}}.
\$ 
For $i = j$, an important observation is that if either $V(T_1) \cap V(S_1) \neq \{i\}$ or $V(T_2) \cap V(S_2) \neq \{i\}$, that is, $T_1$ and $S_1$ or $T_2$ and $S_2$ share a non-root node,  then their decorated union graph cannot consist solely of bichromatically overlapping branches; hence, $m+ \ell \ge 1$. Therefore, by the same argument as in \Cref{lmm:var-step1,lmm:var-ratio},  we have
\$
& \sum_{T,S\in\cT} \aut(T)\aut(S)\sum_{\substack{T_1,T_2\cong T, S_1, S_2\cong S\\ \mathrm{root}(T_1) = \mathrm{root}(S_1) = i \\ \mathrm{root}(T_2) = \mathrm{root}(S_2) = i}}  \E\left[X^{T_1}Y^{T_2}X^{S_1}Y^{S_2}\right]\indc{V(T_1) \cap V(S_1) \neq \{i\} \text{ or } V(T_2) \cap V(S_2) \neq \{i\}} \\
& \qquad \le \left(1 + o(1)\right) \nu^2\sum_{\substack{r\ge m,s\ge 0, \, \ell\ge 0  \\ 2r+\ell = 4D, \, m+ \ell \ge 1}} \ov(r,m,s) \re(\ell)  = o(\nu^2),
\$
where the last equality follows from \eqref{var-ii-final} and \eqref{eq:var-bd-ii}. Hence, we have
\#\label{eq:exp-cov-color-ii}
\E\left[\cov\left(S_{ii}(\phi_a,\psi_b), S_{ii}(\phi_c,\psi_d) \given X, Y\right)\right] \le (p^{2+ \indc{a\neq c} + \indc{b\neq d}} -p^4) o(\nu^2).
\#
For $i\neq j$, it suffices to bound the second moment. Combining \eqref{eq:var-ij-ratio-1}, \eqref{eq:var-ij-ratio-2}, and \eqref{eq:2nd-mmnt-bd-ij} gives 
\$ 
\sum_{T,S\in\cT} \aut(T)\aut(S) \sum_{\substack{T_1,T_2\cong T, S_1, S_2\cong S\\ \mathrm{root}(T_1) = \mathrm{root}(S_1) = i \\ \mathrm{root}(T_2) = \mathrm{root}(S_2) = j}}  \E\left[X^{T_1}Y^{T_2}X^{S_1}Y^{S_2}\right] \le o(\nu^2/n^2).
\$
It follows that
\#\label{eq:exp-cov-color-ij}
\E\left[\cov\left(S_{ij}(\phi_a,\psi_b), S_{ij}(\phi_c,\psi_d) \given X, Y\right)\right] & \le (p^{2+ \indc{a\neq c} + \indc{b\neq d}} -p^4) o(\nu^2/n^2).
\#

Note that when $t > 1/p$,
\$ 
\frac{1}{t^4p^4}\sum_{a=1}^t \sum_{b=1}^t\sum_{c=1}^t \sum_{d=1}^t \left(p^{2+ \indc{a\neq c} + \indc{b\neq d}} - p^{4} \right) \le \frac{1}{t^2p^2} + \frac{2}{tp} \le 3.
\$
Hence, combining \eqref{eq:appx-var}, \eqref{eq:exp-cov-color-ii}, and \eqref{eq:exp-cov-color-ij}, we conclude that
\$ 
\var(\widehat\similarity_{ij} - \similarity_{ij}) & \le \frac{1}{t^4p^4}\sum_{a=1}^t \sum_{b=1}^t\sum_{c=1}^t \sum_{d=1}^t \left(p^{2+ \indc{a\neq c} + \indc{b\neq d}} - p^{4} \right)o\left(\nu^2\indc{i = j} + \nu^2\indc{i \neq j}/n^2\right) \\
& = o\left(\nu^2\indc{i = j} + \nu^2\indc{i \neq j}/n^2\right).
\$
This verifies \eqref{eq:poly-appx-mean-var}. Therefore, by following the proof of \Cref{thm:algo} with $\widehat{\similarity}_{ij}$ in place of $\similarity_{ij}$, we conclude that Theorem  \ref{thm:algo} continues to hold with $\widehat{\similarity}_{ij}$.
\end{proof}

\section{From Almost Exact to Exact Recovery}\label{sect:exact}
With high probability, \Cref{alg:main} finds a permutation $\widehat{\pi} \colon I\to [n]$ such that $\widehat{\pi} = \pi \mid_I$ and $|I| \ge n - n/\log^2 n$.
In this section, we provide a seeded matching algorithm to match the remaining data points, thereby achieving exact recovery. Given a seed set $J$, the basic intuition is as follows: if $(i,j)$ is a true pair, then the vectors $(\inner{x_i, x_r})_{r\in J}$ and $(\inner{y_j, y_{\pi(r)}})_{r\in J}$ are correlated, whereas if $(i,j)$ is a false pair, then these vectors are independent. Hence, we determine whether $i$ and $j$ form a true pair by thresholding the following similarity score:
\$ 
\Psi_{ij}^J \triangleq \sum_{r\in J} \indc{x_i^\top x_r y_j^\top y_{\pi(r)} \ge 0}.
\$
\Cref{alg:exact-while} iteratively adds data points as new seeds once we are confident, based on the current seed set, that they form true pairs. 

\vspace{8pt}
\begin{algorithm}[H]
\DontPrintSemicolon
\KwIn{Data matrices $X$ and $Y$, a mapping $\widehat \pi\colon I\to[n]$, and a threshold $\gamma >0$.}  
{Let $J = I$ and $\widetilde{\pi} = \widehat{\pi}$.\;}
\While{there exist $i\notin J$ and $j\notin \widetilde{\pi}(J)$ such that $\Psi_{ij}^{J} \ge \gamma n $}{Add $i$ to $J$ and let $\widetilde{\pi}(i) = j$.}
\KwOut{$\widetilde\pi\colon [n]\to [n]$.}
\caption{Achieving exact recovery via seeded matching}
\label{alg:exact-while}
\end{algorithm}
\vspace{8pt}

Since the initial seed set $I$ depends on the data matrices $X$ and $Y$, we show that \Cref{alg:exact-while} succeeds for every possible set $I$ 
satisfying $|I|\ge n- n/\log^2 n$. The following proposition gives sufficient conditions under which our seeded algorithm achieves exact recovery.
\begin{proposition}
Suppose that $\rho^2 \ge C\max\{\log n/d, \, \sqrt{1/\log n}\}$ for a sufficiently large constant $C$. Then there exists $\gamma>0$ such that, with probability $1-n^{-\Omega(1)}$, the following holds: for every partial mapping $\widehat{\pi} 
= \pi \mid_I$, where $I\subset [n]$ satisfies $|I| \ge n - n/\log^2 n$, \Cref{alg:exact-while}, when run with inputs $\widehat{\pi}$ and $\gamma$, outputs $\widetilde{\pi}=\pi$ in $O(n^3d)$ time.
\end{proposition}

\begin{proof}
Without loss of generality, assume that $\pi = \id$ and $Q=I_d$. We choose $0<p<q$ and $\delta>0$:
\#\label{eq:def-p-q}
p=\frac12+\frac1\pi\arcsin\left(\frac{\rho^2}{4}\right),
\quad
q=\frac12+\frac1\pi\arcsin\left(\frac{3\rho^2}{4}\right), \quad \delta = \frac{q-p}{2(p+2q)},
\#
where $(1-2\delta)q>(1+\delta)p$. We set the threshold $\gamma = (1+\delta) p$.

We first show that with high probability, no false pair exceeds the threshold $\gamma n$:
\$ 
\prob{\exists \, i\neq j\colon \Psi_{ij}^{[n]\setminus\{i,j\}} \ge \gamma n } \le o(1).
\$
Fix $i\neq j\in[n]$. For independent standard Gaussian vectors $x,z\sim N(0, I_d)$ and $y = \rho x+ \sqrt{1-\rho^2}z$, we define the conditional probability
\#\label{eq:pij} 
p_{ij}\triangleq \prob{x_i^\top x \cdot y_j^\top y \ge 0 \given x_i, y_j}. 
\#
Conditioned on $x_i, y_j$, the events $\{x_i^\top x_r y_j^\top y_r \ge 0\}$ for $r\in [n]\setminus\{i,j\}$ are independent. Therefore, $\Psi_{ij}^{[n]\setminus\{i,j\}}\sim \Bin(n-2, p_{ij})$ conditioned on $x_i, y_j$. 
On the event $\{p_{ij} \le p\}$, this conditional binomial random variable is stochastically dominated by $\Bin(n-2, p)$. Thus, the Chernoff bound yields
\$ 
\prob{\Psi_{ij}^{[n]\setminus\{i,j\}} \ge (1+\delta)np \given x_i, y_j}\cdot \indc{p_{ij}\le p} \le \exp(-c\delta^2 np) \cdot \indc{p_{ij}\le p},
\$  
for all sufficiently large $n$ and $c>0$ is an absolute constant. 
Thus, taking expectations over $x_i, y_j$ on both sides, we have
\$ 
\prob{\left\{\Psi_{ij}^{[n]\setminus\{i,j\}} \ge (1+\delta)np\right\}\cap\left\{p_{ij}\le p \right\}} \le \exp(-c\delta^2 np).
\$
Let $\cE = \{p_{ij} \le p, \; \forall i\neq j\}$. 
\Cref{lmm:p-q-sep} will establish that $\prob{\cE} \ge 1-n^{-\Omega(1)}$ under the assumption that $\rho^2\ge C\log n/d$ for a sufficiently large constant $C$.
Therefore, a union bound over all $\binom{n}{2}$ false pairs $(i,j)$ with $i\neq j$ yields that
\$ 
\prob{\exists \, i\neq j\colon \Psi_{ij}^{[n]\setminus\{i,j\}} \ge (1+\delta)np } & \le \prob{\cE^c} +  \sum_{i\neq j}\prob{\Psi_{ij}^{[n]\setminus\{i,j\}} \ge (1+\delta)np, \; \cE }  \\
&\le n^{-\Omega(1)} + n^2 \exp(-c\delta^2 np) = n^{-\Omega(1)},
\$
where the last equality holds when $\delta^2 \ge C\log n /n$ for a sufficiently large $C>0$ (equivalently, $\rho^2 \ge C\sqrt{\log n /n}$ since $\delta = \Theta(\rho^2)$). 
Henceforth, we assume that $\Psi_{ij}^{[n]\setminus\{i,j\}} < (1+\delta)np$ for all $i\neq j$. 

We prove by induction that $\widetilde{\pi} = \id \mid_J$ throughout the algorithm. At initialization, $J=I$ and $\widetilde{\pi} = \widehat{\pi} = \id\mid_I$ by assumption. Suppose that $\widetilde{\pi} = \id \mid_J$ after $t$ iterations. At iteration $t+1$, suppose the algorithm selects $i,j\notin J$. If $i\neq j$, by monotonicity ($J\subset [n]\setminus\{i,j\}$), we have
\$ 
\Psi_{ij}^J \le \Psi_{ij}^{[n]\setminus\{i,j\}} < (1+\delta)np.
\$
Therefore, if $i$ is added to $J$, we must have $j = i$ and $\widetilde{\pi}(i) = i$ is correct.

Next, we show that the algorithm matches all datapoints so that $\widetilde{\pi} = \id$. Suppose, for contradiction, that the algorithm terminates with $J\subsetneq [n]$ with $|J^c| \le n/\log^2n.$ Then for every $i\in J^c$, $\Psi_{ii}^J < (1+\delta)np$ (otherwise the while-loop would not have stopped). Summing over all $i\in J^c$, 
\begin{align}
e(J, J^c) \triangleq \sum_{i\in J^c} \Psi_{ii}^J < (1+\delta) np |J^c|.
\label{eq:e_J_upper}
\end{align}
Here $e(J, J^c) = \sum_{i\in J^c} \sum_{r\in J} \indc{x_i^\top x_r y_i^\top y_r \ge 0}$ counts the number of edges between $J$ and $J^c$ in the random geometric graph $G$ on $[n]$, where $(i,r) \in E(G)$ if and only if $x_i^\top x_r y_i^\top y_r \ge 0$. 

On the contrary, each edge exists with a constant marginal probability, so \Cref{lmm:G-expan} 
shows the edge expansion property of $G$: with probability $1-o(1)$, for all $J$ with $|J^c| \le n/\log^2n$, 
\begin{align}
e(J, J^c) \ge (1-\delta) q|J||J^c| 
\ge (1-2\delta) nq|J^c|, 
\label{eq:e_J_lower}
\end{align}
where the last inequality holds for all sufficiently large $n$.
Note that $\delta$ was chosen so that $(1-2\delta) q > (1+\delta) p $. 
Thus, we arrive at a contradiction between~\eqref{eq:e_J_upper} and~\eqref{eq:e_J_lower}. Therefore $J = [n]$ and $\widetilde{\pi} = \id$.

\paragraph{Running time.} Initially, there are at most $O(n^2)$ unmatched candidate pairs, and each score sums over at most $n$ seeds. Since computing each pair of inner products $x_i^\top x_r$ and $y_j^\top y_r$ takes $O(d)$ time, computing all initial scores $\{\Phi_{ij}^I\}$ costs at most $O(n^3d)$. During the algorithm, whenever a new seed $r$ is added to $J$, we update the score for each unmatched pair $(i,j)$ by checking whether this new seed contributes one additional term, namely whether $x_i^\top x_r y_j^\top y_{\widetilde\pi(r)} \ge 0$. Since there are at most $O(n^2)$ unmatched pairs and each check takes $O(d)$ time, each iteration costs $O(n^2d)$. There are at most $O(n)$ iterations, so the total update cost is $O(n^3d)$. Together with the initialization cost, the total running time is $O(n^3d)$.
\end{proof}

The following lemma establishes the separation between $p_{ij}$ for $i\neq j$ and $p_{ii}$.
\begin{lemma} \label{lmm:p-q-sep}
Let $p_{ij}$ be the conditioned probability defined in \eqref{eq:pij} for all $i,j\in [n]$, and let $p$ and $q$ be as defined in \eqref{eq:def-p-q}. We consider the event
\$ 
\cE_{\mathrm{sep}} = \left\{p_{ij}\le p \text{ for all }i\neq j, \; p_{ii}\ge q \text{ for all }i\in[n]\right\}.
\$
If $\rho^2\ge C\log (n)/d$ for a sufficiently large constant $C$, then $\prob{\cE_{\mathrm{sep}}}= 1-n^{-\Omega(1)}$.
\end{lemma}
\begin{proof}
For $i,j\in[n]$, we define 
\$
\xi_i \triangleq \frac{x_i^\top x}{\|x_i\|}, \quad \zeta_j \triangleq \frac{y_j^\top y}{\|y_j\|}, \quad\theta_{ij}\triangleq \frac{\rho x_i^\top y_j}{\|x_i\|\|y_j\|}.\$ 
Conditioned on $x_i$ and $y_j$, the pair $(\xi_i,\zeta_j)$ is Gaussian with mean zero, where
\$
\var(\xi_i\mid x_i,y_j)=1,
\qquad
\var(\zeta_j\mid x_i,y_j)=1,
\$
and
\[
\cov(\xi_i,\zeta_j\mid x_i,y_j)
=
\frac{x_i^\top \, \E[xy^\top]\, y_j}{\|x_i\|\,\|y_j\|}
=\frac{\rho x_i^\top y_j}{\|x_i\|\,\|y_j\|}
= \theta_{ij}.
\]
Hence, by Sheppard's formula \cite{sheppard1899iii} for a standard bivariate normal with correlation $\theta_{ij}$,
\$
p_{ij} = \prob{\xi_i\zeta_j>0\mid x_i,y_j}
=
\frac12+\frac{\arcsin(\theta_{ij})}{\pi}.
\$
Since $\arcsin$ is strictly increasing on $[-1,1]$, it suffices to show that with probability
$1-o(1)$,
\$
\theta_{ij}\le  \frac{\rho^2}{4},
\quad
\theta_{ii}\ge \frac{3\rho^2}{4},\quad \text{for all } i\neq j\in [n].
\$

Fix constants $\delta=1/20$ and $\tau=|\rho|/10$. By Gaussian norm concentration, we have
\$ 
\prob{\left|\|x_i\| -\sqrt{d}\right| \ge \delta\sqrt{d} } \le 2 \exp(-cd), \quad \prob{\left|\|y_j\| -\sqrt{d}\right| \ge \delta\sqrt{d} } \le 2 \exp(-cd).
\$
If $i\neq j$, we have $\E[x_i^\top y_j] = 0$ since $x_i$ and $y_j$ are independent. Moreover, $x_i^\top y_j = \sum_{k = 1}^d x_{ik} y_{jk}$ is a sum of independent sub-exponential random variables. Thus, Bernstein's inequality gives
\$ 
\prob{\left|x_i^\top y_j\right| > \tau d} \le 2\exp(-c\rho^2 d).
\$
On the event  
\$
\left|\|x_i\| -\sqrt{d}\right| \le \delta\sqrt{d}, \quad \left|\|y_j\| -\sqrt{d}\right| \le \delta\sqrt{d}, \quad \left|x_i^\top y_j\right| \le \tau d,
\$ we have 
\$
\theta_{ij} \le \frac{\rho\tau d}{(1-\delta)^2 d} \le \frac{\rho^2}{4}.
\$

If $i = j$, we have $\E[x_i^\top y_i] = \rho\E[x_i^\top x_i] + \sqrt{1-\rho^2}\E[x_i^\top z_i] = \rho d$. Since $x_i^\top y_i$ is a sum of independent sub-exponential random variables, Bernstein's inequality yields that
\$ 
\prob{\left|x_i^\top y_i - \rho d\right| > \tau d} \le 2\exp(-c\rho^2 d).
\$
On the event  
\$
\left|\|x_i\| -\sqrt{d}\right| \le \delta\sqrt{d}, \quad \left|\|y_i\| -\sqrt{d}\right| \le \delta\sqrt{d}, \quad \left|x_i^\top y_i - \rho d\right| \le \tau d,
\$ we have 
\$
\theta_{ii} \ge \frac{(\rho^2 - |\rho|\tau) d}{(1+\delta)^2 d} \ge \frac{3\rho^2}{4}.
\$

In summary, for each false pair $(i,j)$, failure occurs with probability at most
$Ce^{-c\rho^2d}$, and the same holds for each true pair $(i,i)$. A union bound over all
$n(n-1)$ false pairs and all $n$ true pairs yields
\[
\prob{\exists\, i\neq j\colon \theta_{ij}>\rho^2/4
\text{ or } \exists\, i\colon \theta_{ii}<3\rho^2/4}
\le C n^2 e^{-c\rho^2d}.
\]
Hence, if $\rho^2\ge C\log (n)/d$ for $C$ large enough, the error probability is at most $C n^2 e^{-c\rho^2d}=n^{-\Omega(1)}$.
\end{proof}

The following lemma shows the edge expansion property of a random geometric graph $G$.
\begin{lemma}\label{lmm:G-expan}
Let $G$ be the random geometric graph on $[n]$, where each node $i\in[n]$ is independently associated with a pair $(x_i, y_i)$, with $x_i, z_i\sim N(0,I_d)$ independently and $y_i = \rho x_i + \sqrt{1-\rho^2}z_i$. There is an edge between nodes $i$ and $r$ if and only if $x_i^\top x_r y_i^\top y_r \ge 0$. Let $e(J, J^c)$ denote the number of edges between $J$ and $J^c$ in $G$. Suppose that $\rho^2 \ge \sqrt{C /\log n}$ for a sufficiently large constant $C$. Then, with probability at least $1-o(1)$, for all $J\subset [n]$ with $|J^c| \le n/\log^2 n$, \$
e(J, J^c) \ge (1-\delta)q|J||J^c|,
\$
where $q$ and $\delta$ are defined in \eqref{eq:def-p-q}.
\end{lemma}
\begin{proof}
The statement trivially holds when $J^c = \emptyset$. Fix a nonempty set $J\subset[n]$ with $|J^c| = k$ and $1\le k\le n/\log^2n$. We have  
\$ 
e(J, J^c) = \sum_{r\in J} \sum_{i\in J^c} \indc{x_i^\top x_r y_i^\top y_r \ge 0} = \sum_{r\in J} M_r,
\$
where $M_r \triangleq \sum_{i\in J^c} \indc{x_i^\top x_r y_i^\top y_r\ge 0}$. Conditional on $\{x_i, y_i\}_{i\in J^c}$, by definition, we have
\$ 
\E\left[e(J, J^c) \given \{x_i, y_i\}_{i\in J^c} \right] = \sum_{r\in J} \E\left[M_r \given \{x_i, y_i\}_{i\in J^c} \right] = \sum_{r\in J}\sum_{i\in J^c} p_{ii} = (n-k)\sum_{i\in J^c} p_{ii}.
\$
Let $\cE_{J^c} = \{p_{ii} \ge q, \forall i \in J^c\}$. Then it holds that
\$ 
\E\left[e(J, J^c) \given \{x_i, y_i\}_{i\in J^c} \right] \cdot\Ind_{\cE_{J^c}} = (n-k)\sum_{i\in J^c} p_{ii} \cdot\Ind_{\cE_{J^c}} \ge (n-k)kq \cdot\Ind_{\cE_{J^c}}.
\$
Moreover, conditional on $\{x_i, y_i\}_{i\in J^c}$, the random variables $\{M_r\}$ for $r\in J$ are independent. Each $M_r$ satisfies $0 \le M_r \le k$. Thus, Hoeffding's inequality gives that
\$ 
\prob{e(J, J^c) \le (1-\delta) (n-k)k q \given \{x_i, y_i\}_{i\in J^c}} \cdot\Ind_{\cE_{J^c}} & \le \exp\left(-\frac{2(\delta (n-k)k q)^2}{(n-k)k^2}\right) \cdot\Ind_{\cE_{J^c}} \\
& = \exp\left(-2\delta^2(n-k)q^2\right)\cdot\Ind_{\cE_{J^c}}.
\$
Therefore, taking expectations over $\{x_i, y_i\}_{i\in J^c}$ on both sides, we have
\$ 
\prob{e(J, J^c) \le (1-\delta) (n-k)k q, \; \cE_{J^c}} \le \exp\left(-2\delta^2(n-k)q^2\right).
\$
Let $\cE = \{p_{ii} \ge q, \forall i \in [n]\}$. Then for every fixed $J$, we have $\cE \subset \cE_{J^c}$; hence 
\$ 
\prob{e(J, J^c) \le (1-\delta) (n-k)k q, \; \cE} \le \exp\left(-2\delta^2(n-k)q^2\right).
\$
By \Cref{lmm:p-q-sep}, $\prob{\cE}\ge \prob{\cE_{\mathrm{sep}}} \ge 1-n^{-\Omega(1)}$.
Since there are at most $n^k$ such $J$ with $|J|=k$, taking a union bound over all $J$ with $|J^c|\le n/\log^2 n$, we have
\$ 
& \prob{\exists\, J \text{ with } |J^c| \le n/\log^2 n \colon e(J, J^c) \le (1-\delta) q|J||J^c| } \\
& \quad \le \prob{\cE^c} + \sum_{k=1}^{n/\log^2 n}n^k \exp\left(-2\delta^2(n-k)q^2\right) \\
& \quad\le n^{-\Omega(1)} + \frac{n}{\log^2 n}\exp\left(-\delta^2nq^2 + n/\log n\right) = n^{-\Omega(1)},
\$
where the last equality holds when $\delta^2\ge C/\log n$ for a sufficiently large $C>0$ (or equivalently, $\rho^2 \ge \sqrt{C /\log n}$ since $\delta = \Theta(\rho^2)$). 
\end{proof}

\section{Quadratic Assignment Estimator} \label{sect:qap} 
In this section, we prove the information-theoretic condition in~\Cref{thm:IT}, by analyzing the  quadratic assignment problem
(QAP) estimator~\eqref{eq:def-qap}. Recall that $A = XX^\top$ and $B = YY^\top$.

\begin{proof}[Proof of~\Cref{thm:IT}]
Without loss of generality, we assume the true permutation is identity, that is, $\Pi=I$. Specifically, we prove the result under the condition
\#\label{eq:qap-cond} 
d\ge C\log n, \quad \frac{1-\rho^2}{\rho^2} < \frac{cnd}{(n + d) \log n}, \quad \frac{(1-\rho^2)^2}{\rho^4} < \frac{cn}{\log n},
\#
for some absolute constants $C,c>0$. The conditions in \eqref{eq:qap-cond} imply those in \eqref{eq:IT_sufficient}. For a permutation matrix $P\in \fS_n$, we define
\[
  D_P \triangleq A - P A P^\top,
  \quad
  \Delta_P \triangleq \iprod{D_P}{B}.
\]
We now show that, under condition~\eqref{eq:qap-cond}, we have $\Delta_P>0$ for every $P\neq I$ with high probability. Hence, the QAP estimator exactly recovers the true permutation. By definition, 
\$B = YY^\top =  \rho^2 X X^\top + \rho \sqrt{1-\rho^2}(X Z^\top + Z X^\top) + (1-\rho^2) Z Z^\top.\$
We decompose $\Delta_P$ according to these three terms in $B$. First, since $\Fnorm{PAP^\top} = \Fnorm{A}$, we have $\Fnorm{D_P}^2 = 2\Fnorm{A}^2 - 2 \iprod{P A P^\top}{A}$. Thus, for the following \emph{signal} term,
\$ 
\iprod{D_P}{XX^\top} = \Fnorm{A}^2 -  \iprod{P A P^\top}{A} = \Fnorm{D_P}^2/2.
\$
Next, by the symmetry of $D_P$, the  following \emph{linear-noise} term can be written as 
\$
\iprod{D_P}{X Z^\top + Z X^\top} = 2 \Tr(D_P X Z^\top) = 2 \iprod{D_P X}{Z}.\$
Finally, since $\Tr(D_P) = \Tr(A) - \Tr(P A P^\top) = 0$, the following \emph{quadratic-noise} term satisfies 
\$\iprod{D_P}{Z Z^\top} = \iprod{D_P}{Z Z^\top - d I_n}.
\$
Combining them together, we have
\#\label{eq:deltap} 
\Delta_P \ge \rho^2 \fnorm{D_P}^2/2
  - 2 \rho \sqrt{1 - \rho^2} \left|\iprod{D_P X}{Z}\right|
  - (1 - \rho^2) \left|\iprod{D_P}{Z Z^\top - d I_n}\right|.
\#

\Cref{lmm:qap-signal,lmm:linear-noise,lem:quad-noise} provide a lower bound on the signal term and upper bounds on the noise terms, respectively. In particular, we work on the event 
\begin{equation}\label{eq:Eop}
  \cE \triangleq \left\{\norm{X}^2 \le C(n + d)\right\},
\end{equation}
which occurs with probability at least $1 - e^{-c(n+d)}$ by standard Gaussian matrix theory~\cite{vershynin2010nonasym}. For $P\in\fS_n$, we define $m(P) = |\{i \colon P(i) \ne i\}|$ as the number of indices at which $P$ differs from the true permutation. For $P$ with $m(P)=m$, whenever the events $\cA_P\cap\cL_P\cap\cQ_P$ in \Cref{lmm:qap-signal,lmm:linear-noise,lem:quad-noise} occur, by \eqref{eq:deltap}, we have
\$ 
\Delta_P & \ge \rho^2 \fnorm{D_P}^2/2 - \rho^2 \fnorm{D_P}^2/16 - C (1 - \rho^2)(n + d) m \log n - \rho^2 \fnorm{D_P}^2/16 - C \frac{(1-\rho^2)^2}{\rho^2} d  m \log n \\
& \ge \rho^2 \fnorm{D_P}^2/4 + c\rho^2 dmn  - C (1 - \rho^2)(n + d) m \log n - C \frac{(1-\rho^2)^2}{\rho^2} d  m \log n \\
& \ge \rho^2 \fnorm{D_P}^2/4 + c'\rho^2 dmn >0,
\$
where the last inequality follows from condition \eqref{eq:qap-cond}. Therefore,
\$ 
\prob{\cE\cap\{\Delta_P \le 0 \}} & \le \prob{\cE\cap \cA_P^c} + \prob{\cE\cap \cL_P^c}+\prob{\cE\cap \cQ_P^c} \\
& \le e^{-c' n d} + e^{-c'm nd/(n+d)} + 4n^{-2m}.
\$
Finally, taking a union bound over all $P\in \fS_n$, and noting that there are at most $n^m$ such permutations with $m$ mismatches for each $m\in\{2, \ldots, n\}$ yields that
\$ 
\prob{\exists P\neq I\colon  \Delta_P \le  0} & \le \prob{\cE^c} + \sum_{m=2}^n n^m \left(e^{-c' n d} + e^{-c'mnd/(n+d)} + 4n^{-2m} \right) \\
& \le e^{-c(n+d)} + e^{-c n d} - e^{-cnd/(n+d)}  + \frac{4}{n(n-1)} = o(1),
\$
where the last inequality holds since $d\ge C \log n$ and $nd/(n+d) \ge \min\{n,d\} \ge C \log n$ with a sufficiently large $C>0$.
\end{proof}

\subsection{Signal Lower Bound}
We begin with several preliminaries regarding the signal term. For a permutation $P\in\fS_n$, we define $M(P) = \{i \colon P(i) \ne i\}$ as its mismatch set; then $m(P) = |M(P)|$. Given a permutation $P\in \fS_n$ and a subset $S\subset [n]$, let $P(S) = \{P(i) \colon i\in S\}$. The following lemma shows that one can extract a linear-size subset $S\subset M(P)$ such that $S$ and $P(S)$ are disjoint.
\begin{lemma}\label{lem:cycle}
Let $P$ be a permutation on $[n]$ with $m(P) \ge 2$. Then there exists $S \subset M(P)$ such that
\[
S \cap P(S) = \emptyset
  \quad\text{and}\quad
  |S| \ge m(P) /3.
\]
Moreover, $P(S)\subset M(P)$ and $|S \cup P(S)| = 2|S|$.
\end{lemma}
\begin{proof}
Note that $P$ maps $M(P)$ into itself, so $P|_{M(P)}$ is a permutation of $M(P)$ with no fixed points. Hence, $P|_{M(P)}$ decomposes into cycles of at least $2$. On each cycle of length $\ell$, choose every other vertex; this gives $\lfloor \ell/2 \rfloor \ge \ell/3$ vertices, no two of which are adjacent along the cycle. Taking the union of these selected vertices over all cycles gives a set $S$ satisfying $|S| \ge m/3$, $P(S) \subset M(P),$ 
$S \cap P(S) = \emptyset$, and $|S \cup P(S)|=2|S|$, as required. 
\end{proof}

The next lemma states the concentration bounds for weighted chi-square random variables.
\begin{lemma}\label{lem:chisq}
Let $a_1,\ldots, a_m$ be nonnegative constants. For $\iid$ $\xi_i\sim N(0,1)$, let $\zeta = \sum_{i=1}^m a_i(\xi_i^2 - 1)$. Then for any $t>0$,
\#\label{eq:chi2-lm} 
\Prob\left(\zeta \le - 2 \sqrt{\sum_{i=1}^m a_i^2 t} \right) \le e^{-t}.
\#
Moreover, let $K \succeq 0$ be a deterministic $r \times r$ matrix and let $g_1, \dots, g_s \stackrel{\mathrm{i.i.d.}}{\sim} N(0, I_r)$. Then
\begin{equation}\label{eq:chisq}
  \Prob\left(\sum_{i=1}^s g_i^\top K g_i  \le  \frac{s \Tr(K)}{2} \right)
 \le  \exp \left(-c s  \frac{\Tr(K)}{\norm{K}}\right).
\end{equation}
\end{lemma}
\begin{proof}
The concentration bound \eqref{eq:chi2-lm} is the Laurent--Massart inequality \cite{laurent2000adaptive}.
Next, we diagonalize $K = U \diag(\lambda_1, \dots, \lambda_r) U^\top$ and define $h_{ij} = (U^\top g_i)_j \stackrel{\mathrm{i.i.d.}}{\sim} N(0,1)$. Then 
\$
\sum_{i=1}^s g_i^\top K g_i = \sum_{i=1}^s \sum_{j=1}^r \lambda_j h_{ij}^2,
\$
which is a sum of independent weighted $\chi^2_1$ random variables, with weight vector $\lambda$ (each weight $\lambda_j$ repeated $s$ times). Its mean equals $s \Tr(K)$. Applying \eqref{eq:chi2-lm} to the deviation level $s\Tr(K)/2$ yields a tail bound with exponent of order $s \Tr(K)^2 / \fnorm{K}^2$. Since $K \succeq 0$ yields $\fnorm{K}^2 \le \norm{K} \Tr(K)$, we may lower bound the exponent by $s\Tr(K) / \norm{K}$, which gives \eqref{eq:chisq}.
\end{proof}

We are now ready to lower bound the signal.
\begin{lemma}\label{lmm:qap-signal}
Recall $\cE$ from \eqref{eq:Eop}.  There are absolute constants $c,c' > 0$ such that if we  define the event for every $P \in \fS_n$,
\$
\cA_P \triangleq \left\{\Fnorm{D_P}^2 > cm(P)nd\right\},
\$
then it holds that
$\prob{\cE\cap\cA_P^c} \le e^{-c' n d} + e^{-c'm(P)nd/(n+d)}$.
\end{lemma}
\begin{proof}
We separately analyze the small- and large-mismatch regimes, corresponding to $m(P)\le n/2$ and $m(P) > n/2$. Fix $P$ with $m(P) = m$.

\paragraph{Small-mismatch regime: $m\le n/2$.}
Let $F = [n] \setminus M(P)$; thus $|F| = n-m \ge n/2$. By \Cref{lem:cycle}, we extract $S \subset M(P)$ with $s \triangleq |S| \ge m/3$ and $S \cap P(S) = \emptyset$. For $i \in S$ and $j \in F$, we have $P(j) = j$; hence
\[
  (D_P)_{ij} = A_{ij} - A_{P(i), j} = \iprod{X_i - X_{P(i)}}{X_j}.
\]
We define matrices $G_S \in \R^{s \times d}$ with rows $(X_i - X_{P(i)})^\top$ for $i \in S$ and $H_F \in \R^{|F| \times d}$ with rows $X_j^\top$ for $j \in F$.
Restricting the off-diagonal sum to $S \times F$, we obtain
\begin{equation}\label{eq:small-block}
  \fnorm{D_P}^2 \ge 2 \sum_{i \in S} \sum_{j \in F} (D_P)_{ij}^2
  = 2 \fnorm{G_S H_F^\top}^2.
\end{equation}
By \Cref{lem:cycle}, the $2s$ row-indices in $S \cup P(S)$ are distinct elements of $M(P)$. Hence each row of $G_S$ is the difference of two independent $N(0, I_d)$ vectors from disjoint rows of $X$; the rows of $G_S$ are \iid $N(0, 2 I_d)$. Moreover, $F$ is disjoint from $M(P)$; thus $G_S$ and $H_F$ are independent. Writing $G_S = \sqrt{2}  \widetilde{G}_S$, where $\widetilde G_S$ has \iid $N(0, 1)$ entries and rows $g_1, \dots, g_s \stackrel{\mathrm{iid}}{\sim} N(0, I_d)$, we have
\begin{equation}\label{eq:GHrewrite}
  \fnorm{G_S H_F^\top}^2  =  2 \fnorm{\widetilde G_S H_F^\top}^2  =  2 \sum_{i=1}^s g_i^\top H_F^\top H_F g_i.
\end{equation}
Let $K = H_F^\top H_F$.
Combining \eqref{eq:small-block} and \eqref{eq:GHrewrite}, we have
\begin{equation}\label{eq:DP-small}
  \fnorm{D_P}^2  \ge  4 \sum_{i=1}^s g_i^\top K g_i.
\end{equation}
Conditional on $H_F$, applying \Cref{lem:chisq} to the $g_i$'s, which are independent of $K = H_F^\top H_F$, implies,
\#\label{eq:gkg1}
\Prob\left(\sum_{i=1}^s g_i^\top K g_i \le \frac{s \Tr(K)}{2} \given H_F\right)
 \le \exp \left(-c_1 s \frac{\Tr(K)}{\norm{K}}\right).
\#
Recall $\cE$ from \eqref{eq:Eop} and let $\cE_1\triangleq\{ \|H_F\|^2 \le C(n+d)\}$.
Since $\|H_F\|^2 \le \|X\|^2$, we have $\cE \subset \cE_1$. Moreover, since $\fnorm{H_F}^2 = \sum_{j \in F} \sum_{r\in[d]} X_{jr}^2 \sim \chi^2_{|F| d}$ and $|F| \ge n/2$, the bound \eqref{eq:chi2-lm} yields for small constants $c, c_2 >0$ and $\cE_2 \triangleq \{\fnorm{H_F}^2 \ge 3cnd/2\}$, $\prob{\cE_2} \ge 1-e^{-c_2nd}$. Whenever $\cE_1\cap \cE_2$ occurs, the matrix $K = H_F^\top H_F$ satisfies $\norm{K} = \norm{H_F}^2 \le C(n+d)$ and $\Tr(K) = \fnorm{H_F}^2 \ge 3cnd/2$; hence
\#\label{eq:gkg2}
\frac{\Tr(K)}{\norm{K}}\ge \frac{3cnd}{2C(n+d)}.
\#
Combining \eqref{eq:gkg1} and \eqref{eq:gkg2} and using $s \ge m/3$, we obtain
\$
\Prob\left(\sum_{i=1}^s g_i^\top K g_i \le \frac{s \Tr(K)}{2} \given H_F\right) \Ind_{\cE_1\cap \cE_2}
&  \le \exp \left(-c_1 s \frac{\Tr(K)}{\norm{K}}\right) \Ind_{\cE_1\cap \cE_2} \le \exp\left(- \frac{c_3 m n d}{n + d}\right) \Ind_{\cE_1\cap \cE_2}.
\$
Taking the expectation over $H_F$, we have
\$ 
\Prob\left(\left\{\sum_{i=1}^s g_i^\top K g_i \le \frac{s \Tr(K)}{2} \right\} \cap \cE_1\cap \cE_2\right) &\le \exp\left(- \frac{c_3 m n d}{n + d}\right).
\$
Finally, together with \eqref{eq:DP-small} and $s \ge m/3$, the event $\cA_P^c\cap \cE_2$ with $\cA_P^c = \{\Fnorm{D_P}^2 < cmnd\}$ implies 
\$
\sum_{i=1}^s g_i^\top K g_i \le \frac{\fnorm{D_P}^2}{4} < \frac{cmnd}{4} \le \frac{s\Tr(K)}{2}.
\$
Therefore, since $\cE\subset \cE_1$, we have
\#\label{eq:small-m}
\prob{\cE\cap\cA_P^c} &\le \prob{\cE_2^c} + \prob{\cE_2\cap\cE_1\cap\cA_P^c} \notag\\
&\le \prob{\cE_2^c} + \Prob\left(\left\{\sum_{i=1}^s g_i^\top K g_i \le \frac{s \Tr(K)}{2} \right\} \cap \cE_1\cap \cE_2 \right) \notag \\
& \le e^{-c_2nd} + e^{-c_3mnd/(n+d)}.
\#
\paragraph{Large-mismatch regime: $m > n/2$.}
By \Cref{lem:cycle}, we again extract $S \subset M(P)$ with $s \triangleq |S| \ge m/3 \ge n/6$ and $S \cap P(S) = \emptyset$. For $i, j \in S$,
\[
  (D_P)_{ij}  =  \iprod{X_i}{X_j} - \iprod{X_{P(i)}}{X_{P(j)}}.
\]
Define $X_S \in \R^{s \times d}$ with rows $X_i^\top$ for $i \in S$ and $X_{P(S)}\in \R^{s \times d}$ with rows $X_{P(i)}^\top$ for $i \in S$, both with \iid $N(0,1)$ entries. Since $S \cap P(S) = \emptyset$, $X_S$ and $X_{P(S)}$ are independent.
Note that the $S \times S$ block of $D_P$ satisfies
\$
  (D_P)_{S, S}  =  X_S X_S^\top - X_{P(S)} X_{P(S)}^\top.
\$
Let $a \triangleq \lceil s/2 \rceil$ and $b \triangleq \lfloor s/2 \rfloor$. We partition the row-set of $ X_S$ into $ X_S = \begin{bmatrix}U_1\\U_2\end{bmatrix}$ with $U_1\in\R^{a\times d}$ and $U_2\in\R^{b\times d}$, and likewise for $V$. Then taking the off-diagonal sum gives
\#\label{eq:agdbopaasd}
 \fnorm{D_P}^2  \ge \fnorm{X_S X_S^\top - X_{P(S)} X_{P(S)}^\top}^2  \ge  2 \fnorm{U_1 U_2^\top - V_1 V_2^\top}^2.
\#
Let $W \triangleq [U_1, -V_1] \in \R^{a \times 2d}$, $K = W^\top W$, and $w_j \triangleq (u_{2,j}^\top, v_{2,j}^\top)^\top \in \R^{2d}$ for $j \in[b]$, where $u_{2,j}$ and $v_{2,j}$ are the rows of $U_2$ and $V_2$, respectively. Here $w_j$'s are \iid $N(0, I_{2d})$. Then we may rewrite
\#\label{eq:agdbd}
  \fnorm{U_1 U_2^\top - V_1 V_2^\top}^2  =  \sum_{j=1}^b w_j^\top K w_j.
\#
Applying \Cref{lem:chisq} conditionally on $W$, we have
\#\label{eq:wkw1}
  \Prob\left(\sum_{j=1}^b w_j^\top K w_j \le \frac{b  \Tr(K)}{2}  \given  W\right)
   \le  \exp \left(-c_4 b  \frac{\Tr(K)}{\norm{K}}\right).
\#
Recall $\cE$ from \eqref{eq:Eop} and let $\cE_3\triangleq \{\|W\|^2 \le 2C(n+d)\}$. Since $\norm{W}^2 \le \norm{U_1}^2 + \norm{V_1}^2 \le 2 \norm{X}^2$, we have $\cE \subset \cE_3$. Moreover, let $\cE_4 \triangleq \{\fnorm{W}^2 \ge 9cnd\}$. Since $\fnorm{W}^2 \sim \chi^2_{2 a d}$ and $a\ge s/2 \ge m/6\ge n/12$, the bound \eqref{eq:chi2-lm} gives $\prob{\cE_4}\ge 1- e^{-c_5 n d}$. Whenever $\cE_3\cap\cE_4$ occurs, the matrix $K = W^\top W$ satisfies $\norm{K} = \norm{W}^2 \le 2C(n+d)$ and $\Tr(K) = \fnorm{W}^2 \ge 9cnd$; hence 
\#\label{eq:wkw2}\frac{\Tr(K)}{\norm{K}} \ge \frac{9cnd}{2C(n+d)}.\# 
Combining \eqref{eq:wkw1} and \eqref{eq:wkw2} and using $b \ge s/3 \ge m/9$, we obtain
\$
\Prob\left(\sum_{j=1}^b w_j^\top K w_j \le \frac{b  \Tr(K)}{2}  \given  W\right) \Ind_{\cE_3\cap \cE_4}
&  \le \exp \left(-c_4 b  \frac{\Tr(K)}{\norm{K}}\right) \Ind_{\cE_3\cap \cE_4} \le \exp\left(- \frac{c_6 m n d}{n + d}\right) \Ind_{\cE_3\cap \cE_4}.
\$
Taking the expectation over $W$, we have
\$ 
\Prob\left(\left\{\sum_{j=1}^b w_j^\top K w_j \le \frac{b  \Tr(K)}{2} \right\} \cap \cE_3\cap \cE_4\right) &\le \exp\left(- \frac{c_6 m n d}{n + d}\right).
\$
Finally, with \eqref{eq:agdbopaasd}, \eqref{eq:agdbd}, and $b \ge m/9$, the event $\cA_P^c\cap \cE_4$ where $\cA_P^c = \{\Fnorm{D_P}^2 < cmnd\}$ gives 
\$
\sum_{j=1}^b w_j^\top K w_j \le \frac{\fnorm{D_P}^2}{2} < \frac{cmnd}{2} \le \frac{b\Tr(K)}{2}.
\$
Therefore, since $\cE\subset \cE_3$, we have
\#\label{eq:large-m}
\prob{\cE\cap\cA_P^c} &\le \prob{\cE_4^c} + \prob{\cE_4\cap\cE_3\cap\cA_P^c} \notag\\
&\le \prob{\cE_4^c} + \Prob\left(\left\{\sum_{j=1}^b w_j^\top K w_j \le \frac{b  \Tr(K)}{2} \right\} \cap \cE_3\cap \cE_4 \right) \notag \\
& \le e^{-c_5nd} + e^{-c_6mnd/(n+d)}.
\#
Combining \eqref{eq:small-m} and \eqref{eq:large-m} concludes the proof of the lemma.
\end{proof}

\subsection{Noise Upper Bound}
The lemmas in this section provide upper bounds on the two noise terms, respectively.
\begin{lemma}\label{lmm:linear-noise}
There exists a constant $C > 0$ such that, if for every $P \in \fS_n$ we define the event
\$ 
\cL_P \triangleq \left\{2 \rho \sqrt{1-\rho^2}  \left|\iprod{D_P X}{Z}\right|
    \le  \rho^2 \fnorm{D_P}^2/16 + C (1 - \rho^2)(n + d) m(P) \log n\right\},
\$ then recalling $\cE$ from \eqref{eq:Eop}, it holds that 
$\prob{\cE\cap\cL_P^c} \le 2 n^{-2m(P)}$.
\end{lemma}
\begin{proof}
Fix $P\in \cS_n$ with $m(P) = m$. Conditional on $X$, since $Z$ has \iid $N(0,1)$ entries, it holds that $\iprod{D_P X}{Z} \sim N(0, \Fnorm{D_P X}^2)$. Hence by Chernoff bound, for any $t>0$,
\$ 
\prob{|\iprod{D_P X}{Z}| \ge t \given X} \le 2 \exp\left(-t^2 / \left(2 \fnorm{D_P X}^2\right)\right).
\$
Let $\cR_P\triangleq \left\{|\iprod{D_P X}{Z}| \le 2\sqrt{m \log n} \fnorm{D_P X}\right\}$.
With $t^2 = 4 m \log n \cdot\fnorm{D_P X}^2 $, it follows that 
\$ 
\prob{\cR_P^c \given X} \le 2n^{-2m}.
\$
When $\cE = \{\norm{X}^2 \le C(n + d)\}$ and $\cR_P$ both occur, we have
\$
  \fnorm{D_P X} \le \fnorm{D_P} \norm{X} \le C \fnorm{D_P} \sqrt{n + d},
\$
and
\$ 
2 \rho \sqrt{1-\rho^2} |\iprod{D_P X}{Z}| &\le C \rho \sqrt{1-\rho^2}  \sqrt{(n + d)m\log n} \fnorm{D_P}\\
& \le \rho^2 \fnorm{D_P}^2/16 + 4C^2 (1 - \rho^2)(n + d)  m \log n,
\$
where the last line applies the Cauchy--Schwarz inequality $2 a b \le a^2/16 + 16 b^2$ with $a = \rho \fnorm{D_P}$ and $b = C \sqrt{1-\rho^2}  \sqrt{(n+d) m \log n}/2$. Thus, $\cE \cap \cR_P \subset \cE \cap \cL_P$; hence $\cE \cap \cL_P^c \subset \cE \cap \cR_P^c$. Therefore,
\$
\prob{\cE \cap \cL_P^c} \le \prob{\cE \cap \cR_P^c} = \E[\Indc_\cE \prob{\cR_P^c \given X}]\le 2n^{-2m}.
\$
\end{proof}

\begin{lemma}\label{lem:quad-noise}
There exists a constant $C > 0$ such that, if for every $P \in \fS_n$ we define the event
\$
\cQ_P \triangleq \bigg\{(1-\rho^2) \left|\iprod{D_P}{Z Z^\top - d I_n}\right|
    & \le  \rho^2 \fnorm{D_P}^2/16 + C \frac{(1-\rho^2)^2}{\rho^2} d  m(P) \log n \\
    &\qquad\qquad \qquad + C (1-\rho^2) (n+d) m(P) \log n\bigg\},
\$
then recalling $\cE$ from \eqref{eq:Eop}, it holds that 
$\prob{\cE\cap\cQ_P^c} \le 2 n^{-2m(P)}$.
\end{lemma}

\begin{proof}
Fix $P\in \cS_n$ with $m(P) = m$. Let $z_1, \ldots, z_d$ denote the columns of $Z$, where $z_\ell \sim N(0, I_n)$.
Note that
\$ 
\iprod{D_P}{Z Z^\top} = \sum_{\ell=1}^d z_\ell^\top D_P z_\ell = \vecc(Z)^\top (I_d\otimes D_P)\vecc(Z),
\$
with mean $\E\iprod{D_P}{Z Z^\top} = \iprod{D_P}{dI_n}$. Thus, by Hanson--Wright inequality, for any $u > 0$,
\#\label{eq:qpa} 
\prob{\left|\iprod{D_P}{Z Z^\top - dI_n} \right| \ge C\left(\fnorm{I_d\otimes D_P}\sqrt{u} + \norm{I_d\otimes D_P} u\right) \given X} \le 2e^{-u},
\#
where $\fnorm{I_d\otimes D_P} = \sqrt{d}\fnorm{D_P}$ and $\norm{I_d\otimes D_P} = \norm{D_P}$. We define the event
\$
\cB_P \triangleq \left\{\left|\iprod{D_P}{Z Z^\top - d I_n}\right|
    \le  C\left(\fnorm{D_P} \sqrt{d m \log n} + \norm{D_P}  m \log n\right)\right\}.
\$
Taking $u = 2 m \log n$ in \eqref{eq:qpa} yields that $\prob{\cB_P^c \given X} \le 2 n^{-2m}$.
Moreover, when $\cE$ occurs, since $\|PAP^\top\| = \|A\|$, we have
\$ 
\norm{D_P} \le \norm{A} + \|PAP^\top\| = 2\|A\| = 2\|X\|^2 \le 2C(n+d).
\$
Therefore, when $\cE\cap \cB_P$ both occur, we obtain
\$
(1-\rho^2)\left|\iprod{D_P}{Z Z^\top - d I_n}\right|
    & \le  C(1-\rho^2)\fnorm{D_P} \sqrt{d m \log n} + C(1-\rho^2)\norm{D_P}  m \log n \\
    & \le \rho^2 \fnorm{D_P}^2/16 + 16C^2 \frac{(1-\rho^2)^2}{\rho^2} d  m \log n + 2C (1-\rho^2) (n+d) m \log n,
\$
where the last inequality uses $2 a b \le  a^2/16 + 16 b^2$ with $a = \rho \fnorm{D_P}$ and $b = C(1-\rho^2)\sqrt{d m \log n}/\rho$.
In other words, we have $\cE \cap \cB_P \subset \cE\cap \cQ_P$. Therefore,
\$
\prob{\cE \cap \cQ_P^c} \le \prob{\cE \cap \cB_P^c} = \E[\Indc_\cE \prob{\cB_P^c \given X}]\le 2n^{-2m}.
\$
\end{proof}

\input{low_degree_analysis.tex}

%% file: fig/tree-t.tex
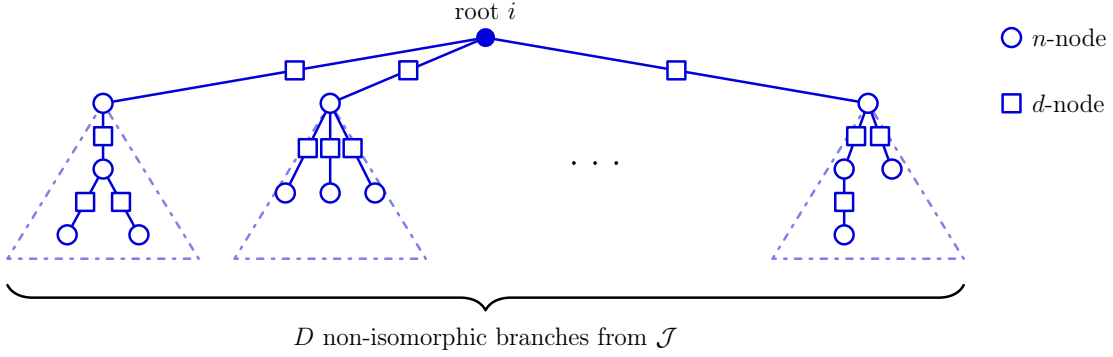
\begin{figure}[htbp]
\centering
\resizebox{0.9\textwidth}{!}{
\begin{tikzpicture}[
    line cap=round, line join=round,
    n_root/.style={circle, fill=Tcolor, minimum size=9pt, inner sep=0pt},
    n_circ/.style={circle, draw=Tcolor, very thick, fill=white, minimum size=9pt, inner sep=0pt},
    n_sq/.style={rectangle, draw=Tcolor, very thick, fill=white, minimum size=8.5pt, inner sep=0pt}
]

\colorlet{Tcolor}{blue!85!black}

\begin{scope}[shift={(8.8, 0)}]
    \node[n_circ] (leg_circ) at (0, 0) {};
    \node[right=6pt, font=\large] at (leg_circ) {$n$-node};
    \node[n_sq] (leg_sq) at (0, -1.1) {};
    \node[right=6pt, font=\large] at (leg_sq) {$d$-node};
\end{scope}

\coordinate (Apex_Middle)   at (-6.4, -1.1);
\coordinate (Apex_Right) at (-2.6, -1.1);
\coordinate (Apex_Left)  at ( 6.4, -1.1);

\draw[very thick, dash pattern=on 2pt off 4pt, loosely dashdotted, Tcolor!50] (Apex_Left) -- ++(-1.6, -2.6) -- ++(3.2, 0) -- cycle;
\draw[very thick, dash pattern=on 2pt off 4pt, loosely dashdotted, Tcolor!50] (Apex_Middle) -- ++(-1.6, -2.6) -- ++(3.2, 0) -- cycle;
\draw[very thick, dash pattern=on 2pt off 4pt, loosely dashdotted, Tcolor!50] (Apex_Right) -- ++(-1.6, -2.6) -- ++(3.2, 0) -- cycle;

\node[n_root] (root) at (0,0) {};
\node[above=5pt, font=\large] at (root) {root $i$};

\node[n_circ] (l_c1) at (Apex_Left) {};
\node[n_circ] (l_c2l) at (6, -2.2) {};
\node[n_circ] (l_c3l) at (6, -3.3) {};
\node[n_circ] (l_c2r) at (6.8, -2.2) {};

\draw[very thick, Tcolor] (root) -- node[n_sq] {} (l_c1);
\draw[very thick, Tcolor] (l_c1) -- node[n_sq] {} (l_c2l);
\draw[very thick, Tcolor] (l_c2l) -- node[n_sq] {} (l_c3l);
\draw[very thick, Tcolor] (l_c1) -- node[n_sq] {} (l_c2r);

\node[n_circ] (m_c1) at (Apex_Middle) {};
\node[n_circ] (m_c2) at (-6.4, -2.2) {};
\node[n_circ] (m_c3l) at (-7, -3.3) {};
\node[n_circ] (m_c3r) at (-5.8, -3.3) {};

\draw[very thick, Tcolor] (root) -- node[n_sq] {} (m_c1);
\draw[very thick, Tcolor] (m_c1) -- node[n_sq] {} (m_c2);
\draw[very thick, Tcolor] (m_c2) -- node[n_sq] {} (m_c3l);
\draw[very thick, Tcolor] (m_c2) -- node[n_sq] {} (m_c3r);

\node[n_circ] (r_c1) at (Apex_Right) {};
\node[n_circ] (r_c2l) at (-3.35, -2.6) {};
\node[n_circ] (r_c2m) at (-2.6, -2.6) {};
\node[n_circ] (r_c2r) at (-1.85, -2.6) {};

\draw[very thick, Tcolor] (root) -- node[n_sq] {} (r_c1);
\draw[very thick, Tcolor] (r_c1) -- node[n_sq] {} (r_c2l);
\draw[very thick, Tcolor] (r_c1) -- node[n_sq] {} (r_c2m);
\draw[very thick, Tcolor] (r_c1) -- node[n_sq] {} (r_c2r);

\node[font=\Huge] at (1.9, -2.1) {\dots};


\draw[decorate, decoration={brace, amplitude=12pt, mirror, raise=4pt}, very thick] 
    (-8, -4) -- (8, -4) 
    node[midway, below=20pt, font=\large] {$D$ non-isomorphic branches from $\cJ$};

\end{tikzpicture}
}
\vspace{0.4em}
\caption{Schematic illustration of a tree in $\cT$. Each dashed triangle represents the descendant subtree of a branch. }
\label{fig:tree-plot}
\end{figure}

%% file: fig/mean-sep.tex
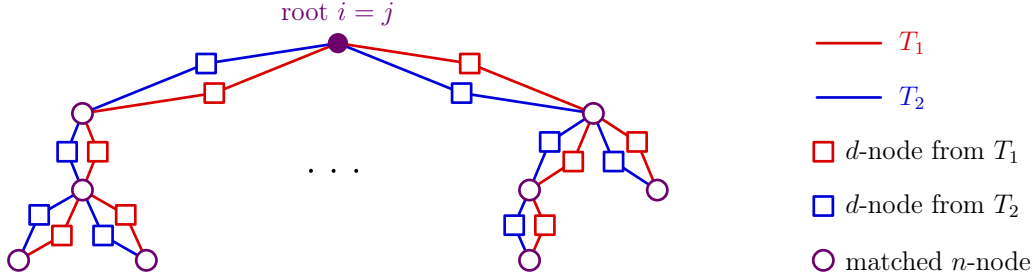
\begin{figure}[htbp]
\centering
\colorlet{Tcolor}{red!85!black}
\colorlet{Scolor}{blue!85!black}
\colorlet{TScolor}{Tcolor!50!Scolor} 

\resizebox{0.84\textwidth}{!}{
\begin{tikzpicture}[
    line cap=round, line join=round,
    n_root/.style={circle, fill=TScolor, minimum size=9pt, inner sep=0pt},
    n_circ/.style={circle, draw=TScolor, very thick, fill=white, minimum size=9pt, inner sep=0pt},
    n_sq/.style={rectangle, draw=black, very thick, fill=white, minimum size=8.5pt, inner sep=0pt}
]

\newcommand{\overlapbichro}[4]{
    \coordinate (M) at ($(#1)!0.5!(#2)$);
    \coordinate (D1) at ($(M)!0.24cm!90:(#2)$); 
    \coordinate (D2) at ($(M)!0.24cm!-90:(#2)$); 
    
    \draw[very thick, #3] (#1) -- (D1) -- (#2);
    \draw[very thick, #4] (#1) -- (D2) -- (#2);
    
    \node[n_sq, draw=#3] at (D1) {};
    \node[n_sq, draw=#4] at (D2) {};
}

\begin{scope}[shift={(7.5, 0)}]
        \draw[very thick, Tcolor, font=\large] (0, 0) -- (1, 0) node[right=4pt] {$T_1$};
        \draw[very thick, Scolor, font=\large] (0, -0.85) -- (1, -0.85) node[right=4pt] {$T_2$};

    \node[n_sq, draw = Tcolor] (leg_sq) at (0.1, -1.7) {};
    \node[right=6pt, font=\large] at (leg_sq) {$d$-node from $T_1$};
    \node[n_sq, draw = Scolor] (leg_sq) at (0.1, -2.55) {};
    \node[right=6pt, font=\large] at (leg_sq) {$d$-node from $T_2$};
    \node[n_circ] (leg_circ) at (0.1, -3.4) {};
    \node[right=6pt, font=\large, align=right] at (leg_circ) {matched $n$-node};
    \end{scope}

\begin{scope}[shift={(0, 0)}]

\coordinate (root) at (0,0);

\coordinate (r1)   at (4, -1.1);
\coordinate (r2)  at ( -4, -1.1);


\coordinate (r11) at (3, -2.3);
\coordinate (r12) at (5, -2.3);

\coordinate (r13) at (3, -3.4);

\coordinate (r21) at (-4, -2.3);
\coordinate (r22) at (-5, -3.4);
\coordinate (r23) at (-3, -3.4);

%


\overlapbichro{root}{r1}{Tcolor}{Scolor}
\overlapbichro{r1}{r11}{Tcolor}{Scolor}
\overlapbichro{r1}{r12}{Tcolor}{Scolor}
\overlapbichro{r11}{r13}{Tcolor}{Scolor}

\overlapbichro{root}{r2}{Tcolor}{Scolor}
\overlapbichro{r2}{r21}{Tcolor}{Scolor}
\overlapbichro{r21}{r22}{Tcolor}{Scolor}
\overlapbichro{r21}{r23}{Tcolor}{Scolor}



\node[n_root] at (root) {};
\node[above=4pt, TScolor, font=\large] at (root) {root $i=j$};

\node[n_circ] at (r1) {};
\node[n_circ] at (r2) {};
\node[n_circ] at (r11) {};
\node[n_circ] at (r12) {};
\node[n_circ] at (r13) {};
\node[n_circ] at (r21) {};
\node[n_circ] at (r22) {};
\node[n_circ] at (r23) {};

\node[font=\Huge] at (0, -2) {\dots};
\end{scope}

\end{tikzpicture}
}
\vspace{0.3em}
\caption{Dominant contribution to the true-pair mean. Every $n$-node in $T_1$ has the same label as its counterpart in $T_2$; these matched $n$-nodes are drawn in purple. Moreover, every edge pair in $T_1$ is matched with the corresponding edge pair in $T_2$. The $d$-nodes from $T_1$ and $T_2$ are drawn separately to indicate that they need not have the same labels, while still contributing positively after averaging over $Q$.}
\label{fig:mean-sep}
\end{figure}

%% file: fig/var-2.tex
\begin{figure}[t]
\centering
\colorlet{Tcolor}{red!85!black}
\colorlet{Scolor}{blue!85!black}
\colorlet{TScolor}{Tcolor!50!Scolor}
\vspace{0.3em}
\begin{subfigure}[t]{0.9\textwidth}
\centering
\resizebox{\linewidth}{!}{
\begin{tikzpicture}[
    line cap=round, line join=round,
    n_root/.style={circle, minimum size=9pt, inner sep=0pt},
    n_circ/.style={circle, draw=TScolor, very thick, fill=white, minimum size=9pt, inner sep=0pt},
    n_sq/.style={rectangle, very thick, fill=white, minimum size=8.5pt, inner sep=0pt}
]
\newcommand{\overlapbichro}[4]{
    \coordinate (M) at ($(#1)!0.5!(#2)$);
    \coordinate (D1) at ($(M)!0.24cm!90:(#2)$); 
    \coordinate (D2) at ($(M)!0.24cm!-90:(#2)$); 
    
    \draw[very thick, #3] (#1) -- (D1) -- (#2);
    \draw[very thick, #4] (#1) -- (D2) -- (#2);
    
    \node[n_sq, draw=#3, solid] at (D1) {};
    \node[n_sq, draw=#4, solid] at (D2) {};
}

\begin{scope}[shift={(9, 0)}]
        \draw[very thick, Tcolor] (1, 0) -- (2, 0) node[right=4pt] {\Large $T_1$};
        \draw[very thick, dashed, Tcolor] (1, -1) -- (2, -1) node[right=4pt] {\Large $S_1$};
        \draw[very thick, Scolor] (1, -2) -- (2, -2) node[right=4pt] {\Large $T_2$};
        \draw[very thick, dashed, Scolor] (1, -3) -- (2, -3) node[right=4pt] {\Large $S_2$};
\end{scope}
\begin{scope}[shift={(0, 0)}]    
    \coordinate (rootI) at (-2.4,  0); \coordinate (rootJ) at ( 2.4,  0);
    \coordinate (S1)    at (-5.5, -1.2); \coordinate (S2)    at ( -4, -1.2);
    \coordinate (S3)    at (4, -1.2); \coordinate (S4)    at (5.5, -1.2);

    \coordinate (C1) at (-7.0, -2.4); \coordinate (C2) at (-2.5, -2.4); 
    \coordinate (C3) at ( 2.5, -2.4); \coordinate (C4) at ( 7.0, -2.4); 

    \coordinate (R1_1)  at (-7.0, -3.6);
    \coordinate (R1_2)  at (-5.9, -4.8);
    \coordinate (R1_3)  at (-8.1, -4.8);

    \coordinate (R2_1)  at (-4.2, -3.8);
    \coordinate (R2_2)  at (-2.5, -3.8);
    \coordinate (R2_3)  at (-0.8, -3.8);

    \coordinate (R3_1)  at (2.5, -3.6);
    \coordinate (R3_2)  at (2.5, -4.8);
    \coordinate (R3_3)  at (2.5, -6);

    \coordinate (R4_1)  at ( 5.9, -3.6);
    \coordinate (R4_2)  at ( 8.1, -3.6);
    \coordinate (R4_3)  at ( 5.9, -4.8);

    \draw[very thick, Tcolor] (rootI) to (S1); 
    \draw[very thick, Tcolor, dashed, bend right=25] (rootI) to (S1); 
    
    \draw[very thick, Tcolor] (rootI) to (S3); 
    \draw[very thick, Tcolor, dashed, bend right=15] (rootI) to (S3); 

    \draw[very thick, Scolor] (rootJ) to (S2); 
    \draw[very thick, Scolor, dashed, bend left=15] (rootJ) to (S2); 
    
    \draw[very thick, Scolor] (rootJ) to (S4); 
    \draw[very thick, Scolor, dashed, bend left=25] (rootJ) to (S4); 

    \draw[very thick, Tcolor] (S1) -- (C1);
    \draw[very thick, Tcolor] (S3) -- (C4);
    \draw[very thick, Tcolor, dashed] (S1) -- (C2);
    \draw[very thick, Tcolor, dashed] (S3) -- (C3);
    \draw[very thick, Scolor] (S2) -- (C1);
    \draw[very thick, Scolor] (S4) -- (C4);
    \draw[very thick, Scolor, dashed] (S2) -- (C2);
    \draw[very thick, Scolor, dashed] (S4) -- (C3);

    \overlapbichro{C1}{R1_1}{Tcolor}{Scolor}    
    \overlapbichro{R1_1}{R1_2}{Tcolor}{Scolor}   
    \overlapbichro{R1_1}{R1_3}{Tcolor}{Scolor} 

    \overlapbichro{C2}{R2_1}{Tcolor, dashed}{Scolor, dashed} 
    \overlapbichro{C2}{R2_2}{Tcolor, dashed}{Scolor, dashed} 
    \overlapbichro{C2}{R2_3}{Tcolor, dashed}{Scolor, dashed} 

    \overlapbichro{C3}{R3_1}{Tcolor, dashed}{Scolor, dashed} 
    \overlapbichro{R3_1}{R3_2}{Tcolor, dashed}{Scolor, dashed} 
    \overlapbichro{R3_2}{R3_3}{Tcolor, dashed}{Scolor, dashed} 

    \overlapbichro{C4}{R4_1}{Tcolor}{Scolor} 
    \overlapbichro{C4}{R4_2}{Tcolor}{Scolor} 
    \overlapbichro{R4_1}{R4_3}{Tcolor}{Scolor}

    \node[n_root, fill=Tcolor] at (rootI) {}; \node[above=6pt, font=\Large, Tcolor] at (rootI) {root $i$};
    \node[n_root, fill=Scolor] at (rootJ) {}; \node[above=6pt, font=\Large, Scolor] at (rootJ) {root $j$};

    \node[n_circ] at (C1) {}; \node[left=6pt, Tcolor!50!Scolor, font = \Large] at (C1) {$u$};
    \node[n_circ] at (C2) {}; \node[right=6pt, Tcolor!50!Scolor, font = \Large] at (C2) {$v$}; \node[n_circ] at (C3) {}; \node[n_circ] at (C4) {};
    \node[n_circ] at (R1_1) {}; \node[n_circ] at (R1_2) {}; \node[n_circ] at (R1_3) {};
    \node[n_circ] at (R2_1) {}; \node[n_circ] at (R2_2) {}; \node[n_circ] at (R2_3) {};
    \node[n_circ] at (R3_1) {}; \node[n_circ] at (R3_2) {}; \node[n_circ] at (R3_3) {};
    \node[n_circ] at (R4_1) {}; \node[n_circ] at (R4_2) {}; \node[n_circ] at (R4_3) {};

    \node[n_sq, draw = Tcolor] at (S1) {}; \node[left=6pt, Tcolor, font = \Large] at (S1) {$a$};
    \node[n_sq, draw = Tcolor] at (S3) {};
    \node[n_sq, draw = Scolor] at (S2) {};
    \node[left=8pt, Scolor, font = \Large] at (S2) {$b$};
    \node[n_sq, draw = Scolor] at (S4) {};

    \node[font=\Huge] at (0, -2.4) {\dots};
\end{scope}
\end{tikzpicture}
}
\caption{Consider the four branches on the left as an example. The first $d$-node in the branch of $T_1$ coincides with that in the branch of $S_1$ at $a$, and similarly the first $d$-node in the branch of $T_2$ coincides with that in the branch of $S_2$ at $b$. Once these parallel edge pairs are formed, edge pairs in the descendant subtree of this branch in $T_1$ are matched with those in the isomorphic branch of $T_2$, and similarly for $S_1$ and $S_2$.}
\label{fig:var-top}
    \end{subfigure}

    \vspace{1em}
\begin{subfigure}[t]{0.86\textwidth}
\centering
\resizebox{\linewidth}{!}{
\begin{tikzpicture}[
    line cap=round, line join=round,
    n_root/.style={circle, minimum size=9pt, inner sep=0pt},
    n_circ/.style={circle, draw=TScolor, very thick, fill=white, minimum size=9pt, inner sep=0pt},
    n_sq/.style={rectangle, very thick, fill=white, minimum size=8.5pt, inner sep=0pt}
]
\newcommand{\drawvartree}[2]{
\coordinate (root) at (0,0);

\coordinate (s1) at (-2, -0.6);
\coordinate (r1)   at (-4, -1.2);
\coordinate (r2)  at ( 4, -1.2);
\coordinate (s2) at (2, -0.6);

\coordinate (s11) at (-4, -1.8);
\coordinate (r11) at (-4, -2.4);
\coordinate (s12) at (-4.5, -3);
\coordinate (r12) at (-5, -3.6);
\coordinate (s13) at (-3.5, -3);
\coordinate (r13) at (-3, -3.6);

\coordinate (s21) at (3.5, -1.8);
\coordinate (r21) at (3, -2.4);
\coordinate (s22) at (4.5, -1.8);
\coordinate (r22) at (5, -2.4);
\coordinate (s23) at (3, -3);
\coordinate (r23) at (3, -3.6);

\draw[very thick, #1] (root) -- (s1) -- (r1) -- (s11) -- (r11) -- (s12) -- (r12);
\draw[very thick, #1] (r11) -- (s13) -- (r13);
\draw[very thick, #1] (root) -- (s2) -- (r2) -- (s21) -- (r21) -- (s23) -- (r23);
\draw[very thick, #1] (r2) -- (s22) -- (r22);

\draw[very thick, #1, dashed, bend right=25] (root) to (s1) to (r1);
\draw[very thick, #1, dashed, bend right=50](r1) to (s11) to (r11) to (s12) to (r12);
\draw[very thick, #1, dashed, bend left=50] (r11) to (s13) to (r13);
\draw[very thick, #1, dashed, bend left=25] (root) to (s2) to (r2);
\draw[very thick, #1, dashed, bend right=50] (r2) to (s21) to (r21) to (s23) to (r23);
\draw[very thick, #1, dashed, bend left=50] (r2) to (s22) to (r22);

\node[n_root, fill=#1] at (root) {}; \node[above=6pt, font=\Large, #1] at (root) {#2};
\node[n_circ, draw = #1] at (r1) {};
\node[n_circ, draw = #1] at (r2) {};
\node[n_circ, draw = #1] at (r11) {};
\node[n_circ, draw = #1] at (r12) {};
\node[n_circ, draw = #1] at (r13) {};
\node[n_circ, draw = #1] at (r21) {};
\node[n_circ, draw = #1] at (r22) {};
\node[n_circ, draw = #1] at (r23) {};

\node[n_sq, draw = #1] at (s1) {};
\node[n_sq, draw = #1] at (s2) {};
\node[n_sq, draw = #1] at (s11) {};
\node[n_sq, draw = #1] at (s12) {};
\node[n_sq, draw = #1] at (s13) {};
\node[n_sq, draw = #1] at (s21) {};
\node[n_sq, draw = #1] at (s22) {};
\node[n_sq, draw = #1] at (s23) {};

\node[font=\Huge] at (0, -1.8) {\dots};
}

\begin{scope}[shift={(0,0)}]
\drawvartree{Tcolor}{root $i$}
\end{scope}

\begin{scope}[shift={(11.5,0)}]
\drawvartree{Scolor}{root $j$}
\end{scope}
\end{tikzpicture}
}
\caption{Here $T_1=S_1$ and $T_2=S_2$ as labeled trees, and therefore $T\cong S$. 
}
\label{fig:var-bottom}
    \end{subfigure}
\vspace{0.4em}
\caption{ Dominant configurations for the false-pair variance.}
\label{fig:variance-illustrate}
\end{figure}
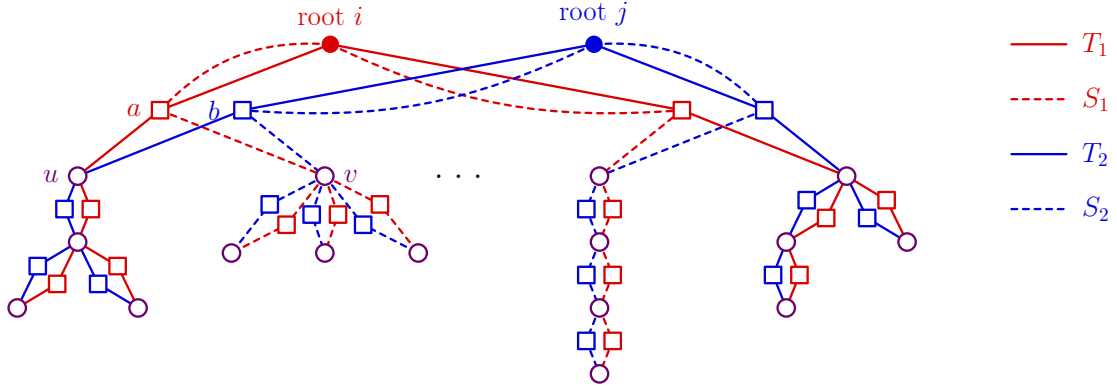
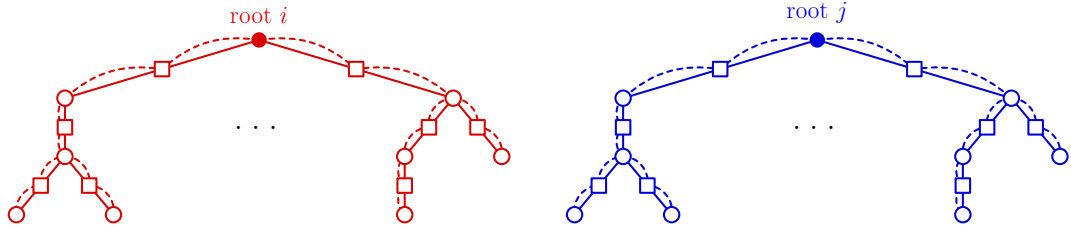

%% file: fig/weing-example.tex
\begin{figure}[htbp]
    \centering 
\resizebox{0.32\textwidth}{!}{
\begin{tikzpicture}[
    rednode/.style={draw=red!85!black, fill=white, very thick, minimum size=12pt, inner sep=0pt},
    bluenode/.style={draw=blue!85!black, fill=white, very thick, minimum size=12pt, inner sep=0pt},
    node distance=3cm and 4cm
]


    \node[rednode, label=left:{\color{red!85!black}$1\quad$}] (n1) at (0,0) {};
    \node[bluenode, label=right:{\color{blue!85!black}$\quad 1$}] (n1p) at (4,0) {};
    \node[bluenode, label=right:{\color{blue!85!black}$\quad 2$}] (n2p) at (4,-2) {};

\draw[very thick] (n1) to[bend left=18] (n1p);
    \node at (2, 0.57) {$e_1$}; 
    
    \draw[very thick] (n1) to[bend right=18] (n1p);
    \node at (2, -0.2) {$e_2$};

    \draw[very thick] (n1) to[bend left=5] (n2p);
    \node at (2, -1.1) {$e_3$};
    
    \draw[very thick] (n1) to[bend right=30] (n2p);
    \node at (2, -2.) {$e_4$};
\end{tikzpicture}
}
\caption{Here $E(J)= \{e_1, e_2, e_3, e_4\}$ and $Q^J = Q_{11}^2Q_{12}^2$. The graph $J$ admits three circuit decompositions: one consisting of two circuits of length $2$, $\{(e_1,e_2), (e_3,e_4)\}$, and two consisting of a single circuit of length $4$, $\{(e_1,e_2,e_3,e_4)\}$ and $\{(e_1,e_2,e_4,e_3)\}$.}
    \label{fig:weing-example}
\end{figure}
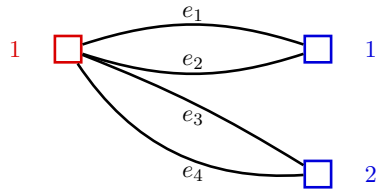

%% file: fig/fully-overlap.tex
\begin{figure}[htbp]
\centering
\resizebox{\textwidth}{!}{
\begin{tikzpicture}[
    line cap=round, line join=round, 
    n_solid/.style={circle, fill, minimum size=9pt, inner sep=0pt},
    n_hollow/.style={circle, draw, fill=white, very thick, minimum size=9pt, inner sep=0pt},
    n_sq/.style={rectangle, draw, fill=white, very thick, minimum size=8.5pt, inner sep=0pt}
]

\colorlet{T1color}{red!85!black}
\colorlet{S1color}{red!85!black}
\colorlet{T2color}{blue!85!black}
\colorlet{S2color}{blue!85!black}

\begin{scope}[shift={(7.5, 1.5)}]
        \draw[very thick, T1color] (-4.3, 0) -- (-3.3, 0) node[right=4pt] {\Large $T_1$};
        \draw[very thick, dashed, S1color] (-1.8, 0) -- (-0.8, 0) node[right=4pt] {\Large $S_1$};
        \draw[very thick, T2color] (0.7, 0) -- (1.7, 0) node[right=4pt] {\Large $T_2$};
        \draw[very thick, dashed, S2color] (3.2, 0) -- (4.2, 0) node[right=4pt] {\Large $S_2$};
\end{scope}

\begin{scope}[shift={(0,0)}]
    \coordinate (N0) at (0, 0);
    \coordinate (N1) at (0, -0.7);
    \coordinate (N2) at (0, -1.4);
    
    \coordinate (L1) at (-0.8, -2.1);
    \coordinate (L2) at (-1.4, -2.8);
    
    \coordinate (M1) at (0, -2.1);
    \coordinate (M2) at (0, -2.8);
    
    \coordinate (R1) at (0.8, -2.1);
    \coordinate (R2) at (1.4, -2.8);

    \draw[very thick, T1color] (N0) -- (N1) -- (N2);
    \draw[very thick, T1color] (N2) -- (L1) -- (L2);
    \draw[very thick, T1color] (N2) -- (M1) -- (M2);
    \draw[very thick, T1color] (N2) -- (R1) -- (R2);

    \draw[very thick, dashed, S1color] (N0) to[bend right=50] (N1);
    \draw[very thick, dashed, S1color] (N1) to[bend right=50] (N2);
    \draw[very thick, dashed, S1color] (N2) to[bend right=50] (L1);
    \draw[very thick, dashed, S1color] (L1) to[bend right=50] (L2);
    \draw[very thick, dashed, S1color] (N2) to[bend left=50] (M1);
    \draw[very thick, dashed, S1color] (M1) to[bend left=50] (M2);
    \draw[very thick, dashed, S1color] (N2) to[bend left=50] (R1);
    \draw[very thick, dashed, S1color] (R1) to[bend left=50] (R2);

    \node[n_solid, T1color!50!S1color] at (N0) {}; \node[right=6pt, T1color] at (N0) {root $i$};
    \node[n_sq, draw = T1color!50!S1color]    at (N1) {};
    \node[n_hollow, draw = T1color!50!S1color] at (N2) {};
    
    \node[n_sq, draw = T1color!50!S1color]    at (L1) {}; \node[n_hollow, draw = T1color!50!S1color] at (L2) {};
    \node[n_sq, draw = T1color!50!S1color]    at (M1) {}; \node[n_hollow, draw = T1color!50!S1color] at (M2) {};
    \node[n_sq, draw = T1color!50!S1color]    at (R1) {}; \node[n_hollow, draw = T1color!50!S1color] at (R2) {};

    \node at (0, -5.2) {(a) Monochromatic $(T_1, S_1)$};
\end{scope}

\begin{scope}[shift={(5,0)}]
    \coordinate (N0) at (0, 0);
    \coordinate (N1) at (0, -0.7);
    \coordinate (N2) at (0, -1.4);
    
    \coordinate (L1) at (-0.6, -2.1);
    \coordinate (L2) at (-1.0, -2.8);
    \coordinate (L3) at (-1.0, -3.5);
    \coordinate (L4) at (-1.0, -4.2);
    
    \coordinate (R1) at (0.6, -2.1);
    \coordinate (R2) at (1.0, -2.8);

    \draw[very thick, T2color] (N0) -- (N1) -- (N2);
    \draw[very thick, T2color] (N2) -- (L1) -- (L2) -- (L3) -- (L4);
    \draw[very thick, T2color] (N2) -- (R1) -- (R2);

    \draw[very thick, dashed, S2color] (N0) to[bend right=50] (N1);
    \draw[very thick, dashed,  S2color] (N1) to[bend right=50] (N2);
    \draw[very thick, dashed,  S2color] (N2) to[bend right=50] (L1);
    \draw[very thick, dashed,  S2color] (L1) to[bend right=50] (L2);
    \draw[very thick, dashed,  S2color] (L2) to[bend right=50] (L3);
    \draw[very thick, dashed,  S2color] (L3) to[bend right=50] (L4);
    \draw[very thick, dashed,  S2color] (N2) to[bend left=50] (R1);
    \draw[very thick, dashed,  S2color] (R1) to[bend left=50] (R2);

    \node[n_solid, T2color!50!S2color] at (N0) {}; \node[right=6pt, T2color] at (N0) {$j$};
    \node[n_sq, draw = T2color!50!S2color]    at (N1) {};
    \node[n_hollow, draw = T2color!50!S2color] at (N2) {};
    
    \node[n_sq, draw = T2color!50!S2color]    at (L1) {}; \node[n_hollow, draw = T2color!50!S2color] at (L2) {};
    \node[n_sq, draw = T2color!50!S2color]    at (L3) {}; \node[n_hollow, draw = T2color!50!S2color] at (L4) {};
    \node[n_sq, draw = T2color!50!S2color]    at (R1) {}; \node[n_hollow, draw = T2color!50!S2color] at (R2) {};

    \node at (0, -5.2) {(b) Monochromatic $(T_2, S_2)$};
\end{scope}

\begin{scope}[shift={(10,0)}]
    \coordinate (N0) at (0, 0);
    \coordinate (N2)  at (0, -1.4);
    \coordinate (RS0) at (-0.2, -0.7);
    \coordinate (BS0) at (+0.2, -0.7);

    \coordinate (L2)  at (0, -2.8);
    \coordinate (RS1) at (-0.2, -2.1);
    \coordinate (BS1) at (0.2, -2.1);

    \coordinate (M2)  at (-1.4, -2.8);
    \coordinate (RS2) at (-1.0, -2.1);
    \coordinate (BS2) at (-0.6, -2.1);

    \coordinate (R2)  at (1.4, -2.8);
    \coordinate (RS3) at (1.0, -2.1);
    \coordinate (BS3) at (0.6, -2.1);

    \draw[very thick, dashed, T2color] (N0) -- (BS0) -- (N2);
    \draw[very thick, T1color] (N0) -- (RS0) -- (N2);
    
    \draw[very thick, dashed, T2color] (N2) -- (BS1) -- (L2);
    \draw[very thick, T1color] (N2) -- (RS1) -- (L2);
    
    \draw[very thick, dashed, T2color] (N2) -- (BS2) -- (M2);
    \draw[very thick, T1color] (N2) -- (RS2) -- (M2);
    
    \draw[very thick, dashed, T2color] (N2) -- (BS3) -- (R2);
    \draw[very thick, T1color] (N2) -- (RS3) -- (R2);

    \node[n_solid, T1color!50!T2color]  at (N0) {}; \node[right=6pt] at (N0) {\textcolor{T1color!50!T2color}{$i=j$}};
    \node[n_sq, draw = T2color]     at (BS0) {}; \node[n_sq, draw = T1color]     at (RS0) {}; \node[n_hollow, draw = T1color!50!T2color] at (N2) {};
    \node[n_sq, draw = T2color]     at (BS1) {}; \node[n_sq, draw = T1color]     at (RS1) {}; \node[n_hollow, draw = T1color!50!T2color] at (L2) {};
    \node[n_sq, draw = T2color]     at (BS2) {}; \node[n_sq, draw = T1color]     at (RS2) {}; \node[n_hollow, draw = T1color!50!T2color] at (M2) {};
    \node[n_sq, draw = T2color]     at (BS3) {}; \node[n_sq, draw = T1color]     at (RS3) {}; \node[n_hollow, draw = T1color!50!T2color] at (R2) {};

    \node[align=center] at (0, -5.2) {(c) Bichromatic\\cross-isomorphism $(T_1, S_2)$};
\end{scope}

\begin{scope}[shift={(15,0)}]
    \coordinate (N0)  at (0, 0);
    \coordinate (N2)  at (0, -1.4);
    \coordinate (GS0) at (0.2, -0.7);
    \coordinate (OS0) at (-0.2, -0.7);

    \coordinate (L2)  at (-1.0, -2.8);
    \coordinate (GS1) at (-0.8, -2.1);
    \coordinate (OS1) at (-0.4, -2.1);

    \coordinate (L4)  at (-1.0, -4.2);
    \coordinate (GS2) at (-1.2, -3.5);
    \coordinate (OS2) at (-0.8, -3.5);

    \coordinate (R2)  at (1.0, -2.8);
    \coordinate (GS3) at (0.8, -2.1);
    \coordinate (OS3) at (0.4, -2.1);

    \draw[very thick, S2color] (N0) -- (GS0) -- (N2);
    \draw[very thick,  S1color] (N0) -- (OS0) -- (N2);
    
    \draw[very thick, S2color] (N2) -- (GS1) -- (L2);
    \draw[very thick,  S1color] (N2) -- (OS1) -- (L2);
    
    \draw[very thick, S2color] (L2) -- (GS2) -- (L4);
    \draw[very thick,  S1color] (L2) -- (OS2) -- (L4);
    
    \draw[very thick, S2color] (N2) -- (GS3) -- (R2);
    \draw[very thick, S1color] (N2) -- (OS3) -- (R2);

    \node[n_solid, S1color!50!S2color]  at (N0) {}; \node[right=6pt, T1color!50!T2color] at (N0) {\textcolor{T1color!50!T2color}{$i=j$}};
    \node[n_sq, draw = S2color]     at (GS0) {}; \node[n_sq, draw = S1color]     at (OS0) {}; \node[n_hollow, draw = S1color!50!S2color] at (N2) {};
    \node[n_sq, draw = S2color]     at (GS1) {}; \node[n_sq, draw = S1color]     at (OS1) {}; \node[n_hollow, draw = S1color!50!S2color] at (L2) {};
    \node[n_sq, draw = S2color]     at (GS2) {}; \node[n_sq, draw = S1color]     at (OS2) {}; \node[n_hollow, draw =T1color!50!T2color]  at (L4) {}; 
    \node[n_sq, draw = S2color]     at (GS3) {}; \node[n_sq, draw = S1color]     at (OS3) {}; \node[n_hollow, draw = S1color!50!S2color] at (R2) {};

    \node[align=center] at (0, -5.2) {(d) Bichromatic\\within-isomorphism $(T_1, T_2)$};
\end{scope}
\end{tikzpicture}
}
\caption{Different cases of fully overlapping branches.}
\label{fig:fully-overlap}
\end{figure}
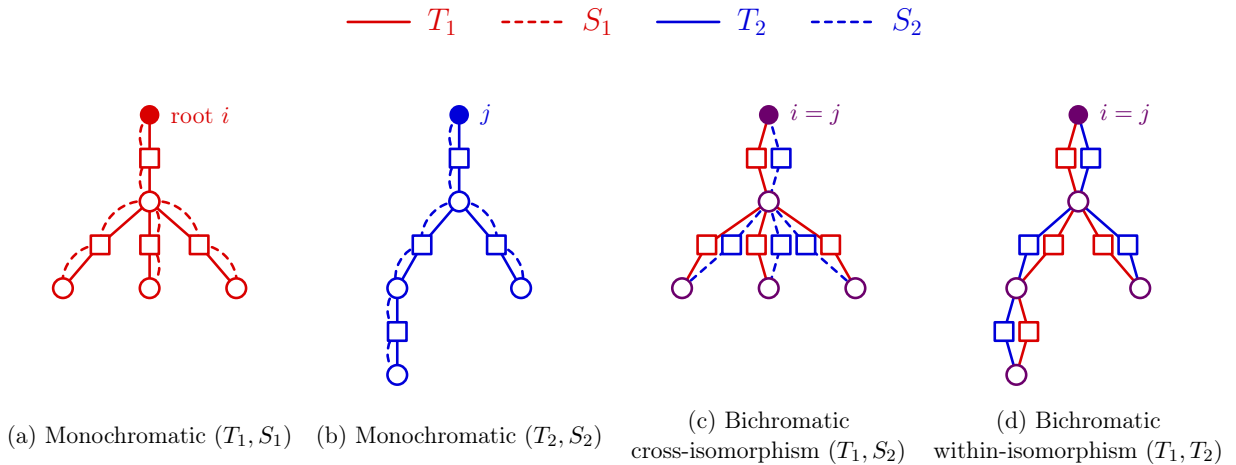

%% file: fig/decomposition-u.tex
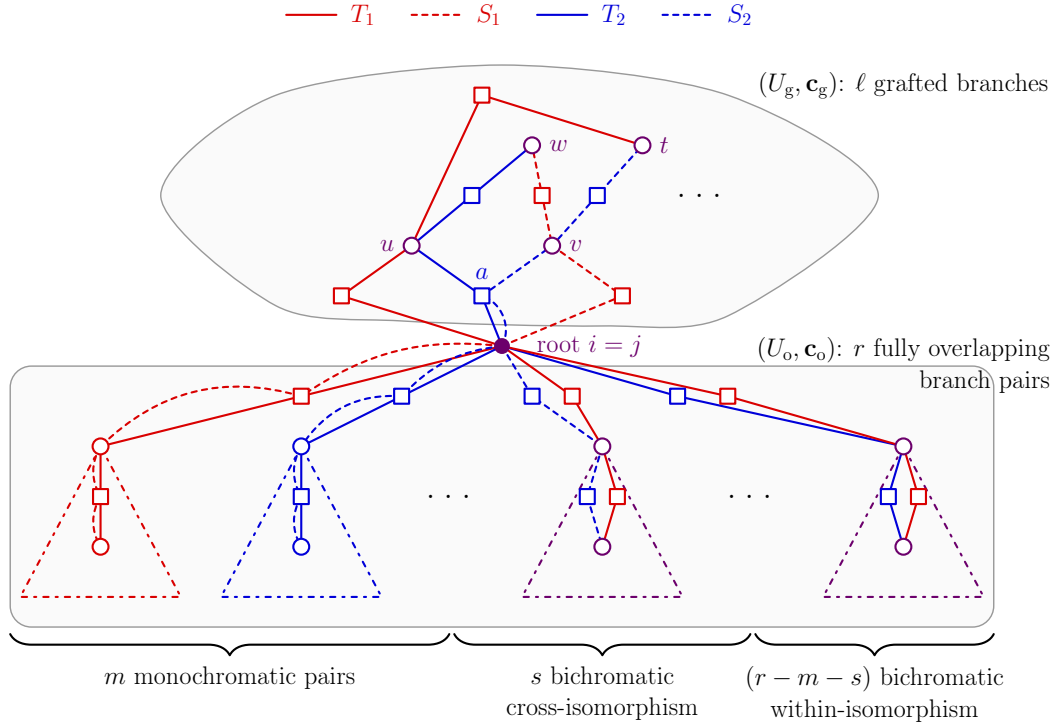
\begin{figure}[htbp]
\centering
\colorlet{T1color}{red!85!black}
\colorlet{S1color}{red!85!black}
\colorlet{T2color}{blue!85!black}
\colorlet{S2color}{blue!85!black}

\resizebox{0.85\textwidth}{!}{
\begin{tikzpicture}[
    line cap=round, line join=round,
    n_solid/.style={circle, fill, minimum size=9pt, inner sep=0pt},
    n_hollow/.style={circle, draw, fill=white, very thick, minimum size=9pt, inner sep=0pt},
    n_sq/.style={rectangle, draw, fill=white, very thick, minimum size=8.5pt, inner sep=0pt}
]

\begin{scope}[shift={(0, 6.6)}]
        \draw[very thick, T1color] (-4.3, 0) -- (-3.3, 0) node[right=4pt] {\Large $T_1$};
        \draw[very thick, dashed, S1color] (-1.8, 0) -- (-0.8, 0) node[right=4pt] {\Large $S_1$};
        \draw[very thick, T2color] (0.7, 0) -- (1.7, 0) node[right=4pt] {\Large $T_2$};
        \draw[very thick, dashed, S2color] (3.2, 0) -- (4.2, 0) node[right=4pt] {\Large $S_2$};
\end{scope}

\newcommand{\drawtriangle}[2]{
    \draw[very thick, dash pattern=on 2pt off 4pt, loosely dashdotted, #2] (#1) -- +(-1.6, -3) -- +(1.6, -3) -- cycle;
}

\coordinate (Root) at (0, 0); \node[right=16pt, T1color!50!T2color, font =\Large] at (0, 0) {root $i=j$}; 

\coordinate (G_1) at (-3.2,  1);
\coordinate (G_2) at (-0.4,  1);
\coordinate (G_3) at ( 2.4,  1);

\coordinate (G_4) at (-1.8,  2);
\coordinate (G_5) at ( 1,  2);

\coordinate (G_6) at (-0.6,  3);
\coordinate (G_7) at (0.8,  3);
\coordinate (G_8) at (1.9,  3);

\coordinate (G_9) at (0.6,  4);
\coordinate (G_10) at (2.8,  4);

\coordinate (G_11) at (-0.4,  5);

\coordinate (S1) at (-4.0, -1); \coordinate (C1) at (-8.0, -2);
\coordinate (S2) at (-2, -1); \coordinate (C2) at (-4, -2);
\coordinate (S3_L) at ( 0.6, -1); \coordinate (S3_R) at ( 1.4, -1);
\coordinate (C3)   at ( 2, -2);
\coordinate (S4_L) at ( 3.5, -1); \coordinate (S4_R) at ( 4.5, -1);
\coordinate (C4)   at ( 8, -2);

\coordinate (IS1) at (-8, -3); \coordinate (IC1) at (-8, -4);
\coordinate (IS2) at (-4, -3); \coordinate (IC2) at (-4, -4);
\coordinate (IS3_L) at ( 1.7, -3); \coordinate (IS3_R) at ( 2.3, -3);
\coordinate (IC3)   at ( 2, -4);
\coordinate (IS4_L) at ( 7.7, -3); \coordinate (IS4_R) at ( 8.3, -3);
\coordinate (IC4)   at ( 8, -4);

 \draw[gray!80, thick, fill=gray!4] plot[smooth cycle] coordinates {(-2, 0.5) (-4.5, 0.8) (-6.8, 3.0) (-4.5, 5.0) (0, 5.6) (4.5, 5.0) (7.5, 3.0) (4.5, 0.6) (2, 0.4)};

 \draw[gray!80, thick, fill=gray!4, rounded corners=12pt] (-9.8, -0.4) rectangle (9.8, -5.6);

\drawtriangle{C1}{T1color}
\drawtriangle{C2}{T2color}
\drawtriangle{C3}{T1color!50!T2color}
\drawtriangle{C4}{T1color!50!T2color}


\draw[very thick, T1color] (Root) -- (G_1) -- (G_4) -- (G_11) -- (G_10);
\draw[very thick, T2color] (Root) -- (G_2) -- (G_4) -- (G_6) -- (G_9);
\draw[very thick, S2color, dashed] (G_2) -- (G_5) -- (G_8)-- (G_10);
\draw[very thick, S2color, dashed] (Root) to[bend right=50] (G_2);
\draw[very thick, S1color, dashed] (Root) -- (G_3) -- (G_5)-- (G_7)-- (G_9);

\draw[very thick, T1color] (Root) -- (S1) -- (C1);
\draw[very thick, dashed, T1color] (Root) to[bend right=20] (S1);
\draw[very thick, dashed, T1color] (S1) to[bend right=30] (C1);

\draw[very thick, T1color] (C1) -- (IS1) -- (IC1);
\draw[very thick, dashed, T1color] (C1) to[bend right=30] (IS1);
\draw[very thick, dashed, T1color] (IS1) to[bend right=30] (IC1);

\draw[very thick, T2color] (Root) -- (S2) -- (C2);
\draw[very thick, dashed, T2color] (Root) to[bend right=20] (S2);
\draw[very thick, dashed, T2color] (S2) to[bend right=30] (C2);

\draw[very thick, T2color] (C2) -- (IS2) -- (IC2);
\draw[very thick, dashed, T2color] (C2) to[bend right=30] (IS2);
\draw[very thick, dashed, T2color] (IS2) to[bend right=30] (IC2);

\draw[very thick, dashed, T2color] (Root) -- (S3_L) -- (C3);
\draw[very thick, T1color] (Root) -- (S3_R) -- (C3);

\draw[very thick, dashed, T2color] (C3) -- (IS3_L) -- (IC3);
\draw[very thick, T1color] (C3) -- (IS3_R) -- (IC3);

\draw[very thick, T2color] (Root) -- (S4_L) -- (C4);
\draw[very thick, T1color] (Root) -- (S4_R) -- (C4);

\draw[very thick, T2color] (C4) -- (IS4_L) -- (IC4);
\draw[very thick, T1color] (C4) -- (IS4_R) -- (IC4);


\node[n_sq, draw=T1color] at (G_1) {};
\node[n_sq, draw=T2color] at (G_2) {};\node[above=6pt, T2color] at (G_2) {\Large $a$}; 
\node[n_sq, draw=T1color] at (G_3) {}; 

\node[n_hollow, draw=T1color!50!T2color] at (G_4) {};\node[left=6pt, T1color!50!T2color] at (G_4) {\Large $u$}; 
\node[n_hollow, draw=T1color!50!T2color] at (G_5) {}; \node[right=6pt, T1color!50!T2color] at (G_5) {\Large $v$}; 

\node[n_sq, draw=T2color] at (G_6) {};
\node[n_sq, draw=T1color] at (G_7) {};
\node[n_sq, draw=T2color] at (G_8) {};

\node[n_hollow, draw=T1color!50!T2color] at (G_9) {}; \node[right=6pt, T1color!50!T2color] at (G_9) {\Large $w$};
\node[n_hollow, draw=T1color!50!T2color] at (G_10) {}; \node[right=6pt, T1color!50!T2color] at (G_10) {\Large $t$};
\node[n_sq, draw=T1color] at (G_11) {};

\node[n_sq, draw=T1color] at (S1) {};            \node[n_hollow, draw=T1color] at (C1) {};
\node[n_sq, draw=T2color] at (S2) {};            \node[n_hollow, draw=T2color] at (C2) {};

\node[n_sq, draw=T2color] at (S3_L) {}; \node[n_sq, draw=T1color] at (S3_R) {};
\node[n_hollow, draw=T1color!50!T2color] at (C3) {};

\node[n_sq, draw=T2color] at (S4_L) {}; \node[n_sq, draw=T1color] at (S4_R) {};
\node[n_hollow, draw=T1color!50!T2color] at (C4) {};

\node[n_sq, draw=T1color] at (IS1) {};          \node[n_hollow, draw=T1color] at (IC1) {};
\node[n_sq, draw=T2color] at (IS2) {}; 
\node[n_hollow, draw=T2color] at (IC2) {};

\node[n_sq, draw=T2color] at (IS3_L) {}; \node[n_sq, draw=T1color] at (IS3_R) {};
\node[n_hollow, draw=T1color!50!T2color] at (IC3) {};

\node[n_sq, draw=T2color] at (IS4_L) {}; \node[n_sq, draw=T1color] at (IS4_R) {};
\node[n_hollow, draw=T1color!50!T2color] at (IC4) {};

\node[n_solid, fill=T1color!50!T2color] at (Root) {};


\node[font=\Large,  inner sep=4pt, align=left] at (8, 5.2) {$(\ur, \bcr)$: $\ell$ grafted branches}; 

\node[font=\Large, inner sep=4pt, align=right] at (8, -0.38) {$(\uo, \bco)$: $r$ fully overlapping\\branch pairs};

\draw[decorate, decoration={brace, mirror, amplitude=12pt, raise=4pt}, very thick] 
    (-9.8, -5.6) -- (-1.05, -5.6) node[midway, below=18pt, align=center, font=\Large] {$m$ monochromatic pairs};

\draw[decorate, decoration={brace, mirror, amplitude=12pt, raise=4pt}, very thick] 
    (-0.95, -5.6) -- (4.95, -5.6) node[midway, below=18pt, align=center, font=\Large] {$s$ bichromatic\\cross-isomorphism};

\draw[decorate, decoration={brace, mirror, amplitude=12pt, raise=4pt}, very thick] 
    (5.05, -5.6) -- (9.8, -5.6) node[midway, below=18pt, align=center, font=\Large] {$(r-m-s)$ bichromatic\\within-isomorphism};

\node[font=\Huge] at (-1, -3) {\dots};
\node[font=\Huge] at (5, -3) {\dots};
\node[font=\Huge] at (4, 3) {\dots};

\end{tikzpicture}
}
\vspace{0.5em}
\caption{Schematic illustration of decomposition of a decorated union $(U, \bc)$.} 
\label{fig:decomposition}
\end{figure}

%% file: low_degree_analysis.tex
\section{Limits of Tree Polynomials via Low-Degree Analysis} \label{sect:low-deg}

Recall that our matching algorithm counts only trees whose $d$-nodes all have degree $2$.
In this section, we show that enlarging the tree family to allow $d$-nodes of arbitrary even degree yields only a negligible improvement in distinguishing power for a closely related hypothesis testing problem: 
\begin{itemize}
    \item Under the null hypothesis $\mathbb{H}_0$: $X,Y\in \R^{n\times d}$ are independent with i.i.d. $N(0,1)$ entries;
    \item Under the alternative hypothesis $\mathbb{H}_1$: the pair $(X, Y)$ follows $Y = \rho \Pi XQ + \sqrt{1-\rho^2} Z$, where $X,Z\in \R^{n\times d}$ are independent with i.i.d. $N(0,1)$ entries,  $\Pi$ is a uniform permutation matrix on $[n]$, and $Q\in \cO_d$ is Haar-distributed.
\end{itemize}

We employ the widely adopted low-degree framework to study the
limits of polynomial tests  \cite{hopkins2017efficient,kunisky2019notes,wein2025computational}. 
Consider the space $\calF$ of functions
$f: \reals^{n\times d} \times \reals^{n\times d} \to \reals$
endowed with inner product $\langle f,g\rangle
\triangleq \mathbb E_0[f(X,Y)g(X,Y)]$. It is well-known that the Hermite polynomials form an orthogonal basis of $\calF.$ Let $\calF_{\le D}$ denote the space of all multivariate polynomials of $X,Y$ with a total degree at most $D$. 
A standard way to measure the distinguishing power of polynomials in the space $\calF_{\le D}$ is  the low-degree advantage defined as
\begin{align}
\Adv_{\le D} \triangleq \max_{f \in \calF_{\le D}}\frac{\E_1[f] - \E_0[f]}{\sqrt{\var_0(f)}}. 
\label{eq:def_low_degree_adv}
\end{align}
The ratio above measures the mean separation of a test $f$ under the two hypotheses relative to its standard deviation under the null hypothesis, and therefore has a natural interpretation as a signal-to-noise ratio. Hence, following the literature, we say that degree-$D$ polynomials fail to strongly (resp.\ weakly) distinguish the two hypotheses if $\Adv_{\le D}=O(1)$ (resp.\ $=o(1)$). 

It is shown in~\cite{li2025computational} that $\Adv_{\le D} = o(1)$ as $n \to \infty$ when $\rho \to 0$ and $D \ll 1/\rho$. In particular, this implies that all $O(\log(n))$-degree polynomials fail to weakly distinguish the two hypotheses when $\rho \ll 1/\log(n)$.
In this section, we take a step further by showing that all multilinear polynomials built from trees have bounded low-degree advantage and hence fail to strongly distinguish the two hypotheses when $\rho^2<\sqrt{\alpha}.$ This suggests that breaking the $\sqrt{\alpha}$ threshold requires counting graphs beyond trees,  for testing and likely for matching as well.

\subsection{Graph Polynomials} 
For bipartite simple graphs $G$ and $H$ on $[n]\times [d]$, define the multilinear polynomial
\$ 
\phi_{G, H}(X, Y) = X^G \cdot Y^H \, .
\$
Since  the entries of $X$ and $Y$ are independent under $\cH_0$, 
and $G$ and $H$ are simple graphs, the polynomials $\{\phi_{G, H}(X, Y)\}$ are orthonormal:
\$ 
\langle \phi_{G, H}, \phi_{G', H'} \rangle = \E_0[\phi_{G, H}(X,Y)\phi_{G', H'}(X,Y)] = \Ind_{\{G = G', H = H'\}}.
\$
We are interested in the subspace spanned by those orthogonal polynomials whose graph indices belong to a  prescribed family $\mathcal{G}$ of unlabeled graphs:
$$
\calF_\calG\triangleq \operatorname{span}(\{\phi_{T_1, S_1}\colon T_1 \cong T, S_1 \cong S, T, S \in \calG\}).
$$
Analogously to~\eqref{eq:def_low_degree_adv}, we measure the distinguishing power of $\calG$  by its advantage 
\$ 
\Adv_{\calG} \triangleq \max_{h \in \calF_\calG}\frac{\E_1[h] - \E_0[h]}{\sqrt{\var_0(h)}}.
\$
Without loss of generality, we assume $\E_0[h]=0$. Moreover, for any $h \in \calF_{\calG}$, 
$$
\E_1[h]=\E_0[Lh] 
=\E_0[ L_{\calG} h] \le \sqrt{\E_0[L_{\calG}^2]}\sqrt{\E_0[h^2]},
$$
where 
$$
L_{\calG} \triangleq \sum_{T,S \in \calG} \sum_{T_1 \cong T, S_1 \cong S}  
\iprod{L}{\phi_{T_1,S_1}}  \phi_{T_1,S_1} $$ 
denotes the projection of the likelihood ratio $L=\mathbb{P}_1(X,Y)/\mathbb{P}_0(X,Y)$ onto $\calF_{\calG}$. The above inequality follows from Cauchy--Schwarz and is achieved when
$h= L_{\calG}$. Therefore, 
\begin{align}
\Adv_{\calG} = 
\sqrt{\E_0[L_{\calG}^2]}
= \left(
\sum_{T,S\in\calG} \sum_{T_1 \cong T, S_1 \cong S}
\langle L,\phi_{T_1,S_1}\rangle^2
\right)^{1/2}. \label{eq:Adv_graph}
\end{align}
The projection coefficient $\langle L, \phi_{T_1, S_1}\rangle$ can be evaluated via the Weingarten calculus:
\begin{align}
\langle L, \phi_{T_1, S_1}\rangle 
 = \E_{\Pi}\left[\E_{X,Y}
\left[X^{T_1} Y^{S_1} \given \Pi\right] \right]
= \frac{1}{n!} \sum_{\pi} 
\sum_{\bc \in \cC(\pi(T_1)\cup_{n} S_1)} \rho^{2K} \wg(\mu(\bc),d). \label{eq:proj_coeff}
\end{align}
Observe that 
$\langle L, \phi_{T_1, S_1}\rangle$ depends only on the isomorphism classes of $T_1$
and $S_1$. Moreover, $\langle L, \phi_{T_1, S_1}\rangle$ is non-zero only when $\cC(\pi(T_1)\cup_{n} S_1) \neq \emptyset$ for some permutation $\pi$. Recalling from \Cref{def:alt-cir-decomp}, this requires that 
\begin{enumerate}
    \item Every $d$-node in both $T_1$ and $S_1$ must have even degree;
    \item The $n$-nodes in $T_1$ and $S_1$ have matching degree sequences under the relabeling.
\end{enumerate} 
Therefore, it suffices to restrict $\calG$ to graphs with even $d$-node degrees. Let $\calT$ be the family of unlabeled simple
bipartite trees on $[n]\times[d]$ with $2K$ edges and even $d$-node degrees, and let $\calT^*$ be the subfamily in which every $d$-node has degree $2$. 
The following theorem bounds the advantage of $\calT$ in terms of the contribution from $\calT^*$.

\begin{theorem}\label{thm:lb-deg-2}
If $d\ge \poly(K)$, then 
\$ 
& \Adv^2_{\calT} 
\le \left(1 + \frac{\poly(K)}{d}\right)
\sum_{T \in \calT^*} \sum_{\substack{T_1\cong T \\ T_2 \cong T}} \langle L, \phi_{T_1, T_2}\rangle^2. 
\$
Furthermore, 
\begin{align*} 
\sum_{T \in \calT^*} \sum_{\substack{T_1\cong T \\ T_2 \cong T}} \langle L, \phi_{T_1, T_2}\rangle^2 \le \left(1 + \frac{\poly(K)}{d}\right) \left(\frac{\rho^4}{\alpha+o(1)}\right)^{K}, 
\end{align*}
which is $O(1)$ when $\rho^2 < \sqrt{\alpha}$. 
\end{theorem}

\Cref{thm:lb-deg-2} shows that the main contribution to $\Adv_{\calT}$ comes from pairs of trees that are both isomorphic to some common $T \in \calT^*$. In particular, when $\rho^2 < \sqrt{\alpha}$, $\Adv_{\calT}=O(1)$; hence no tree polynomial in $\calF_{\calT}$ strongly separates $\cH_0$ from $\cH_1$.
Furthermore, as shown in
\eqref{eq:l-phi-cong} later, for any $T_1 , T_2 \cong T$ with $T \in \calT^*$,
$
\langle L, \phi_{T_1, T_2}\rangle \propto \aut(T). 
$
Therefore, up to negligible terms, $\Adv_{\calT}$ is achieved by the tree polynomial  
\$ 
f_{T^*}(X,Y) \triangleq \sum_{T \in \calT^*} \sum_{\substack{T_1\cong T \\ T_2 \cong T}} \langle L, \phi_{T_1, T_2}\rangle \phi_{T_1, T_2}(X,Y) \propto \sum_{\substack{T\in \calT^* }} \aut(T) \sum_{T_1\cong T} X^{T_1} \sum_{T_2\cong T} Y^{T_2},
\$
This coincides with the similarity score
defined in \eqref{def:score-ij} for matching, except that the matching statistic counts rooted trees.

The proof of \Cref{thm:lb-deg-2} also explains why trees in $\calT^*$ give the leading contribution. In the calculation of the advantage, the number of ways to choose the $n$-node labels cancels the probability that $\pi(T_1)$ and $S_1$ are properly matched under the random permutation appearing in the projection coefficient \eqref{eq:proj_coeff}. By contrast, the $d$-node labels of $T_1$ and $S_1$ need not match each other, and thus they contribute powers of $d$ to the advantage. Therefore, among trees with $2K$ edges, whose total number of nodes is fixed at $2K+1$, the dominant contribution comes from those with the largest possible number of $d$-nodes, namely $D=K$. These are exactly the trees in $\calT^*$.

\subsection{Proof of \Cref{thm:lb-deg-2}}
\begin{proof}[Proof of \Cref{thm:lb-deg-2}]
Recall from~\eqref{eq:Adv_graph} that 
\begin{align*}
\Adv^2_{\calT} = \sum_{T, S \in \calT} \sum_{T_1\cong T, S_1 \cong S}
\langle L, \phi_{T_1, S_1}\rangle^2.
\end{align*}
We split the sum into three parts: isomorphic pairs ($T_1, S_1 \cong T$ for some $T\in \calT^*$); non-isomorphic pairs ($T_1 \cong T, S_1 \cong S$ for some distinct $T,S \in \calT^*$); and all remaining pairs ($T_1 \cong T, S_1 \cong S$ for some $(T,S) \notin \calT^*\times \calT^*$).

\paragraph{Leading term: $T_1, S_1 \cong T$ for some $T\in \calT^*$.} 
Recall from~\eqref{eq:proj_coeff} that 
\begin{align}
\langle L, \phi_{T_1, S_1}\rangle 
= \frac{1}{n!} \sum_{\pi } 
\sum_{\bc \in \cC(\pi(T_1)\cup_{n} S_1)} \rho^{2K} \wg(\mu(\bc),d).  \label{eq:l-phi-t-s-inn}
\end{align}
By \Cref{thm:altcir}, under the assumption that $d \ge \poly(K)$, every circuit decomposition
$\bc$ with $\mu_1(\bc)=K$ satisfies $\rho^{2K}\wg(\mu(\bc),d) \triangleq w_{d,K} = (1 + \poly(K)/d)(\rho^2/d)^K$.
Moreover, given $T_1, S_1$, there are 
$\aut(T) \times (n-K-1)!$ distinct permutations $\pi$ such that 
each edge pair of
$\pi(T_1)$ fully overlaps with an edge pair of $S_1$, so that $\pi(T_1)\cup_n S_1$ admits a
unique decomposition $\bc$ with $\mu_1(\bc)=K$. Hence, we obtain
\begin{align}
 \sum_{\pi }\sum_{\substack{\bc \in \cC(\pi(T_1)\cup_n S_1) \\ \mu_1(\bc) = K}} \rho^{2K} \wg(\mu(\bc),d) & = w_{d,K}
\left|\left\{ (\pi,\bc)\colon  \bc \in \cC(\pi(T_1)\cup_n 
S_1), \mu_1(\bc) = K \right\}\right| \nonumber \\
& = \left(1 + \frac{\poly(K)}{d}\right)\left(\frac{\rho^2}{d}\right)^K \aut(T) \times (n-K-1)!. \label{eq:contribution_mu_K}
\end{align}

For circuit decompositions $\bc$ with $\mu_1(\bc) = \mu_1 \le K-2$,
 by \eqref{eq:bound-mu1} and the assumption that $d \ge \poly(K),$
$\rho^{2K} \left| \wg(\mu(\bc),d) \right| \le 
C (\rho^2/d)^K(4/d)^{(K-\mu_1)/2}$. Recall from \Cref{lem:bound-mu1-le-k-2} that, given $S_1$, 
\begin{align*}
\left| \left\{ (H, \bc): \bc \in \cC(H \cup_{n} S_1), \mu_1(\bc)=\mu_1 \right\}\right| \le (4K)^{K-\mu_1},
\end{align*}
where $H$ is a union of $K$ edge pairs (with unlabeled $d$-nodes). 
Moreover, for each such $H$, there are at most $\aut(T) \times (n-K-1)!$ distinct permutations $\pi$ such that each edge pair of
$\pi(T_1)$ fully overlaps with an edge pair of $H$. It follows that 
\begin{align}
& \sum_{\mu_1 \le K-2} \sum_{\pi }\sum_{\substack{\bc \in \cC(\pi(T_1)\cup_n S_1) \\ \mu_1(\bc) =\mu_1}} \rho^{2K} \left| \wg(\mu(\bc),d) \right| \nonumber \\
&  \le C \sum_{\mu_1 \le K-2} \left(\frac{\rho^2}{d}\right)^K\left(\frac{4(4K)^2}{d}\right)^{(K-\mu_1)/2} \times \aut(T) \times (n-K-1)! \nonumber \\
& \le C 
\left(\frac{\rho^2}{d}\right)^K \aut(T)
\times (n-K-1)!
\times  \frac{4(4K)^2}{d}. \label{eq:contribution_mu_K_2}
\end{align}
Therefore, 
substituting~\eqref{eq:contribution_mu_K} and~\eqref{eq:contribution_mu_K_2} into \eqref{eq:l-phi-t-s-inn} yields 
\#\label{eq:l-phi-cong}
\langle L, \phi_{T_1, S_1}\rangle 
 = \left(1 + \frac{\poly(K)}{d}\right) \frac{(n - K-1)!}{n!} \left(\frac{\rho^2}{d}\right)^K \aut(T), 
\#
Finally, given an unlabeled tree $T \in \calT^*$, there are 
$
\frac{\binom{n}{K+1}(K+1)!}{\aut(T)}\binom{d}{K}K!
$
labeled trees $T_1$ that are isomorphic to $T$; the same for labeled trees $S_1$. Therefore,
\begin{align}
& \sum_{T\in \calT^*} \sum_{T_1, S_1 \cong T}
\langle L, \phi_{T_1, S_1}\rangle^2 \nonumber \\
&= \sum_{T\in\calT^*} \left[\left(1 + \frac{\poly(K)}{d}\right) \frac{(n - K-1)!}{n!} \left(\frac{\rho^2}{d}\right)^K \aut(T) \right]^2 \times \left[\frac{\binom{n}{K+1}(K+1)!}{\aut(T)}\binom{d}{K}K!\right]^2 \nonumber\\
& = \left(1 + \frac{\poly(K)}{d}\right) \rho^{4K}|\cT^*|. \label{eq:l-phi-cong-1}
\end{align}

\paragraph{Non-isomorphic pairs: $T_1 \cong T, S_1 \cong S$ for some $T\neq S \in \calT^*$.} In this case, for any permutation $\pi$, any circuit decomposition $\bc$ of $\pi(T_1)\cup_n S_1$ must have $\mu_1(\bc) \le K-2$. 
Therefore, it follows from~\eqref{eq:contribution_mu_K_2} that
\begin{align*}
\left| \langle L, \phi_{T_1, S_1}\rangle \right|
& \le  \frac{1}{n!} \sum_{\pi } 
\sum_{\substack{\bc \in \cC(\pi(T_1)\cup_{n} S_1) \notag \\ \mu_1(\bc) \le K-2}} \rho^{2K} \left| \wg(\mu(\bc),d) \right| \\
& \le C \frac{(n - K-1)!}{n!} \left(\frac{\rho^2}{d}\right)^K \aut(T) \times \frac{4(4K)^2}{d}.
\end{align*}
Moreover, given $T_1 \cong T$, we have
\begin{align} 
  \sum_{S \in \calT^*} \sum_{S_1 \cong S } \left| \langle L, \phi_{T_1, S_1}\rangle \right|  
& \le 
 \frac{1}{n!}   \sum_{\pi}  \sum_{S \in \calT^*} \sum_{S_1 \cong S } \sum_{\substack{\bc \in \cC(\pi(T_1)\cup_{n} S_1)  \\ \mu_1(\bc) \le K-2}} \rho^{2K} \left| \wg(\mu(\bc),d) \right| \nonumber \\
& =\sum_{S \in \calT^*} \sum_{S_1 \cong S }  \sum_{\substack{\bc \in \cC(T_1\cup_{n} S_1) \notag \\ \mu_1(\bc) \le K-2}} \rho^{2K} \left| \wg(\mu(\bc),d) \right| \nonumber \\
& \le C \left(\frac{\rho^2}{d}\right)^{K}   \sum_{\substack{\mu_1\le K-2}} \left(\frac{4}{d}\right)^{(K-\mu_1)/2}  
\sum_{S \in \calT^*} \sum_{S_1 \cong S }  \sum_{\substack{\bc \in \cC(T_1\cup_{n} S_1) \\ \mu_1(\bc) =\mu_1}}  1, \label{eq:sum_T_1_bound}
\end{align}
where the equality holds as the inner double sum does not depend on $\pi$, and the last inequality follows from $\rho^{2K} \left| \wg(\mu(\bc),d) \right| \le 
C (\rho^2/d)^K(4/d)^{(K-\mu_1)/2}$. 
It follows from~\Cref{lem:bound-mu1-le-k-2} that
$$
\left| \{ (S_1,\bc): S_1 \cong S \text{ for some } S \in \cT^*, \bc \in \cC(T_1 \cup_n S_1), \mu_1(\bc)=\mu_1  \} \right| \le (4K)^{K-\mu_1}  \times \binom{d}{K} K!,
$$
where the factor $\binom{d}{K} K!$ comes from labeling the $K$ $d$-nodes in $S_1.$ The above also follows from the more  general~\Cref{lmm:s-u-c} proved below. 
Combining the last three displayed equations
and summing over $\mu_1 \le K-2$ yields that given $T_1 \cong T$,
$$
 \sum_{S \in \calT^*} \sum_{S_1 \cong S }   \langle L, \phi_{T_1, S_1}\rangle^2
 \le C \frac{(n - K-1)!}{n!}  \left(\frac{\rho^2}{d}\right)^{2K} \aut(T) \times   \left(\frac{4(4K)^2}{d}\right)^2 \times \binom{d}{K} K!.
$$
Finally, note that there are $
\frac{\binom{n}{K+1}(K+1)!}{\aut(T)}\binom{d}{K}K!
$
distinct labeled $T_1$ that are isomorphic to $T \in \calT^*$. Therefore, 
\begin{align}
& \sum_{\substack{T, S\in \calT^* \\ T\neq S}}
\sum_{\substack{T_1\cong T \\ S_1 \cong S}} \langle L, \phi_{T_1, S_1}\rangle^2 \nonumber \\
& \le C |\cT^*|  \left(\frac{\rho^2}{d}\right)^{2K} \frac{(n - K-1)!}{n!}  \binom{n}{K+1}(K+1)!
\left(\frac{4(4K)^2}{d}\right)^2 
\times \left[\binom{d}{K}K!\right]^2 \nonumber \\
& \le  C  \rho^{4K}|\cT^*|
\left(\frac{4(4K)^2}{d}\right)^2. 
\label{eq:lb-cT-2}  
\end{align}

\paragraph{Remaining pairs: ($T_1 \cong T, S_1 \cong S$ for some $(S, T) \notin \calT^*\times \calT^*$).} 
Let $\calT_D\subset  \calT $ denote the subfamily of trees with $2K$ edges and $D$ $d$-nodes with even degrees. Then $\calT_K=\calT^*$.
Recall that $\langle L, \phi_{T_1, S_1}\rangle$ is non-zero only when $T_1$ and $S_1$ have the same number of $n$-nodes and the $d$-nodes have even degrees.  Therefore, 
\begin{align}
\sum_{(T, S) \notin \calT^* \times \calT^*}
\sum_{\substack{T_1\cong T \\ S_1 \cong S}} \langle L, \phi_{T_1, S_1}\rangle^2 
= 
\sum_{D = 1}^{K-1} \sum_{T, S\in \calT_D} 
\sum_{\substack{T_1\cong T \\ S_1 \cong S}} \langle L, \phi_{T_1, S_1}\rangle^2. \label{eq:sum_D_bound}
\end{align}
Note that $\calT_D$ has $N \triangleq 2K+1-D$ $n$-nodes. Given any $T_1 \cong T, S_1 \cong S$ with $T, S \in \calT_D,$
we first pair the edges incident to each $d$-node of $S_1$.
The proof of \Cref{lmm:s-u-c} (cf.~\eqref{eq:edge_pair_bound}) shows that the number of such pairings is at most $(2K)^{2(K-D)}$. After fixing this pairing,
analogous to the derivation of~\eqref{eq:l-phi-cong}, we have
\begin{align*}
\left| \langle L, \phi_{T_1, S_1}\rangle \right|
& \le  \frac{1}{n!} \sum_{\pi } 
\sum_{\bc \in \cC(\pi(T_1)\cup_{n} S_1) } \rho^{2K} \left| \wg(\mu(\bc),d) \right| \\
& \le \frac{C}{n!} \sum_{\mu_1 \le K} \left(\frac{\rho^2}{d}\right)^K\left(\frac{4}{d}\right)^{(K-\mu_1)/2} (4K)^{K-\mu_1} (2K)^{2(K-D)} \times \aut(T) (n-N)! \nonumber \\
& \le C \frac{(n - N)!}{n!} \left(\frac{\rho^2}{d}\right)^K (2K)^{2(K-D)}  \aut(T)
\end{align*}
Moreover, given $T_1 \cong T$, analogous to~\eqref{eq:sum_T_1_bound}, we have
$$
\sum_{S \in \calT_D} \sum_{S_1 \cong S }\left| \langle L, \phi_{T_1, S_1}\rangle \right| 
\le C \left(\frac{\rho^2}{d}\right)^{K}   \sum_{\mu_1\le K} \left(\frac{4}{d}\right)^{(K-\mu_1)/2}  
\sum_{S \in \calT_D} \sum_{S_1 \cong S }  \sum_{\substack{\bc \in \cC(T_1\cup_{n} S_1) \\ \mu_1(\bc) =\mu_1}}  1.
$$
By~\Cref{lmm:s-u-c}, we  have
\begin{align}
& \left| \{ (S_1,\bc): S_1 \cong S \text{ for some } S \in \cT_D, \bc \in \cC(T_1 \cup_n S_1), \mu_1(\bc)=\mu_1  \} \right| \nonumber \\ 
&\le (4K)^{K-\mu_1} (2K)^{4(K-D)} \binom{d}{D} D!.  \label{eq:sucuc-d}
\end{align}
Combining the last three displayed equations
and summing over $\mu_1 \le K$ yields that given $T_1 \cong T$, 
$$
\sum_{S \in \calT_D} \sum_{S_1 \cong S }\langle L, \phi_{T_1, S_1}\rangle^2
 \le C\frac{(n - N)!}{n!}  \left(\frac{\rho^2}{d}\right)^{2K} \aut(T)  \times (2K)^{6(K-D)} \binom{d}{D} D!.
$$
Recall that there are $
\frac{\binom{n}{N}N!}{\aut(T)}\binom{d}{D}D!
$
distinct labeled trees $T_1$ that are isomorphic to $T \in \calT_D$. Therefore, 
\begin{align*}
 \sum_{\substack{T, S\in \calT_D }}
\sum_{\substack{T_1\cong T \\ S_1 \cong S}} \langle L, \phi_{T_1, S_1}\rangle^2 
& \le C |\cT_D|  \left(\frac{\rho^2}{d}\right)^{2K} \frac{(n - N)!}{n!}  \binom{n}{N}N!
\times  (2K)^{6(K-D)} \left[\binom{d}{D}D!\right]^2 \nonumber \\
& \le  C \rho^{4K}|\cT_D|
\left(\frac{8K^3}{d^2}\right)^{2(K-D)}. 
\end{align*}
It remains to bound $|\cT_D|$ as follows:
\$
|\cT_D| \le (\alpha+o(1))^{-D} \binom{ND}{2(K-D)} \le |\cT^*| (2K)^{4(K-D)}.
\$
For the first inequality, each $T\in \cT_D$ can be constructed in two steps: first build a tree with $D+1$ $n$-nodes, $D$ $d$-nodes, and $2D$ edges in which every $d$-node has degree $2$, then add the remaining $2(K-D)$ edges. The number of trees built by the first step is $(\alpha+o(1))^{-D}$, and the second step admits at most $\binom{ND}{2(K-D)}$ choices, since each added edge connects an $n$-node and a $d$-node and there are at most $ND$ such pairs. The second inequality follows from $(\alpha+o(1))^{-D} \le (\alpha+o(1))^{-K} = |\cT^*|$ and $ND\le [(N+D)/2]^2 = (K+1/2)^2$. 
Combining the last two displayed equations,   plugging the result back into~\eqref{eq:sum_D_bound}, and summing over $D\le K-1$ yields 
\#\label{eq:lb-tilde-cT-final}
\sum_{(T, S) \notin \calT^* \times \calT^*}
\sum_{\substack{T_1\cong T \\ S_1 \cong S}} \langle L, \phi_{T_1, S_1}\rangle^2 & \le C\rho^{4K} |\cT^*| \sum_{D = 1}^{K-1} \left(\frac{8K^5}{d}\right)^{2(K-D)}  \le \frac{CK^5}{d^2} \rho^{4K} |\cT^*|.
\#
We conclude the proof of \Cref{thm:lb-deg-2} by combining \eqref{eq:l-phi-cong-1}, \eqref{eq:lb-cT-2}, and \eqref{eq:lb-tilde-cT-final}, and recalling that $|\cT^*| =(\alpha+o(1))^{-K}$.
\end{proof}

The following lemma bounds the number of admissible pairs $(S_1,\bc)$, as claimed in \eqref{eq:sucuc-d}. 
It can be viewed as an extension of \Cref{lem:bound-mu1-le-k-2} from $D=K$ to general $D.$
\begin{lemma}\label{lmm:s-u-c}
Given $T_1 \cong T$ for some $T \in \cT_D$ and $D, \mu_1 \le K,$ it holds that
$$
\left| \{ (S_1,\bc): S_1 \cong S \text{ for some } S \in \cT_D, \bc \in \cC(T_1 \cup_n S_1), \mu_1(\bc)=\mu_1  \} \right| \le (4K)^{K-\mu_1}  (2K)^{4(K-D)} \binom{d}{D} D!.
$$
\end{lemma}
\begin{proof}
Since the $d$-nodes of $T_1$ may have arbitrary even degrees, while the edges of $T_1$ enter $\bc$ as pairs sharing a common $d$-node, we begin by pairing the edges incident to each $d$-node of $T_1$. Let $\{a_i\}_{i\in [D]}$ denote the degree sequence of the $d$-nodes of $T_1$, and define $A \triangleq \{i\in [D] \colon a_i\ge 4\}$. Since the total number of edges is $2K$, we have
\$ 
4|A| + 2(D - |A|) \le \sum_{i\in A} a_i + 2(D - |A|) = 2K. 
\$
Thus $|A|\le K-D$ and $\sum_{i\in A} a_i \le 2(K - D + |A|) \le 4(K-D)$. For a $d$-node of degree $a_i$, the number of pairings of its incident edges is $(a_i -1)!!$; in particular, this pairing is unique when $a_i=2$. Hence, the total number of pairings of the edges of $T_1$ is at most
\begin{align} 
\prod_{i\in A} (a_i - 1)!! \le \prod_{i\in A} a_i^{a_i/2}\le (2K)^{ \sum_{i\in A} a_i/2} \le (2K)^{2(K-D)}. \label{eq:edge_pair_bound}
\end{align}

For each such pairing of the edges of $T_1$, \Cref{lem:bound-mu1-le-k-2} gives at most $(4K)^{K-\mu_1}$ ways to construct
an alternating circuit decomposition $\bc$ with $\mu_1$ $4$-circuits and the corresponding union of $K$ blue edge pairs. At this stage, the blue part has $K$ degree-$2$
$d$-nodes. To determine an unlabeled tree $S \in \cT_D$, we merge these $K$ degree-$2$
$d$-nodes into $D$ nodes. Specifically, we choose $D$ out of the $K$ nodes, in at most $\binom{K}{D}$ ways. The remaining $K-D$ $d$-nodes are then merged with these $D$ nodes, in at most $D^{K-D}$ ways. In total, this can be done in at most $\binom{K}{D} D^{K-D} \le (KD)^{K-D} \le K^{2(K-D)}$ways.  
Finally, to obtain $S_1$ isomorphic to $S$, since $n$-nodes have already been labeled given the $K$ blue edge pairs, we only need to label the $D$ $d$-nodes in $\binom{d}{D} D!$ ways. 
\end{proof}

